\documentclass[11pt]{article}

\usepackage{epsfig, psfrag}
\usepackage{latexsym}
\usepackage{enumerate}
\usepackage{url}
\usepackage{float}
\usepackage[hypertexnames=false,hyperfootnotes=false]{hyperref}
\usepackage{texnansi}
\usepackage{color}
\usepackage{comment}
\usepackage{tikz}
\usepackage[margin=10pt,font=small,labelfont=bf]{caption}
\usepackage[subrefformat=parens,labelformat=simple]{subcaption}
\usepackage{afterpage}
\usepackage{enumitem}
\usepackage[boxed]{algorithm}
\usepackage{algpseudocode}
\usepackage[normalem]{ulem}
\usepackage{lmodern}
\usepackage{booktabs}
\usepackage{sectsty}
\usepackage{ifthen}
\usepackage{amsmath}
\usepackage{amsthm}
\usepackage{amssymb}
\usepackage{amsfonts}
\usepackage{thmtools}
\usepackage{thm-restate}
\usepackage{xspace}
\usepackage{titling}
\usepackage{parskip}
\usepackage{natbib}
\usepackage{xfrac}
\usepackage{multirow}
\usepackage{bigdelim}
\usepackage{bm}
\usepackage{makecell}
\usepackage{colortbl}
\usepackage{xcolor}
\usepackage{framed}
\usepackage[capitalize,noabbrev]{cleveref}
\PassOptionsToPackage{algo2e,ruled}{algorithm2e}
\RequirePackage{algorithm2e}
\newcommand{\E}{\ensuremath{\mathbb{E}}} 

\DeclareMathOperator*{\argmax}{\mathrm{argmax}}

\declaretheoremstyle[headfont=\sffamily\bfseries,bodyfont=\itshape]{thm-sf}
\declaretheorem[style=thm-sf]{theorem}
\declaretheorem[style=thm-sf]{remark}
\declaretheorem[style=thm-sf]{assumption}
\declaretheorem[style=thm-sf]{definition}
\declaretheorem[style=thm-sf]{example}

\declaretheorem[style=thm-sf]{corollary}
\declaretheorem[style=thm-sf]{lemma}
\declaretheorem[style=thm-sf]{fact}
\declaretheorem[style=thm-sf]{proposition}

\renewcommand{\thmcontinues}[1]{\hyperref[#1]{continued}}
\usetikzlibrary{arrows,patterns,plotmarks,pgfplots.groupplots}
\tikzstyle{every picture} += [>=stealth]
\tikzset{axis/.style={semithick, line join=miter}}
\allsectionsfont{\sffamily}
\makeatletter
\def\@seccntformat#1{\csname the#1\endcsname.\quad}
\makeatother

\floatname{algorithm}{\normalfont\sffamily\bfseries Algorithm}
\floatstyle{ruled}
\newcommand{\emailhref}[1]{\href{mailto:#1}{\tt #1}} 
\provideboolean{fastcompile}
\newcommand{\hidefastcompile}[1]{\ifthenelse{\boolean{fastcompile}}{}{#1}}

\usepackage{pgfplots}
\usepackage{pgfplotstable}
\usetikzlibrary{calc}
\usepackage{mathtools}
\definecolor{orange}{rgb}{0.85,0.33,0.13} 
\definecolor{green}{rgb}{0.13,0.85,0.33}
\definecolor{purple}{rgb}{0.33,0.13,0.85}
\definecolor{lime}{rgb}{0.65,0.85,0.13}
\definecolor{blue}{rgb}{0.13,0.65,0.85}
\pgfplotscreateplotcyclelist{tricolor}{%
  orange,every mark/.append style={fill=orange!80!black},mark=*\\%
  green,every mark/.append style={fill=green!80!black},mark=square*\\%
  purple,every mark/.append style={fill=purple!80!black},mark=otimes*\\%
  black,mark=star\\%
  orange,every mark/.append style={fill=orange!80!black},mark=diamond*\\%
  green,densely dashed,every mark/.append style={solid,fill=green!80!black},mark=*\\%
  purple,densely dashed,every mark/.append style={solid,fill=purple!80!black},mark=square*\\%
  black,densely dashed,every mark/.append style={solid,fill=gray},mark=otimes*\\%
  orange,densely dashed,mark=star,every mark/.append style=solid\\%
  green,densely dashed,every mark/.append style={solid,fill=green!80!black},mark=diamond*\\%
}
\pgfplotsset{colormap={tricolormap}{color=(orange) color=(green) color=(purple)},
  colormap={quadcolormap}{color=(orange) color=(lime) color=(blue) color=(purple)}}
\pgfplotstableset{%
  font=\small,
  every head row/.style={before row=\toprule[1pt], after row=\midrule},
  every last row/.style={after row=\bottomrule[1pt]}}

\usepackage{setspace}
\usepackage[margin=1in]{geometry}

\setcitestyle{authoryear,round,semicolon,aysep={,},yysep={,},notesep={, }}
\tikzstyle{rate} += [color=orange,very thick]
\usetikzlibrary{matrix,positioning}
\usepgfplotslibrary{groupplots}
\pgfplotsset{compat=newest}

  \title{\textsf{\textbf{An Information-Theoretic Analysis of Nonstationary Bandit Learning}}%
  }
\author{ \\
  Seungki Min \\
  Industrial and Systems Engineering \\
  KAIST \\
  \emailhref{skmin@kaist.ac.kr} 
  \and \\
  Daniel J. Russo \\
  Graduate School of Business \\
  Columbia University \\
 \emailhref{djr2174@gsb.columbia.edu} \\
}

\date{ Initial Version: February 2023 \\
 Current Revision: December 2023}

\begin{document} 

\maketitle
\singlespacing

\begin{abstract}
In nonstationary bandit learning problems, the decision-maker must continually gather information and adapt their action selection as the latent state of the environment evolves.  
In each time period, some latent optimal action maximizes expected reward under the environment state.  
We view the optimal action sequence as a stochastic process, and take an information-theoretic approach to analyze attainable performance. 
We bound per-period regret in terms of the \emph{entropy rate of the optimal action process}.
The bound applies to a wide array of problems studied in the literature and reflects the problem's information structure through its information-ratio.

\end{abstract}


\onehalfspacing

\section{Introduction}
\label{sec:intro}

We study the problem of learning in interactive decision-making. 
Across a sequence of time periods, a decision-maker selects actions, observes outcomes, and associates these with rewards. They hope to earn high rewards, but this may require investing in gathering information. 

 Most of the literature studies stationary environments --- where the likelihood of outcomes under an action is fixed across time.\footnote{An alternative style of result lets the environment change, but tries only to compete with the best fixed action in hindsight.} Efficient algorithms limit costs required to converge on optimal behavior. We study the design and analysis of algorithms in nonstationary environments, where converging on optimal behavior is impossible. 
 
 In our model, the latent state of the environment in each time period is encoded in a parameter vector. 
 These parameters are unobservable, but evolve according to a known stochastic process. 
 The decision-maker hopes to earn high rewards by adapting their action selection as the environment evolves.
 This requires continual learning from interaction and striking a judicious balance between exploration and exploitation.  
 Uncertainty about the environment's state cannot be fully resolved before the state changes and this necessarily manifests in suboptimal decisions. 
 Strong performance is impossible under adversarial forms of nonstationarity but is possible in more benign environments. 
 Why are A/B testing, or recommender systems, widespread and effective even though nonstationarity is a ubiquitous concern? Quantifying the impact  different forms of nonstationarity have on decision-quality is, unfortunately, quite subtle.

 \paragraph{Our contributions.}

We provide a novel information-theoretic analysis that bounds the inherent degradation of decision-quality in dynamic environments. The work can be seen as a very broad generalization of \cite{russo16,russo22} from stationary to dynamic environments. To understand our results, it is important to first recognize that the latent state evolution within these environments gives rise to a latent optimal action process. This process identifies an optimal action at each time step, defined as the one that maximizes the expected reward based on the current environment state.  \cref{thm:main-result} bounds the per-period regret in terms of the \emph{entropy rate of this optimal action process.} This rate is indicative of the extent to which changes in the environment lead to unpredictable and erratic shifts in the optimal action process. Complementing  \cref{thm:main-result}, \cref{sec:lower} lower bounds regret by the entropy rate of the optimal action process in a family of nonstationarity bandit problems.

Through examples, we illustrate that the entropy rate of the action process may be much lower than the entropy rate of the latent environment states.  Subsection \ref{subsec:example} provides an illustrative example of nonstationarity, drawing inspiration from A/B testing scenarios. This distinction is further explored in the context of news recommendation problems, as detailed in \cref{ex:news-rec} and \cref{ex:news-rec-revisit}. 

We provide several generalizations of our main result. \cref{sec:sts} considers problems where the environment evolves very rapidly, so  aiming to track the optimal action sequence is futile.  The assurances of \cref{thm:main-result} become vacuous in this case. Regardless of how erratic the environment is, \cref{thm:main-result-sts}  shows that it is possible to earn rewards competitive with any weaker benchmark --- something we call a \emph{satisficing} action sequence --- whose entropy rate is low.  

\cref{sec:rate-distortion} interprets and generalizes the entropy rate through the lens of a communication problem. In this problem, an observer of the latent state aims to transmit useful information to a decision-maker. The decision-maker can implement optimal actions if and only if the observer (optimally) transmits information with bit-rate exceeding the entropy rate of the optimal action process. \cref{thm:rate-distortion} replaces the entropy rate in \cref{thm:main-result} with a rate-distortion function \citep{cover2006elements}, which characterizes the bit-rate required to enable action selection with a given per-period regret. 

This abstract theory provides an intriguing link between interactive learning problems and the limits of lossy compression. We do not engage in a full investigation, but provide a few initial connections to existing measures of the severity of nonstationarity. \cref{prop:combinatorial-bound} bounds the entropy rate (and hence the rate-distortion function) in terms the number of switches in the optimal action process, similar to \cite{auer02b} or \cite{suk22}. \cref{prop:variation-non-constructive} upper bounds the rate distortion in terms of the total-variation in arms' mean-rewards, following the influential work of \cite{besbes15}. \cref{sec:adversarial} reviews a duality between stochastic and adversarial models of nonstationarity which shows these information-theoretic techniques can be used to study adversarial models.  

Our general bounds also depend on the algorithm's information ratio. While the problem's entropy rate (or rate-distortion function) bounds the rate at which the decision-maker must acquire information, the information-ratio bounds the per-period price of acquiring each bit of information. Since its introduction in the study of stationary learning problems in \cite{russo16}, the information-ratio has been shown to properly capture the complexity of learning in a range of widely studied problems. Recent works link it to generic limits on when efficient learning is possible \cite{lattimore2022minimax, foster2022on}.  
\cref{sec:information-ratio-bounds-old-and-new} details bounds on the information ratio that apply to many of the most important sequential learning problems --- ranging from contextual bandits to online matching. Through our analysis, these imply regret bounds in nonstationary variants of these problems. 

This work emphasizes understanding of the limits of attainable performance. 
Thankfully, most results apply to variants of Thompson sampling (TS) \citep{thompson1933likelihood}, one of the most widely used learning algorithms in practice. 
In some problems, TS is far from optimal, and better bounds are attained with Information-Directed Sampling \citep{russo18}. 

A short conference version of this paper appeared in  \citep{Min2023information}. In addition to providing a more thorough exposition, this full-length paper provides substantive extensions to that initial work, including the treatment of satisficing in  \cref{sec:sts}, the connections with rate-distortion theory in  \cref{sec:rate-distortion}, and the connections with adversarial analysis in \cref{sec:adversarial}.

 \subsection{An illustrative Bayesian model of nonstationarity}
 \label{subsec:example}

Consider a multi-armed bandit environment where two types of nonstationarity coexist  -- a common variation that affects the performance of all arms, and idiosyncratic variations that affect the performance of individual arms separately.
More explicitly, let us assume that the mean reward of arm $a$ at time $t$ is given by
$$ \mu_{t,a} = \theta_t^{\rm cm} + \theta_{t,a}^{\rm id}, $$
where $( \theta_t^{\rm cm} )_{t \in \mathbb{N}}$ and $( \theta_{t,a}^{\rm id} )_{t \in \mathbb{N}}$'s are latent stochastic processes describing common and idiosyncratic disturbances.
While deferring the detailed description to \cref{app:numerical}, we introduce two hyperparameters $\tau^{\rm cm}$ and $\tau^{\rm id}$ in our generative model to control the time scale of these two types of variations.\footnote{
We assume that $( \theta_t^{\rm cm} )_{t \in \mathbb{N}}$ is a zero-mean Gaussian process satisfying $\text{Cov}( \theta_s^{\rm cm}, \theta_t^{\rm cm} ) = \exp\left( -\frac{1}{2} \left( \frac{t-s}{\tau^{\rm cm}} \right)^2 \right)$ so that $\tau^{\rm cm}$ determines the volatility of the process.
Similarly, the volatility of $( \theta_{t,a}^{\rm id} )_{t \in \mathbb{N}}$ is determined by $\tau^{\rm id}$.}
 
\begin{figure}[ht]
\begin{center}
\centerline{\includegraphics[width=\columnwidth]{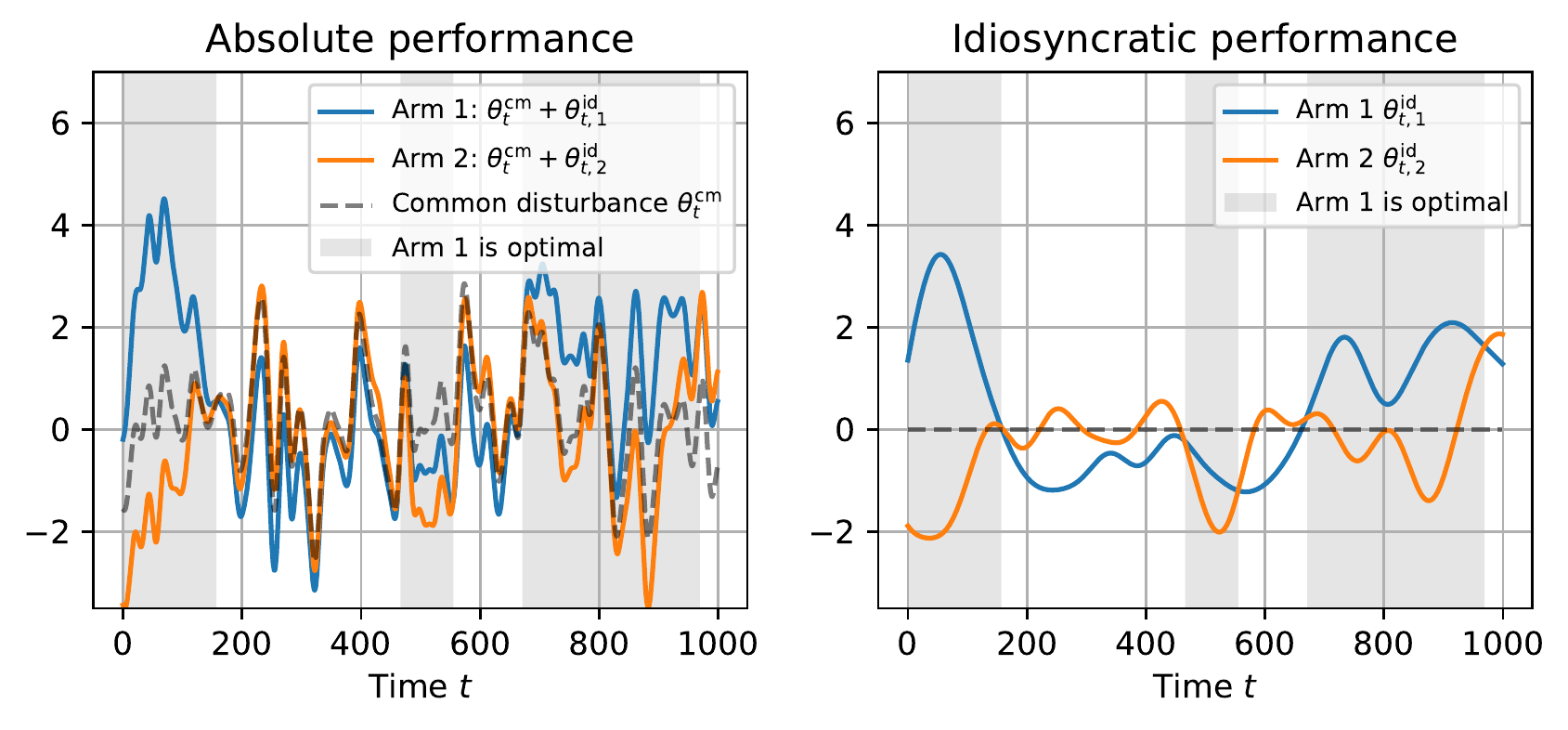}}
\caption{A two-arm bandit environment with two types of nonstationarity -- a common variation $( \theta_t^{\rm cm} )_{t \in \mathbb{N}}$ generated with a time-scaling factor $\tau^{\rm cm}=10$, and idiosyncratic variations $( \theta_{t,a}^{\rm id} )_{t \in \mathbb{N}, a \in \mathcal{A}}$ generated with a time-scaling factor $\tau^{\rm id}=50$.
While absolute performance of two arms are extremely volatile (left), their idiosyncratic performances are relatively stable (right).}
\label{fig:example-illustration}
\end{center}
\vskip -0.2in
\end{figure}

Inspired by real-world A/B tests \citep{wu2022non}, we imagine a two-armed bandit instance involving a common variation that is much more erratic than idiosyncratic variations.
Common variations reflect exogenous shocks to user behavior which impacts the reward under all treatment arms. 
\cref{fig:example-illustration} visualizes such an example, a sample path generated with the choice of $\tau^{\rm cm}=10$ and $\tau^{\rm id}=50$.
Observe that the optimal action $A_t^*$ has changed only five times throughout 1,000 periods.
Although that the environment itself is highly nonstationary and unpredictable due to the common variation term, the optimal action sequence $(A_t^*)_{t \in \mathbb{N}}$ is relatively stable and predictable since it depends only on the idiosyncratic variations.

Now we ask --- \emph{How difficult is this learning task? Which type of nonstationarity determines the difficulty?}
A quick numerical investigation shows that the problem's difficulty is mainly determined by the frequency of optimal action switches, rather than volatility of common variation.

\begin{figure}[ht]
\begin{center}
\centerline{\includegraphics[width=\columnwidth]{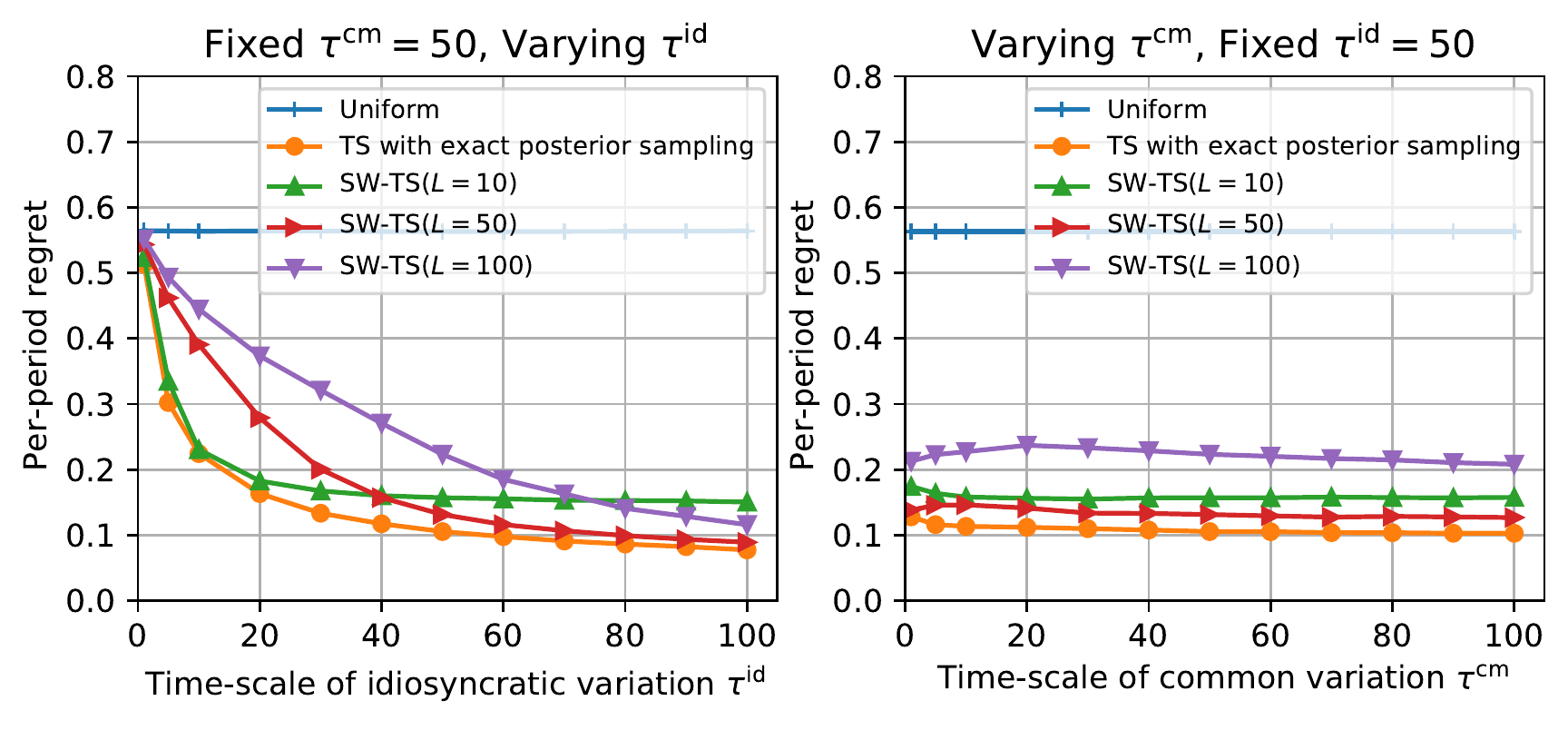}}
\caption{Performance of algorithms in two-armed bandit environments, with difference choices of time-scaling factors $\tau^{\rm cm}$ (common variation) and $\tau^{\rm id}$ (idiosyncratic variations). 
Each data point reports per-period regret averaged over 1,000 time periods and 1,000 runs of simulation.}
\label{fig:example-performance}
\end{center}
\vskip -0.2in
\end{figure}

See \cref{fig:example-performance}, where we report the effect of $\tau^{\rm cm}$ and $\tau^{\rm id}$ on the performance of several bandit algorithms (namely, Thompson sampling with exact posterior sampling,\footnote{
	In order to perform exact posterior sampling, it exploits the specified nonstationary structure as well as the values of $\tau^{\rm cm}$ and $\tau^{\rm id}$.
} and Sliding-Window TS that only uses recent $L \in \{10,50,100\}$ observations; see \cref{app:numerical} for the details).
Remarkably, their performances appear to be sensitive only to $\tau^{\rm id}$ but not to $\tau^{\rm cm}$, highlighting that nonstationarity driven by common variation is benign to the learner.

We remark that our information-theoretic analyses predict this result.
\cref{thm:main-result} shows that the complexity of a nonstationary environment can be sufficiently characterized by the entropy rate of the optimal action sequence, which depends only on $\tau^{\rm id}$ but not on $\tau^{\rm cm}$ in this example.
\cref{prop:effective-horizon-bound} further expresses the entropy rate in terms of effective horizon, which corresponds to $\tau^{\rm id}$ in this example. (See \cref{ex:two-armed-bandit-revisited}, where we revisit this two-armed bandit problem.)

\subsection{Comments on the use of prior knowledge}
A substantive discussion of Bayesian, frequentist, and adversarial perspectives on decision-making uncertainty is beyond the scope of this paper. 
We make two quick observations. 
First, where does a prior like the one in \cref{fig:example-illustration} come from? 
One answer is that company may run many thousands of A/B tests, and an informed prior may let them transfer experience across tests  \citep{azevedo2019empirical}. 
In particular, experience with past tests may let them calibrate $\tau^{\rm id}$, or form hierarchical prior where $\tau^{\rm id}$ is also random.
\cref{ex:news-rec} in \cref{subsubsec:news-rec} illustrates the possibility of estimating prior hyperparameters in a news article recommendation problem.
Second, Thompson sampling with a stationary prior is perhaps the most widely used bandit algorithm. 
One might view the model in Section \ref{subsec:example} as a more conservative way of applying TS that guards against a certain magnitude of nonstationarity.

\subsection{Literature review}

We adopt \emph{Bayesian viewpoints} to describe nonstationary environments: changes in the underlying reward distributions (more generally, changes in outcome distributions) are driven by a stochastic process.
Such a viewpoint dates back to the earliest work of \citet{whittle88} which introduces the term `restless bandits' and has motivated subsequent work \cite{slivkins08, chakrabarti08, jung19}.
On the other hand, since Thompson sampling (TS) has gained its popularity as a Bayesian bandit algorithm, its variants have been proposed for nonstationary settings accordingly: e.g., Dynamic TS \citep{gupta11}, Discounted TS \citep{raj17}, Sliding-Window TS \citep{trovo20}, TS with Bayesian changepoint detection \citep{mellor13, ghatak20}.
Although the Bayesian framework can flexibly model various types of nonstationarity, this literature rarely presents performance guarantees that apply to a broad class of models.

To provide broad performance guarantees, our analysis adopts an \emph{information-theoretic approach} introduced in \citet{russo16} and extended in \citet{russo22}. While past work applies this framework to analyze learning in stationary environments, we show that it can be gracefully extended to dynamic environments. In this extension, bounds that depend on the entropy of \emph{the static optimal action} \citep{russo16} are replaced with the \emph{entropy rate of the optimal action process}; rate-distortion functions are similarly extended. A strength of this approach is enables one to leverage a wealth of existing bounds on the information-ratio, including for $k$-armed bandits/linear bandits/combinatorial bandits \citep{russo16}, contextual bandits \citep{neu22}, logistic bandits \citep{dong19}, bandits with graph-based feedback \citep{liu18}, convex optimization \citep{bubeck2016multi,lattimore20}, etc.

  \citet{liu23} recently proposed an algorithm called Predictive Sampling for nonstationarity bandit learning problems. At a conceptual level, their algorithm is most closely related to our  \cref{sec:sts},  in which Thompson sampling style algorithms are modified to explore less aggressively in the face of rapid dynamics. The authors introduce a modified information-ratio and use it to successfully display the benefits of predictive sampling over vanilla Thompson sampling. Our framework does not allow us to study predictive sampling, but enables us to establish a wide array of theoretical guarantees for other algorithms that are not provided in \cite{liu23}. We leave a detailed comparison in \cref{app:predictive-sampling}.

Recent work of \cite{chen23} develops an algorithm regret analysis for $k$-armed bandit problems when arm means (i.e. the latent states) evolve according to an auto-regressive process. This is a special case of our formulation and a detailed application application of our  \cref{thm:main-result} or \cref{thm:rate-distortion} to such settings would be an interesting avenue for future research.  One promising route is to modify the argument in \cref{ex:two-armed-bandit-revisited}. 

Most existing theoretical studies on nonstationary bandit experiments adopt adversarial viewpoints in the modeling of nonstationarity. Implications of our information-theoretic regret bounds in adversarial environments are discussed in \cref{sec:adversarial}. These works typically fall into two categories -- ``switching environments'' and ``drifting environments''.  

\emph{Switching environments} consider a situation where underlying reward distributions change at unknown times (often referred to as changepoints or breakpoints).
Denoting the total number of changes over $T$ periods by $N$, it was shown that the cumulative regret $\tilde{O}(\sqrt{N T})$ is achievable: e.g., Exp3.S \citep{auer02a, auer02b}, Discounted-UCB \citep{kocsis06}, Sliding-Window UCB \citep{garivier2011upper}, and more complicated algorithms that actively detect the changepoints \citep{auer19, chen19}. While those results depend on the number of switches in the underlying environment, the work of \citet{auer02b} already showed that it is possible to achieve a bound of  $\tilde{O}(\sqrt{S T})$ where $S$ only counts the number of switches in identity of the best-arm.  More recent studies design algorithms that adapt to an unknown number of switches \citep{yadkori22, suk22}.
Our most comparable result, which follows from \cref{thm:main-result} with \cref{prop:combinatorial-bound}, is a bound on the expected cumulative regret of Thompson sampling is $\tilde{O}(\sqrt{S T})$, in a wide range of problems beyond $k$-armed bandits (see Section \ref{sec:information-ratio-bounds}). 

Another stream of work considers \emph{drifting environments}.
Denoting the total variation in the underlying reward distribution by $V$ (often referred to as the variation budget), 
it was shown that the cumulative regret $\tilde{O}(V^{1/3} T^{2/3})$ is achievable \citep{besbes14, besbes15, cheung19}.
Our most comparable result is given in  \cref{subsec:variation-budget}, which recovers the Bayesian analogue of this result by bounding the rate distortion function. In fact, we derive a stronger regret bound that ignores variation that is common across all arms. 
A recent preprint by \cite{jia2023smooth} reveals that stronger bounds are possible under the additional smoothness assumptions on environment variations. It is an interesting open question to study whether one can derive similar bounds by bounding the rate-distortion function.

\section{Problem Setup}
\label{sec:problem}

A decision-maker interacts with a changing environment across rounds $t \in \mathbb{N}:= \{1,2,3,\ldots\}$. 
In period $t$, the decision-maker selects some action $A_t$ from a finite set $\mathcal{A}$, observes an outcome $O_t$, and associates this with reward $R_t = R(O_t, A_t)$ that depends on the outcome and action through a known utility function $R(\cdot)$. 

There is a function $g$, an i.i.d sequence of disturbances $W=(W_t)_{t\in \mathbb{N}}$, and a sequence of latent environment states $\theta=(\theta_t)_{t\in \mathbb{N}}$ taking values in $\Theta$, such that outcomes are determined as
\begin{equation}\label{eq:outcome-generation}
O_t=g(A_t, \theta_t, W_t ). 
\end{equation}
Write potential outcomes as $O_{t,a}=g(a, \theta_t, W_t)$ and potential rewards as $R_{t,a}= R(O_{t,a}, a)$. Equation \eqref{eq:outcome-generation} is equivalent to specifying a known probability distribution over outcomes for each choice of action and environment state. 

The decision-maker wants to earn high rewards even as the environment evolves, but cannot directly observe the environment state or influence its evolution. 
Specifically, the decision-maker's actions are determined by some choice of policy $\pi=(\pi_1, \pi_2,\ldots)$. 
At time $t$, an action $A_t= \pi_t(\mathcal{F}_{t-1}, \tilde{W}_t)$ is a function of the observation history $\mathcal{F}_{t-1}=(A_1, O_1, \ldots, A_{t-1}, O_{t-1})$ and an internal random seed $\tilde{W}_t$ that allows for randomness in action selection. 
Reflecting that the seed is exogenously determined, assume $\tilde{W}=(\tilde{W}_t)_{t\in \mathbb{N}}$ is jointly independent of the outcome disturbance process $W$ and state process
 $\theta=(\theta_t)_{t\in \mathbb{N}}$.   
That actions do not influence the environment's evolution can be written formally through the conditional independence relation $(\theta_{s})_{s\geq t+1}  \perp \mathcal{F}_t \mid (\theta_{\ell})_{\ell \leq t}$.

The decision-maker wants to select a policy $\pi$ that accumulates high rewards as this interaction continues. 
They know all probability distributions and functions listed above, but are uncertain about how environment states will evolve across time. 
To perform `well', they need to continually gather information about the latent environment states and carefully balance exploration and exploitation. 

Rather than measure the reward a policy generates, it is helpful to measure its regret. 
We define the \emph{$T$-period per-period regret} of a policy $\pi$ to be
\[
	\bar{\Delta}_T(\pi) := \frac{ \mathbb{E}_{\pi}\left[ \sum_{t=1}^T \big(  R_{t,A_t^*} - R_{t,A_t} \big) \right] }{T},
\]
where the latent optimal action $A_t^*$ is a function of the latent state $\theta_t$ satisfying $A_t^* \in \argmax_{a \in \mathcal{A}} \mathbb{E}[R_{t,a} \mid \theta_t]$.
We further define the \emph{regret rate} of policy $\pi$ as its limit value,
\[
	\bar{\Delta}_\infty(\pi) := \limsup_{T \to \infty} \, \bar{\Delta}_T(\pi).
\]
It measures the (long-run) per-period degradation in performance due to uncertainty about the environment state.

\begin{remark} \label{rem:conjecture}
	The use of a limit supremum and Ces\`{a}ro averages is likely unnecessary under some technical restrictions. 
	For instance, under Thompson sampling applied to Examples~\ref{ex:k-armed-bandit}--\ref{ex:contextual-bandit}, if the latent state process $(\theta_t)_{t \in \mathbb{N}}$ is ergodic, we conjecture that $\bar{\Delta}_\infty(\pi)= \lim_{t\to \infty} \mathbb{E}_{\pi}\left[ R_{t,A_t^*} - R_{t,A_t} \right]$. 
\end{remark}

Our analysis proceeds under the following assumption, which is standard in the literature. 
\begin{assumption} \label{ass:subgaussian}
	There exists $\sigma$ such that, conditioned on $\mathcal{F}_{t-1}$, $R_{t,a}$ is sub-Gaussian with variance proxy $\sigma^2$.  
\end{assumption}

\subsection{`Stationary processes' in `nonstationary bandits'}
The way the term `nonstationarity' is used in the bandit learning literature could cause confusion as it conflicts with the meaning of `stationarity' in the theory of stochastic process, which we use  elsewhere in this paper. 
	\begin{definition}\label{def:stationary}
		A stochastic process $X=(X_t)_{t\in \mathbb{N}}$ is (strictly) stationary if for each integer $t$, the random vector $(X_{1+m}, \ldots, X_{t+m})$ has the same distribution for each choice of $m$.  
	\end{definition}
`Nonstationarity', as used in the bandit learning literature, means that \emph{realizations} of the latent state $\theta_t$ may differ at different time steps. 
The decision-maker can gather information about the current state of the environment, but it may later change.  
Nonstationarity of the stochastic process $(\theta_{t})_{t\in \mathbb{N}}$, in the language of probability theory, arises when apriori there are predictable differences between environment states at different timesteps -- e.g., if time period $t$ is nighttime then rewards tend to be lower than daytime. 
It is often clearer to model predictable differences like that through contexts, as in \cref{ex:contextual-bandit}.

\subsection{Examples}
\subsubsection{Modeling a range of interactive decision-making problems}

Many interactive decision-making problems can be naturally written as special cases of our general protocol, where actions generate outcomes that are associated with rewards. 

Our first example describes a bandit problem with independent arms, where outcomes generate information only about the selected action.

\begin{example}[$k$-armed bandit]\label{ex:k-armed-bandit}
	Consider a website who can display one among $k$ ads at a time and gains one dollar per click.
	For each ad $a \in [k] := \{1,\ldots,k\}$, the potential outcome/reward $O_{t,a} = R_{t,a} \sim \text{Bernoulli}(\theta_{t,a})$ is a random variable representing whether the ad $a$ is clicked by the $t^\text{th}$ visitor if displayed, where $\theta_{t,a} \in [0,1]$ represents its click-through-rate.
	The platform only observes the reward of the displayed ad, so $O_t = R_{t,A_t}$.
\end{example}

Full information online optimization problems fall at the other extreme of $k$-armed bandit problems. 
There the potential observation $O_{t,a}$ does not depend on the chosen action $a$, so purposeful information gathering is unnecessary.
The next example was introduced by \citet{cover91} and motivates such scenarios. 
\begin{example}[Log-optimal online portfolios]\label{ex:full-information} 
Consider a small trader who has no market impact. 
In period $t$ they have wealth $W_t$ which they divide among $k$ possible investments. 
The action $A_{t}$ is chosen from a feasible set of probability vectors, with $A_{t,i}$ denoting the proportion of wealth invested in stock $i$. 
The observation is $O_t\in \mathbb{R}^k_+$ where $O_{t,i}$ is the end-of-day value of $\$1$ invested in stock $i$ at the start of the day and the distribution of $O_t$ is parameterized by $\theta_t$.
Because the observation consists of publicly available data, and the trader has no market impact, $O_{t}$ does not depend on the investor's decision. Define the reward function $R_t=\log\left( O_t^\top A_t \right)$. Since wealth evolves according to the equation $W_{t+1} = \left(O_t^\top A_t\right) W_t$,
$$\sum_{t=1}^{T-1} R_t = \log(W_T/ W_1) .$$
\end{example}
Many problems lie in between these extremes. 
We give two examples. 
The first is a matching problem. Many pairs of individuals are matched together and, in addition to the cumulative reward, the decision-maker observes feedback on the quality of outcome from each individual match. 
This kind of observation structure is sometimes called ``semi-bandit'' feedback \cite{audibert14}. 
\begin{example}[Matching] \label{ex:matching}
	Consider an online dating platform with two disjoint sets of individuals $\mathcal{M}$ and $\mathcal{W}$.
	On each day $t$, the platform suggests a matching of size $k$, $A_t \subset \{ (m,w): m \in \mathcal{M}, w \in \mathcal{W} \}$ with $|A_t| \leq k$.
	For each pair $(m,w)$, their match quality is given by $Q_{t,(m,w)}$, where the distribution of $Q_t = (Q_{t,(m,w)}: m \in \mathcal{M}, w \in \mathcal{W})$ is parameterized by $\theta_t$.
	The platform observes the quality of individual matches, $O_t = \big( Q_{t,(m,w)} : (m,w) \in A_t \big)$, and earns their average, $R_t = \frac{1}{k}\sum_{(m,w) \in A_t} Q_{t,(m,w)}$.
\end{example}
Our final example is a contextual bandit problem.
Here an action is itself more like a policy --- it is a rule for assigning treatments on the basis of an observed context. 
Observations are richer than in the $k$-armed bandit. 
The decision-maker sees not only the reward a policy generated but also the context in which it was applied. 
\begin{example}[Contextual bandit] \label{ex:contextual-bandit}
	Suppose that the website described in \cref{ex:k-armed-bandit} can now access additional information about each visitor, denoted by $X_t \in \mathcal{X}$. 
	The website observes the contextual information $X_t$, chooses an ad to display, 
	and then observes whether the user clicks. 
	To represent this task using our general protocol, we let the decision space $\mathcal{A}$ be the set of mappings from the context space $\mathcal{X}$ to the set of ads $\{1,\ldots, k\}$, the decision  $A_t \in \mathcal{A}$  be a personalized advertising rule, and the observation $O_t=(X_t, R_t)$ contains the observed visitor information and the reward from applying the ad $A_t(X_t)$. 
	Rewards are drawn according to $R_t \mid X_t, A_t, \theta_t \sim  \text{Bernoulli}( \phi_{\theta_t}(X_t, A_t(X_t) ) )$, where $\phi_\theta:\mathcal{X}\times [k] \to [0,1]$ is a parametric click-through-rate model.
	Assume $X_{t+1} \perp (A_t, \theta) \mid X_t, \mathcal{F}_{t-1}$. 
	This assumption means that advertising decisions cannot influence the future contexts and that parameters of the click-through rate model $\theta=(\theta_t)_{t\in \mathbb{N}}$ cannot be inferred passively by observing contexts.
\end{example}

\subsubsection{Unconventional models of nonstationarity}\label{subsubsec:news-rec}
\cref{subsec:example} already presented one illustration of a $k$-armed bandit problem with in which arms' mean-rewards change across time. This section illustrates a very different kind of dynamics that can also be cast as a special case of our framework .

The next example describes a variant of the $k$-armed bandit problem motivated by a news article recommendation problem treated in the seminal work of \cite{chapelle11}. We isolate one source of nonstationarity: the periodic replacement of news articles with fresh ones. 
Each new article is viewed as a draw from a  population distribution; experience recommending other articles gives an informed prior on an article's click-through-rate, but resolving residual uncertainty requires exploration. Because articles are continually refreshed, the system must perpetually balance exploration and exploitation. 

Our treatment here is purposefully stylized. A realistic treatment might include features of the released articles which signal its click-through-rate, features of arriving users, time-of-day effects, and more general article lifetime distributions.  
\begin{example}[News article recommendation]\label{ex:news-rec}
         At each user visit, a news website recommends one out of a small pool of hand-picked news articles. The pool is dynamic -- old articles are removed and new articles are added to the pool in an endless manner. An empirical investigation on a dataset from Yahoo news (Yahoo! Front Page User Click Log dataset) shows that, as visualized in \cref{fig:example-article-pool}, the pool size varies between 19 and 25 while one article stays in the candidate pool 17.9 hours in average.
        
\begin{figure}[ht]
\begin{center}
\centerline{\includegraphics[width=\columnwidth]{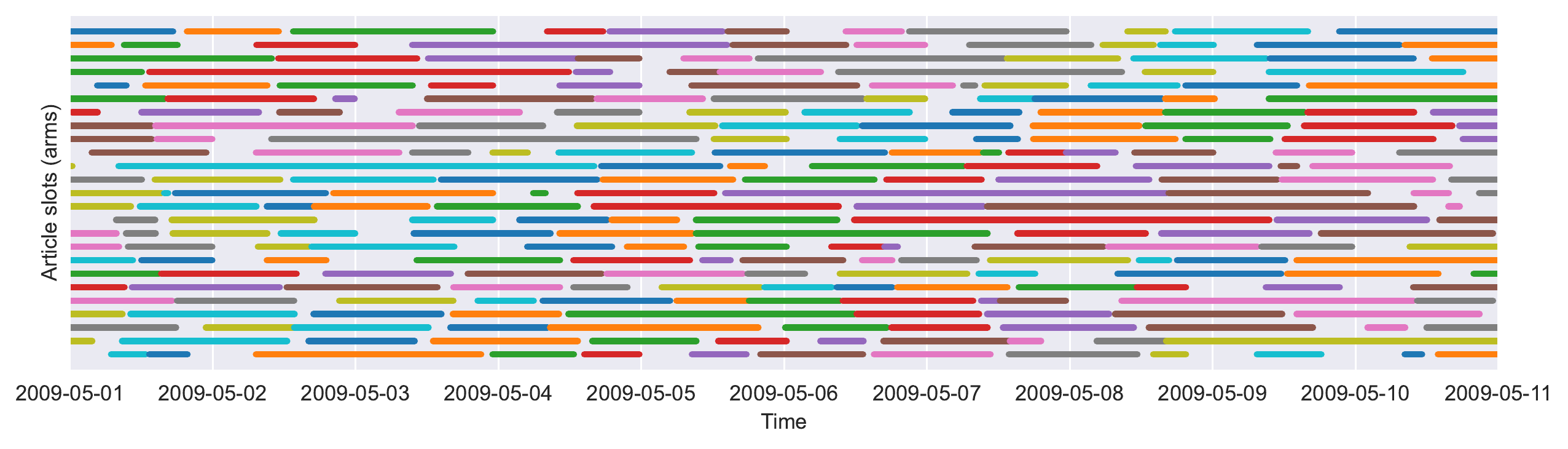}}
\caption{Addition/removal times of 271 articles recorded in ``Yahoo! Front Page User Click Log'' dataset.
Each horizontal line segment represents the time period on which an article was included in the candidate pool, and these line segments are packed into 25 slots.
Colors are randomly assigned.
	}
\label{fig:example-article-pool}
\end{center}
\vskip -0.2in
\end{figure}
We can formulate this situation with $k=25$ article slots (arms), each of which contains one article at a time.
Playing an arm $i \in [k]$ corresponds to recommending the article in slot $i$.
The article in a slot gets replaced with a new article at random times; the decision-maker observes that it was replaced but is uncertain about the click-through rate of the new article. 

Let $\theta_t \in [0,1]^{k}$ denote the vector of click-through rates across $k$ article slots, where $R_{t,a} \mid \theta_{t,a} \sim \text{Bernoulli}(\theta_{t,a})$. The observation at time $t$ when playing arm $a$ is $O_{t,a} = (R_{t,a}, \chi_t)$, where $\chi_t \in \{0,1\}^{k}$ indicates which article slots were updated. If $\chi_{t,a}=0$, then $\theta_{t+1,a}=\theta_{t,a}$ almost surely. If $\chi_{t,a}=1$, then $\theta_{t+1,a}$ is drawn independently from the past as 
\[
    \theta_{t+1,a} \mid (\chi_{t,a}=1)  \sim Q. 
\]
As this process repeats across many articles, the population click-through-rate distribution $Q$ can be learned from historical data. For purposes of analysis, we simply assume $Q$ is known.

Assume that $\chi_{t,a}\sim {\rm Bernoulli}(1/\tau)$ is independent across times $t$ and arms $a$; this means that each article's lifetime follows a  Geometric distribution with mean lifetime $\tau$.  (Assume also that  $(\chi_t)_{t\in \mathbb{N}}$ is independent of the disturbance process $W$ and random seeds $\tilde{W}$). 
\end{example}

\section{Thompson sampling}\label{sec:TS}

While this paper is primarily focused on the limits of attainable performance, 
many of our regret bounds apply to Thompson sampling \citep{thompson1933likelihood}, one of the most popular bandit algorithms. 
This section explains how to generalize its definition to dynamic environments.
 
We denote the Thompson sampling policy by $\pi^{\rm TS}$ and define it by the probability matching property:
\begin{equation}\label{eq:ts}
	\mathbb{P}(A_t=a \mid \mathcal{F}_{t-1})=\mathbb{P}(A_t^*=a \mid \mathcal{F}_{t-1}), 
\end{equation}
which holds for all $t\in \mathbb{N}$, $a\in \mathcal{A}$. Actions are chosen by sampling from the posterior distribution of the optimal action at the current time period. This aligns with the familiar definition of Thompson sampling in static environments where $A_1^*=\cdots = A_T^*.$

Practical implements of TS sampling from the posterior distribution of the optimal action by sampling plausible latent state $\hat{\theta}_t\sim \mathbb{P}(\theta_t = \cdot \mid \mathcal{F}_{t-1})$ and picking the action $A_t \in \argmax_{a\in \mathcal{A}} \E[R_{t,a} \mid \theta_t = \hat{\theta}_t]$.

\paragraph{How the posterior gradually forgets stale data.}

The literature on nonstationary bandit problems often constructs mean-reward estimators that explicitly throw away or discount old data \citep{kocsis06, garivier2011upper}. Under Bayesian algorithms like TS, forgetting of stale information happens \emph{implicitly} through the formation of posterior distributions. 
To provide intuition, we explain that proper posterior updating resembles an exponential moving average estimator under a simple model of nonstationarity.

Consider  $k$-armed bandit with Gaussian noises where each arm's mean-reward process follows an auto-regressive process with order $1$.
That is,
$$ \theta_{t,a} = \alpha \cdot \theta_{t-1,a} + \xi_{t,a}, \quad R_{t,a} = \theta_{t,a} + W_{t,a}, $$
where $\xi_{t,a} \sim \mathcal{N}(0, \sigma_\xi^2)$ are i.i.d. innovations in the mean-reward process, $W_{t,a} \sim \mathcal{N}(0, \sigma_W^2)$ are i.i.d. noises in the observations, and $\alpha \in (0,1)$ is the autoregression coefficient.
Assume $\theta_{1,a} \sim \mathcal{N}( 0, \frac{\sigma_\xi^2}{1-\alpha^2} )$ so that the mean-reward process is stationary.

\cref{alg:AR-bandit} implements Thompson sampling in this environment.
It uses a (simple) Kalman filter to obtain the posterior distribution. 
Past data is ``forgotten'' at a geometric rate over time, and the algorithm shrinks posterior mean estimates toward $0$ (i.e., the prior mean) if few recent observations are available. 

\begin{algorithm}
	\caption{Thompson sampling for AR(1)-bandit}
        \label{alg:AR-bandit}
	\KwIn{AR(1) process parameters $(\alpha, \sigma_\xi^2)$, noise variance $\sigma_W^2$.}
	   Initialize $\hat{\mu}_{0,a} \gets 0, \hat{\nu}_{0,a} \gets \frac{\sigma_\xi^2}{1-\alpha^2}$ for each $a \in \mathcal{A}$\;
    	
	\For{$t=1,2,\ldots$}{
            Diffuse posterior as it proceeds one time step: for each $a \in \mathcal{A}$,
            \begin{equation}
                \hat{\mu}_{t,a} \gets \alpha \hat{\mu}_{t-1,a}, \quad\hat{\nu}_{t,a} \gets \alpha^2 \hat{\nu}_{t-1,a} + \sigma_\xi^2
            \end{equation}

            Sample $\tilde{\theta}_a \sim \mathcal{N}( \hat{\mu}_{t,a}, \hat{\nu}_{t,a} )$ for each $a \in \mathcal{A}$\;
  
		Play $A_t = \argmax_a \tilde{\theta}_a$, and observe $R_t$\;

            Additionally update posterior for the new observation: for $a = A_t$ only,
	\begin{equation}
            \hat{\mu}_{t,a} \gets \frac{ \hat{\nu}_{t,a}^{-1} \times \hat{\mu}_{t,a} + \sigma_W^{-2} \times R_t }{ \hat{\nu}_{t,a}^{-1} + \sigma_W^{-2} }, \quad \hat{\nu}_{t,a} \gets \frac{1}{ \hat{\nu}_{t,a}^{-1} + \sigma_W^{-2} }
        \end{equation}\;
	}
\end{algorithm}

\section{Information Theoretic Preliminaries}

The entropy of a discrete random variable $X$, defined by $H(X) := -\sum_{x} \mathbb{P}(X=x)\log(\mathbb{P}(X=x))$,  measures the uncertainty in its realization. 
The entropy rate of a stochastic process $(X_1, X_2, \ldots)$ is the rate at which entropy of the partial realization $(X_1, \ldots, X_t)$ accumulates as $t$ grows.
\begin{definition}
	The \emph{$T$-period entropy rate} of a stochastic process $X=(X_t)_{t\in \mathbb{N}}$, taking values in a discrete set $\mathcal{X}$, is
	\begin{align*}
		\bar{H}_T(X) := \frac{ H\left( ~[X_1, \ldots, X_T]~ \right) }{T}.
	\end{align*}
	The \emph{entropy rate} is defined as its limit value:
	\[
		\bar{H}_\infty(X) := \limsup_{T \to \infty} \, \bar{H}_T(X).
	\]
\end{definition}

We provide some useful facts about the entropy rate:
\begin{fact}
    $0 \leq \bar{H}_T(X) \leq \log(|\mathcal{X}|)$.
\end{fact}
\begin{fact}
    By chain rule,
    \[
        \bar{H}_T(X) = \frac{1}{T} \sum_{t=1}^{T} H(X_t | X_{t-1}, \ldots, X_1 ).
    \]
    If $X$ is a stationary stochastic process, then 
	\begin{equation}\label{eq:entropy-rate-stationary}
		\bar{H}_\infty(X) = \lim_{t\to \infty}  H(X_t | X_{t-1}, \ldots, X_1 ).
	\end{equation}
\end{fact}
The form \eqref{eq:entropy-rate-stationary} is especially elegant. 
The entropy rate of a stationary stochastic process is the residual uncertainty in the draw of $X_t$ which cannot be removed by knowing the draw of $X_{t-1}, \ldots, X_1$. 
Processes that evolve quickly and erratically have high entropy rate. 
Those that tend to change infrequently (i.e., $X_t=X_{t-1}$ for most $t$) or change predictably will have low entropy rate.

The next lemma sharpens this understanding by establishing an upper bound on the entropy rate in terms of the switching frequency.

\begin{lemma} \label{lem:combinatorial-bound}
    Consider a stochastic process $X=(X_t)_{t \in \mathbb{N}}$ taking values in a finite set $\mathcal{X}$.
    Let $\mathcal{S}_T(X)$ be the number of switches in $X$ occurring up to time $T$,
    \begin{equation} \label{eq:switching-rate}
        \mathcal{S}_T(X) := 1 + \sum_{t=2}^T \mathbb{I}\{ X_t \ne X_{t-1} \}.
    \end{equation}
    If $\mathcal{S}_T(X) \leq S_T$ almost surely for some $S_T \in \{1,\ldots,T\}$, the $T$-period entropy rate of $X$ is upper bounded as 
    \[
        \bar{H}_T(X)
            ~\leq~
            \frac{S_T}{T} \cdot \left( 1 + \log \left( 1 + \frac{T}{S_T} \right) + \log( |\mathcal{X}| ) \right).
    \]
\end{lemma}
The proof is provided in \cref{app:proof}, which simply counts the number of possible realizations $(X_1, \ldots, X_T) \in \mathcal{X}^T$ with a restricted number of switches.

The conditional entropy rate and mutual information rate of two stochastic processes are defined similarly.
For brevity, we write $X_{1:T}$ to denote the random vector $(X_1, \ldots, X_T)$.

\begin{definition}
    Consider a stochastic process $(X,Z)=(X_t, Z_t)_{t\in \mathbb{N}}$, where $X_t$ takes values in a discrete set.
    The $T$-period conditional entropy rate of $X$ given $Z$ is defined as
    \[
        \bar{H}_T(X|Z) := \frac{H( ~X_{1:T}~ \mid ~Z_{1:T}~ )}{T},
    \]
    and the $T$-period mutual information rate between $X$ and $Z$ is defined as
    \[
        \bar{I}_T(X;Z) := \frac{I( ~X_{1:T}~ ; ~Z_{1:T}~ )}{T}.
    \]
    The conditional entropy rate and mutual information rate are defined as their limit values: i.e., $\bar{H}_\infty(X|Z) := \limsup_{T \to \infty} \bar{H}_T( X | Z )$, and $\bar{I}_\infty(X;Z) := \limsup_{T \to \infty} \bar{I}_T( X; Z )$.
\end{definition}

The mutual information rate $\bar{I}_T(X;Z)$ represents the average amount of communication happening between two processes, $X$ and $Z$.
This quantity can also be understood as the decrease in the entropy rate of $X$ resulting from observing $Z$, and trivially it cannot exceed the entropy rate of $X$:

\begin{fact}
    $0 \leq \bar{I}_T(X;Z) = \bar{H}_T(X) - \bar{H}_T(X|Z) \leq \bar{H}_T(X)$.
\end{fact}

The next fact shows that the mutual information rate can be expressed as the average amount of information that $Z$ accumulates about $X$ in each period, following from the chain rule.
\begin{fact}
    $\bar{I}_T(X;Z) = \frac{1}{T}\sum_{t=1}^T I( ~X_{1:T} ~; ~ Z_t ~|~ Z_{t-1},\ldots,Z_1 )$.
\end{fact}

\section{Information-Theoretic Analysis of Dynamic Regret}
\label{sec:dregret}
We apply the information theoretic analysis of \citet{russo16} and establish upper bounds on the per-period regret, expressed in terms of (1) the algorithm's information ratio, and (2) the entropy rate of the optimal action process.

\subsection{Preview: special cases of our result}
We begin by giving a special case of our result. 
It bounds the regret of Thompson sampling in terms of the reward variance proxy $\sigma^2$, the number of actions $|\mathcal{A}|$, and the entropy rate of the optimal action process $\bar{H}_T(A^*)$. 

\begin{corollary}\label{cor:k-armed}
	Under any problem in the scope of our problem formulation,
	\[
		\bar{\Delta}_T(\pi^{\rm TS}) \leq \sigma \sqrt{2 \cdot |\mathcal{A}| \cdot \bar{H}_T(A^*) },
	\]
	\[
		\text{and} \quad \bar{\Delta}_\infty(\pi^{\rm TS}) \leq \sigma \sqrt{2 \cdot |\mathcal{A}| \cdot \bar{H}_\infty(A^*) }.
	\]

\end{corollary}
This result naturally covers a wide range of bandit learning tasks while highlighting that the entropy rate of optimal action process captures the level of degradation due to nonstationarity.
Note that it includes as a special case the well-known regret upper bound established for a stationary $k$-armed bandit: when $A_1^* = \ldots = A_T^*$, we have $H([A_1^*,\ldots,A_T^*]) = H(A_1^*) \leq \log k$, and thus $\bar{\Delta}_T(\pi^{\rm TS}) \leq \tilde{O}( \sigma \sqrt{ k / T} )$.

According to \eqref{eq:entropy-rate-stationary}, the entropy rate is small when the conditional entropy $H(A_t^* \mid A^*_1, \ldots, A^*_{t-1})$ is small. 
That is, the entropy rate is small if most uncertainty in the optimal action $A_t^*$ is removed through knowledge of the past optimal actions.
Of course, Thompson sampling does not observe the environment states or the corresponding optimal actions, so  its dependence on this quantity is somewhat remarkable. 

The dependence of regret on the number of actions, $|\mathcal{A}|$, is unavoidable in a problem like the $k$-armed bandit of \cref{ex:k-armed-bandit}. 
But in other cases, it is undesirable. Our general results depend on the problem's information structure in a more refined manner. 
To preview this, we give another corollary of our main result, which holds for problems with full-information feedback (see \cref{ex:full-information} for motivation).
In this case, the dependence on the number of actions completely disappears and the bound depends on the variance proxy and the entropy rate. 
The bound applies to TS and the policy $\pi^{\rm Greedy}$, which chooses $A_t \in \argmax_{a \in \mathcal{A}} \mathbb{E}[R_{t,a} \mid \mathcal{F}_{t-1}]$ in each period $t$.

\begin{corollary}\label{cor:full-info}
	For full information problems, where $O_{t,a}=O_{t,a'}$ for each $a,a'\in\mathcal{A}$, we have
	\[
		\bar{\Delta}_T(\pi^{\rm Greedy}) \leq \bar{\Delta}_T(\pi^{\rm TS}) \leq \sigma \sqrt{2 \cdot \bar{H}_T(A^*) },
	\]
	\[
		\text{and} \quad \bar{\Delta}_\infty(\pi^{\rm Greedy}) \leq \bar{\Delta}_\infty(\pi^{\rm TS}) \leq \sigma \sqrt{2 \cdot \bar{H}_\infty(A^*) }.
	\]
\end{corollary}

\subsection{Bounds on the entropy rate} \label{subsec:entropy-rate-bounds}
Our results highlight that the difficulty arising due to the nonstationarity of the environment is sufficiently characterized by the entropy rate of the optimal action process, denoted by $\bar{H}_T(A^*)$ or $\bar{H}_\infty(A^*)$.
We provide some stylized upper bounds on these quantities to aid in their interpretation and to characterize the resulting regret bounds in a comparison with the existing results in the literature.

\paragraph{Bound with the effective time horizon.}

In settings where optimal action process $(A_t^*)_{t \in \mathbb{N}}$ is stationary (\cref{def:stationary}), define $\tau_\textup{eff} := 1/\mathbb{P}( A_t^* \ne A_{t-1}^* )$. 
We interpret $\tau_\textup{eff}$ as the problem's ``effective time horizon'', as it captures the average length of time before the identity of the optimal action changes.
The effective time horizon $\tau_\textup{eff}$ is long when the optimal action changes infrequently, so that, intuitively, a decision-maker could continue to exploit the optimal action for a long time if it were identified, achieving a low regret.
The next proposition shows that the entropy rate $\bar{H}_T(A^*)$ is bounded by the inverse of $\tau_\textup{eff}$, regardless of $T$:

\begin{proposition} \label{prop:effective-horizon-bound} When the process $(A_t^*)_{t\in \mathbb{N}}$ is stationary, 
	\begin{equation} \label{eq:effective-horizon-bound}
		 \bar{H}_T( A^* ) \leq  \frac{1 + \log(\tau_\textup{eff}) + H(A_t^* | A_t^* \ne A_{t-1}^* ) }{\tau_\textup{eff}} + \frac{H(A^*_1)}{T},
	\end{equation}
	for every $T\in \mathbb{N}$, 	where 
	\begin{equation}\label{eq:effective-horizon}
		\tau_\textup{eff} := \frac{1}{\mathbb{P}( A_t^* \ne A_{t-1}^* ) }.
	\end{equation}
\end{proposition}
Combining this result with \cref{cor:k-armed}, and the fact that $\frac{H(A^*_1)}{T} \to 0$ as $T\to \infty$, we obtain
\[
	\bar{\Delta}_\infty(\pi^{\rm TS}) \leq \tilde{O}\left( \sigma\sqrt{ \frac{ |\mathcal{A}| }{ \tau_{\rm eff}} } \right),
\]
which closely mirrors familiar $O(\sqrt{k/T})$ regret bounds on the average per-period regret in bandit problems with $k$ arms, $T$ periods, and i.i.d rewards \citet{bubeck12}, except that the effective time horizon replaces the problem’s raw time horizon.

\begin{example}[Revisiting new article recommendation] \label{ex:news-rec-revisit}
    In \cref{ex:news-rec}, the stochastic process $(\theta_{t})_{t\in \mathbb{N}}$ is stationary as long as $\theta_{1,a} \overset{i.i.d}{\sim} Q$ for each article slot $a\in [k]$. Recall that each article slot is refreshed with probability $1/\tau$, independently. As a result, 
    \[
    \tau_{\rm eff} =\frac{1}{\mathbb{P}( A_t^* \ne A_{t-1}^* ) } \approx \frac{k+1}{2(k-1)} \cdot \tau,
    \]
    and $\bar{\Delta}_\infty(\pi^{\rm TS}) \leq \tilde{O}\left( \sigma\sqrt{ \frac{ k }{ \tau }} \right)$. Per-period regret scales with number of article slots that need to be explored and inversely with the lifetime of an article; the latter determines how long information about an article's click-rate can be exploited before it is no longer relevant. 

    To attain this result, it is critical that our theory depends on the entropy rate or effective horizon of the optimal action process, rather than the parameter process. The vector of click-through rates $\theta_t$ changes every $\tau/k$ periods, since it changes whenever an article is updated. Nevertheless, the optimal action process changes roughy once every $\tau$ periods, on average, since a new article is only optimal  $1/k^{\rm th}$ of the time. 
\end{example}

We now revisit the two-armed bandit example from \cref{subsec:example}.
\begin{example}[Revisiting example in \cref{subsec:example}]\label{ex:two-armed-bandit-revisited}
	Consider the two-armed bandit example discussed in \cref{subsec:example}.
	The idiosyncratic mean reward process of each arm is a stationary zero-mean Gaussian process such that, for each $a \in \{1,2\}$, $\textup{Cov}\left( \theta_{s,a}^{\rm id}, \theta_{t,a}^{\rm id} \right) = \exp\left( -\frac{1}{2}\left( \frac{t-s}{\tau^{\rm id}} \right)^2 \right)$ for any $s, t \in \mathbb{N}$.
	By Rice's formula (\citealp{rice44}; \citealp{barnett1991zero}, eq. (2)) (also known as the zero-crossing rate formula), we have
	\[
		\tau_{\rm eff} = \frac{ \pi }{ \cos^{-1}\left( \exp\left( - \frac{1}{ 2{\tau^{\rm id}}^2} \right) \right) }  
            \approx \pi \cdot \tau^{\rm id}.
	\]
	Consequently, if $\tau^{\rm id} \geq \pi^{-1}$,
	\[
		\bar{H}_\infty( A^* ) \leq  \frac{1 + \log(\pi \cdot \tau^{\rm id}) }{\pi \cdot \tau^{\rm id}}.
	\]
\end{example}

Below we give an example where the upper bound in \cref{prop:effective-horizon-bound} is nearly exact. 
\begin{example}[Piecewise stationary environment] \label{ex:piecewise}
	Suppose $(A_t^*)_{t\in \mathbb{N}}$ follows a switching process. 
	With probability $1-\delta$ there is no change in the optimal action, whereas with probability $\delta$ there is a change-event and $A_t$ is drawn uniformly from among the other $k-1 \equiv|\mathcal{A}|-1$ arms.  
	Precisely, $(A_t^*)_{t\in \mathbb{N}}$ follows a Markov process with transition dynamics:
	\[
	\mathbb{P}(A^*_{t+1}=a \mid  A^*_t=a')= \begin{cases}
		1-\delta  & \text{if } a=a' \\
		\delta/(k-1) & \text{if } a\neq a'
	\end{cases}
	\]
	for $a, a' \in \mathcal{A}$. Then 
	\begin{align*}
		\bar{H}_\infty(A^*) &= (1-\delta)\log\left(\frac{1}{1-\delta}\right) + \delta\log\left(\frac{k-1}{\delta}\right)\\
		&= (1-\delta)\log\left(1 + \frac{\delta}{1-\delta}\right) + \delta\log\left(\frac{k-1}{\delta}\right)\\
		&\approx \delta+\delta\log((k-1)/\delta),
	\end{align*}
	where we used the approximation $\log(1+x) \approx x$. 
	Plugging in $\tau_{\rm eff}=1/\delta$ and $H(A_t^* \mid A_t^*\neq A^*_{t-1})=\log(k-1)$ yields, 
	\[
	\bar{H}_\infty(A^*) \approx \frac{1+\log(\tau_{\rm eff}) + H(A_t^* \mid A_t^*\neq A^*_{t-1})}{\tau_{\rm eff}},
	\]
	which matches the upper bound \eqref{eq:effective-horizon-bound}.
\end{example}

\paragraph{Bound with the switching rate.}
\cref{prop:effective-horizon-bound} requires that the optimal action process is stationary, in the sense of \cref{def:stationary}. We now provide a similar result that removes this restriction. It is stated in terms of the switching rate, which in stationary settings is similar to the inverse of the effective horizon. The added genenerality comes at the expense of replacing the conditional entropy $H(A_t^* \mid A_t^* \neq A_{t-1}^*)$ in \cref{prop:effective-horizon-bound} with a the crude upper bound $\log(|\mathcal{A}|)$.

\begin{proposition}\label{prop:combinatorial-bound}
    Let $\bar{S}_T$ be the expected switching rate of the optimal action sequence, defined as
    \[
        \bar{S}_T := \frac{\mathbb{E}\left[ 1 + \sum_{t=2}^T \mathbb{I}\{A_t^* \ne A_{t-1}^* \} \right]}{T}.
    \]
    Then,
    \[
	\bar{H}_T(A^*) \leq \bar{S}_T \cdot \left( 1 + \log \left( 1 + 1/\bar{S}_T \right) + \log( |\mathcal{A}| ) \right) + \frac{\log T}{T}.
    \]    
\end{proposition}
This results follows from \cref{lem:combinatorial-bound} while slightly generalizing it by bounding the entropy rate with the expected number of switches instead of an almost sure upper bound.
Combining this result with \cref{cor:k-armed} gives a bound
\[
    \bar{\Delta}_T(\pi^{\rm TS}) \leq \tilde{O}\left( \sigma \sqrt{|\mathcal{A}| \cdot \bar{S}_T } \right),
\]
which precisely recovers the well-known results established for switching bandits \citep{auer02a,suk22,yadkori22}.
Although other features of the environment may change erratically, a low regret is achievable if the optimal action switches infrequently.

\paragraph{Bound with the entropy rate of latent state process.}
Although it can be illuminating to consider the number of switches or the effective horizon, the entropy rate is a deeper quantity that better captures a problem's intrinsic difficulty.
A simple but useful fact is that the entropy rate of the optimal action process cannot exceed that of the environment's state process:
\begin{remark}
Since the optimal action $A_t^*$ is completely determined by the latent state $\theta_t$, by data processing inequality,
\[
	\bar{H}_T(A^*) \leq H_T(\theta)
	, \quad
	\bar{H}_\infty(A^*) \leq \bar{H}_\infty(\theta).
\]
\end{remark}
These bounds can be useful when the environment's nonstationarity has predictable dynamics.
The next example illustrates such a situation, in which the entropy rate of the latent state process can be directly quantified.

\begin{example}[System with seasonality]
	Consider a system that exhibits a strong intraday seasonality.
	Specifically, suppose that the system's hourly state (e.g., arrival rate) at time $t$ can be modeled as
	\[
		\theta_t = \xi_{\text{day}(t)} \cdot \mu_{\text{time-of-the-day}(t)} + \epsilon_t,
	\]
	where $(\xi_d)_{d \in \mathbb{N}}$ is a sequence of i.i.d random variables describing the daily random fluctuation, $(\mu_h)_{h \in \{0,\ldots,23\}}$ is a known deterministic sequence describing the intraday pattern, and $(\epsilon_t)_{t \in \mathbb{N}}$ is a sequence of i.i.d random variables describing the hourly random fluctuation.
	Then we have 
	\begin{equation*} 
		\bar{H}_\infty(A^*) \leq \bar{H}_\infty(\theta) = \frac{1}{24}H(\xi) + H(\epsilon),
	\end{equation*}
	regardless of the state-action relationship.
	Imagine that the variation within the intraday pattern $\mu$ is large so that the optimal action changes almost every hour (i.e., $S_T \approx T$ and $\tau_\textup{eff} \approx 1$).
	In this case, the bound like above can be easier to compute and more meaningful than the bounds in \cref{prop:combinatorial-bound,prop:effective-horizon-bound}.
\end{example}

\subsection{Lower bound}\label{sec:lower}

We provide an impossibility result through the next theorem, showing that no algorithm can perform significantly better than the upper bounds provided in Corollary \ref{cor:k-armed} and implied by \cref{prop:effective-horizon-bound}.  Our proof is built by modifying well known lower bound examples for stationary bandits. 

\begin{proposition} \label{thm:lower-bound}
	Let $k > 1$ and $\tau \geq k$.
	There exists a nonstationary bandit problem instance  with $|\mathcal{A}|=k$ and $\tau_{\rm eff}=\tau$, such that
	$$ \inf_{\pi} \bar{\Delta}_\infty(\pi) \geq C \cdot \sigma \sqrt{ \frac{|\mathcal{A}|}{\tau_{\rm eff}} }, $$
    where $C$ is a universal constant.
\end{proposition}

\begin{remark}
	For the problem instance constructed in the proof, the entropy rate of optimal action process is $\bar{H}(A^*) \approx \log(| \mathcal{A} |)/\tau_{\rm eff}$.
	This implies that the upper bound established in \cref{cor:k-armed} is tight up to logarithmic factors, and so is the one established in \cref{thm:main-result}.
\end{remark}

\subsection{Main result}
The corollaries presented earlier are special cases of a general result that we present now. 
Define the (maximal) information ratio of an algorithm $\pi$ by,
\[
	\Gamma(\pi) := \sup_{t \in \mathbb{N}}  \underbrace{ \frac{ \left( \mathbb{E}_\pi\left[ R_{t,A_t^*} - R_{t,A_t} \right] \right)^2 }{ I\left( A_t^*;  (A_t, O_{t,A_t})  \mid \mathcal{F}_{t-1} \right) } }_{=: \Gamma_t(\pi)},
\]
The per-period information ratio $\Gamma_t(\pi)$ was defined by \citet{russo16} and presented in this form by \citet{russo22}. 
It is the ratio between the square of expected regret and the conditional mutual information between the optimal action and the algorithm's observation. 
It measures the cost, in terms of the square of expected regret, that the algorithm pays to acquire each bit of information about the optimum. 

The next theorem shows that any algorithm's per-period regret is bounded by the square root of the product of its information ratio and the entropy rate of the optimal action sequence.
The result has profound consequences, but follows easily by applying elegant properties of information measures. 
\begin{theorem} \label{thm:main-result}
	Under any algorithm $\pi$,
	\[
		\bar{\Delta}_T(\pi) \leq \sqrt{ \Gamma(\pi) \cdot \bar{H}_T(A^*)  },
		\quad \text{and} \quad
		\bar{\Delta}_\infty(\pi) \leq \sqrt{ \Gamma(\pi) \cdot \bar{H}_\infty(A^*) }.
	\]
	
\end{theorem}
\begin{proof} Use the shorthand notation $\Delta_t := \mathbb{E}[ R_{t,A_t^*} - R_{t,A_t} ]$ for regret, $G_t := I( A_t^*; (A_t, O_{t,A_t}) | \mathcal{F}_{t-1} )$ for information gain, and $\Gamma_t = \Delta_t^2/G_t$ for the information ratio at period $t$. 
Then, 
	\[
		\sum_{t=1}^T \Delta_t
			= \sum_{t=1}^T \sqrt{\Gamma_t} \sqrt{G_t}
			\leq \sqrt{ \sum_{t=1}^T \Gamma_t } \sqrt{ \sum_{t=1}^T G_t },
	\]
	by Cauchy-Schwarz, and 
	\[
		\sqrt{ \sum_{t=1}^T \Gamma_t } \sqrt{ \sum_{t=1}^T G_t } \leq \sqrt{ \Gamma(\pi) \cdot T \cdot \sum_{t=1}^T G_t }.
	\]
	by definition of $\Gamma(\pi)$.
	We can further bound the information gain. 
	This uses the chain rule, the data processing inequality, and the fact that entropy bounds mutual information:
	\begin{align*}
		\sum_{t=1}^T G_t = \sum_{t=1}^T I( A_t^*; (A_t, O_{t,A_t}) | \mathcal{F}_{t-1} )
		\leq \sum_{t=1}^T I( [A_1^*,\ldots,A_T^*] ; (A_t, O_{t,A_t}) | \mathcal{F}_{t-1} )
		&= I(  [A_1^*,\ldots,A_T^*] ; \mathcal{F}_T ) \\
		& \leq H(  [A_1^*,\ldots,A_T^*] ).
	\end{align*}
	Combining these results, we obtain 
	\begin{align*}
		\bar{\Delta}_T(\pi) 
			\leq \frac{ \sqrt{ \Gamma(\pi) \cdot T \cdot H([A_1^*,\ldots,A_T^*]) } }{T}
			  = \sqrt{ \Gamma(\pi) \cdot \bar{H}_T(A^*) }.
	\end{align*}
	Taking limit yields the bound on regret rate $\bar{\Delta}_\infty(\pi)$.
\end{proof}
\begin{remark}%
	A careful reading of the proof reveals that it is possible to replace the entropy rate $\bar{H}_T(A^*)$ with the mutual information rate $\bar{I}_T(A^*; \mathcal{F})$.
\end{remark}

\begin{remark}\label{rem:lambda-info-ratio}
	Following \citet{lattimore21}, one can generalize the definition of information ratio,
	$$ \Gamma_\lambda(\pi) := \sup_{t \in \mathbb{N}} \frac{ \left( \mathbb{E}_\pi\left[ R_{t,A_t^*} - R_{t,A_t} \right] \right)^\lambda }{ I\left( A_t^*;  (A_t, O_{t,A_t})  \mid \mathcal{F}_{t-1} \right) }, $$
	which immediately yields an inequality, $\bar{\Delta}_T(\pi) \leq \left( \Gamma_\lambda(\pi) \bar{H}_T(A^*) \right)^{1/\lambda}$ for any $\lambda \geq 1$.
\end{remark}

\section{Bounds on the Information Ratio}\label{sec:information-ratio-bounds-old-and-new}

\subsection{Some known bounds on the information ratio} \label{sec:information-ratio-bounds}

We list some known results about the information ratio. These were originally established for stationary bandit problems but immediately extend to nonstationary settings considered in this paper. Most results apply to Thompson sampling, and essentially all bounds apply to Information-directed sampling, which is designed to minimize the information ratio \citep{russo18}. The first four results were shown by \citet{russo16} under \cref{ass:subgaussian}.
\begin{description}
	\item[Classical bandits.] $\Gamma(\pi^{\rm TS}) \leq 2\sigma^2|\mathcal{A}| $, for bandit tasks with a finite action set (e.g., \cref{ex:k-armed-bandit}).
	\item[Full information.] $\Gamma(\pi^{\rm TS}) \leq 2\sigma^2$, for problems with full-information feedback (e.g., \cref{ex:full-information}).
	\item[Linear bandits.] $\Gamma(\pi^{\rm TS}) \leq 2\sigma^2 d$, for linear bandits of dimension $d$ (i.e., $\mathcal{A} \subseteq \mathbb{R}^d$, $\Theta \subseteq \mathbb{R}^d$, and $\mathbb{R}[R_{t,a} | \theta_t] = a^\top \theta_t$).
	\item[Combinatorial bandits.] $\Gamma(\pi^{\rm TS}) \leq 2\sigma^2 \frac{d}{k^2}$, for combinatorial optimization tasks of selecting $k$ items out of $d$ items with semi-bandit feedback (e.g., \cref{ex:matching}). 
	\item[Contextual bandits.] See the below for a new result.  
	\item[Logistic bandits.] \citet{dong19} consider problems where mean-rewards follow a generalized linear model with logistic link function, and bound the information ratio by the dimension of the parameter vector and a new notion they call the `fragility dimension.'
	\item[Graph based feedback.]  With graph based feedback, the decision-maker observes not only the reward of selected arm but also the reward of its neighbors in feedback graph. One can bound the information ratio by the feedback graph's clique cover number \cite{liu18} or its independence number  \cite{hao2022contextual}.
	\item[Sparse linear models.] \citet{hao21} consider sparse linear bandits and show conditions under which the information ratio of Information-Directed Sampling in \cref{rem:lambda-info-ratio} is bounded by the number of nonzero elements in the parameter vector. 
	\item[Convex cost functions.] \citet{bubeck2016multi} and \citet{lattimore2020improved} study bandit learning problems where the reward function is known to be concave and bound the information ratio by a polynomial function of the dimension of the action space. 
\end{description}

\subsection{A new bound on the information ratio of contextual bandits}\label{subsec:contextual}
Contextual bandit problems are a  special case of our formulation that satisfy the following abstract assumption. 
Re-read \cref{ex:contextual-bandit} to get intuition. 
\begin{assumption}\label{assumption:contextual}
	There is a set $\mathcal{X}$ and an integer $k$ such that $\mathcal{A}$ is the set of functions mapping $\mathcal{X}$ to $[k]$. The observation at time $t$ is the tuple $O_t =(X_t, R_t) \in \mathcal{X}\times \mathbb{R}$.  
	Define $i_t := A_t(X_t)\in [k]$.	
	Assume that for each $t$,  $X_{t+1} \perp (A_t, R_t) \mid (X_t, \mathcal{F}_{t-1})$, 
	and $R_t \perp A_t \mid (X_t, i_t, \theta_t).$
\end{assumption}

Under this assumption, we provide an information ratio bound that depends on the number of arms $k$. 
It is a massive improvement over \cref{cor:k-armed}, which depends on the number of \emph{decision-rules}. 
\begin{lemma}\label{lem:contextual}
	Under Assumption \ref{assumption:contextual}, $\Gamma(\pi^{\rm TS}) \leq 2\cdot \sigma^{2} \cdot k$. 
\end{lemma}
Theorem \ref{thm:main-result} therefore bounds regret in terms of the entropy rate of the optimal decision rule process $(A^*_t)_{t\in \mathbb{N}}$, the number of arms $k$, and the reward variance proxy $\sigma^2$.  

\citet{neu22} recently highlighted that information-ratio analysis seems not to deal adequately with context, and proposed a substantial modification which considers information gain about model parameters rather than optimal decision-rules.  
\cref{lem:contextual} appears to resolve this open question without changing the information ratio itself. 
Our bounds scale with the entropy of the optimal decision-rule, instead of the entropy of the true model parameter, as in \citet{neu22}. 
By the data processing inequality, the former is always smaller. 
Our proof bounds the per-period information ratio, so it can be used to provide finite time regret bounds for stationary contextual bandit problems. 
\citet{hao2022contextual} provide an interesting study of variants of Information-directed sampling in contextual bandits with complex information structure. 
It is not immediately clear how that work relates to \cref{lem:contextual} and the information ratio of Thompson sampling. 
	
The next corollary combines the information ratio bound above with the earlier bound of \cref{prop:effective-horizon-bound}. 
The bound depends on the number of arms, the dimension of the parameter space, and the effective time horizon. 
No further structural assumptions (e.g., linearity) are needed. 
An unfortunate feature of the result is that it applies only to parameter vectors that are quantized at scale $\epsilon$. 
The logarithmic dependence on $\epsilon$ is omitted in the $\tilde{O}(\cdot)$ notation, but displayed in the proof. 
When outcome distributions are smooth in $\theta_t$,  we believe this could be removed with careful analysis. 
\begin{corollary}\label{cor:contextual} Under \cref{assumption:contextual}, 
	if $\theta_t \in \{ -1, -1+\epsilon, \ldots, 1-\epsilon, 1 \}^{p}$ is a discretized $p$-dimensional vector, and the optimal policy process $(A^*_t)_{t\in \mathbb{N}}$ is stationary, then 
	\[
	\bar{\Delta}_\infty(\pi^{\rm TS}) \leq \tilde{O}\left( \sigma \sqrt{ \frac{p \cdot k}{\tau_\textup{eff}} } \right).
	\]
\end{corollary}

\section{Extension 1: Satisficing in the Face of Rapidly Evolving Latent States}\label{sec:sts}
In  nonstationary learning problems with short effective horizon, the decision-maker is continually uncertain about the latent states and latent optimal action. Algorithms like Thompson sampling explore aggressively in an attempt to resolve this uncertainty. But this effort may be futile, and whatever information they do acquire may have low  value since it cannot be exploited for long. 

Faced with rapidly evolving environments, smart algorithms should \emph{satisfice}. As noted by Herbert Simon (\cite{simon1979rational}, p. 498) when receiving his Nobel Prize, ``decision makers can satisfice either by finding optimum solutions for a simplified world, or by finding satisfactory solutions for a more realistic world.'' We focus on the latter case, and decision-makers that realistically model a rapidly evolving environment but seek a more stable decision-rule with satisfactory performance. 

We introduce a broad generalization of TS that performs probability matching with respect to a latent \emph{satisficing} action sequence instead of the latent optimal action sequence. Our theory gracefully generalizes to this case, and confirms that the decision-maker can attain rewards competitive with any satisficing action sequence whose entropy rate is low. Whereas earlier results were meaningful only if the optimal action sequence had low entropy rate --- effectively an assumption on the environment --- this result allows the decision-maker to compete with any low entropy action sequence regardless of the extent of nonstationarity in the true environment. We provide some implementable examples of satisficing action sequences in \cref{subsec:satisficing-examples}.

One can view this section as a broad generalization of the ideas in  \cite{russo22}, who explore the connection between satisficing, Thompson sampling, and information theory in stationary environments. It also bears a close conceptual relationship to the work of \cite{liu23}, which is discussed in detail in \cref{app:predictive-sampling}. In the adversarial bandit literature, the most similar work is that of \citet{auer02a}. They show a properly tuned exponential weighting algorithm attains low regret relative to any sequence of actions which switches infrequently.

\subsection{Satisficing regret: competing with a different benchmark}
In our formulation, the latent optimal action process $A^*=(A^*_t)_{t\in \mathbb{N}}$ serves as both a \emph{benchmark} and a \emph{learning target}. 
When calculating regret, the rewards a decision maker accrues are compared against the rewards $A^*$ would have accrued, using $A^*$ as a benchmark.
When defining TS, we implicitly direct it to resolve uncertainty about the latent optimal action $A^*$ through probability matching, using $A^*$ as a learning target.
When the environment evolves rapidly and unpredictably, it is impossible to learn about changes in the latent state quickly enough to exploit them, as $A^*$ does in such cases, the latent optimal action is not a reasonable benchmark or learning target. 

We introduce a \emph{satisficing action sequence} $A^\dagger = \big( A_t^\dagger \big)_{t \in \mathbb{N}}$ that serves as an alternative benchmark and learning target. We want the optimal action sequence $A^*=(A^*_t)_{t\in \mathbb{N}}$ to be a feasible choice of satisficing action sequence, so we must allow $A^{\dagger}$ to be a function of $\theta$. The next definition also allows for randomization in the choice. For now, this definition is very abstract, but we give specific examples in Section \ref{subsec:satisficing-examples}.
\begin{definition}\label{def:satisficing_action}
	A collection of random variables $A^\dagger = \big( A_t^\dagger \big)_{t \in \mathbb{N}}$ taking values in $\mathcal{A}^{\infty}$ is a satisficing action sequence if there exists a function $f(\cdot)$ and an exogenous random variable $\xi$ (independent of $\theta$, $W$ and $\tilde{W}$) such that $A^\dagger = f(\theta, \xi)$. 
\end{definition}

Replacing $A^*$ with $A^\dagger$ as a learning target,  we modify Thompson sampling so its exploration aims to resolve uncertainty about $A^\dagger$, which could be much simpler. \emph{Satisficing Thompson sampling} (STS) \citep{russo22} with respect to this satisficing action sequence is denoted by $\pi^\dagger$. It instantiates the probability matching with respect to $A^\dagger$ so that, under $\pi^\dagger$,
\begin{equation*}\label{eq:sts}
	\mathbb{P}(A_t=a \mid \mathcal{F}_{t-1})=\mathbb{P}(A_t^\dagger=a \mid \mathcal{F}_{t-1}),
\end{equation*}
for all $t \in \mathbb{N}$ and $a \in \mathcal{A}$.

Replacing $A^*$ with $A^\dagger$ as a benchmark, we define the regret rate of an action sequence $A=(A_t)_{t\in \mathbb{N}}$ (i.e., some sequence of non-anticipating random variables) with respect to a satisficing action sequence $A^{\dagger}$ to be
\[
\bar{\Delta}_{T}(A ; A^\dagger) := \frac{1}{T} \E\left[ \sum_{t=1}^{T} R_{t,A_t^\dagger}  - R_{t, A_t} \right], \qquad  \bar{\Delta}_{\infty}(A ; A^\dagger) := \limsup_{T \rightarrow \infty}\bar{\Delta}_{T}(A ; A^\dagger) . 
\] 
Each policy $\pi$ induces an action sequence $A$, so we can overload notation to write $\bar{\Delta}_{T}(\pi; A^\dagger)$.

\cref{thm:main-result} is easily generalized to bound satisficing regret.  Define the information ratio with respect to $A^\dagger$:
\begin{equation} \label{eq:information-ratio-sts}
	\Gamma(\pi;A^\dagger) := \sup_{t \in \mathbb{N}} \frac{ \left( \mathbb{E}_\pi\left[ R_{t,A_t^\dagger} - R_{t,A_t} \right] \right)^2 }{ I\left( A_t^\dagger; (A_t, O_{t,A_t}) \mid \mathcal{F}_{t-1} \right) }.
\end{equation}
It represents the (squared) cost that the algorithm $\pi$ needs to pay in order to acquire one unit of information about the satisficing action sequence $A^\dagger$.

Replicating the proof of \cref{thm:main-result} with respect to $A^\dagger$ yields a bound on regret relative to this alternative benchmark in terms of the information ratio with respect to this alternative benchmark. Intuitively, less information is needed to identify simpler satisficing action sequences, reducing the $\bar{I}_{\infty}(A^\dagger ; \theta)$ term in this bound. 
\begin{theorem} \label{thm:main-result-sts}
	For any algorithm $\pi$ and satisficing action sequence $A^{\dagger}$, 
	\[
	\bar{\Delta}_{\infty}(\pi; A^{\dagger}) \leq \sqrt{ \Gamma(\pi; A^{\dagger})\times \bar{I}_{\infty}(A^\dagger ; \theta)}.
	\]
	If the satisficing action sequence satisfies $A^{\dagger}=f(\theta_{1:T}, \xi)$ for some function $f(\cdot)$, then 
	\[
	\bar{\Delta}_T(\pi; A^{\dagger}) \leq \sqrt{ \Gamma(\pi; A^\dagger) \times \bar{I}_T(A^\dagger_{1:T} ; \theta_{1:T}) }. 
	\]
\end{theorem}
Interpret the mutual information $\bar{I}_T(A^\dagger  ; \theta  ) $ as the rate at which bits about $\theta=\big( \theta_t \big)_{t \in \mathbb{N}}$ must be communicated in order to implement the changing decision rule $A^{\dagger}$. Interpret $\Gamma(\pi; A^\dagger)$ as the price (in terms of squared regret) that policy $\pi$ pays per bit of information acquired about $A^{\dagger}.$ As made clear in the next remark, satisficing TS controls the price per bit of information $\Gamma(\pi; A^\dagger)$ regardless of the choice of satisficing action sequence $A^{\dagger}$. \cref{thm:main-result-sts} therefore offers much stronger guarantees than \cref{thm:main-result} as whenever $I(A^{\dagger}; \theta) \lll H(A^*)$, i.e. whenever the the satisficing action sequence can be implemented while acquiring much less information about the latent states. 
\begin{remark} \label{rem:information-ratio-bounds-sts}
	All the upper bounds on $\Gamma(\pi^\text{TS};A^*)$ listed in \cref{sec:information-ratio-bounds} also apply for $\Gamma(\pi^\dagger;A^\dagger)$.
	Namely, for any satisficing action sequence $A^\dagger$, the corresponding STS $\pi^\dagger$ satisfies (a) $\Gamma(\pi^\dagger; A^\dagger) \leq 2\sigma^2 |\mathcal{A}| $ for bandit tasks with finite action set, (b) $\Gamma(\pi^\dagger; A^\dagger) \leq 2\sigma^2 $ for problems with full-information feedback, (c) $\Gamma(\pi^\dagger; A^\dagger) \leq 2\sigma^2 d $ for linear bandits, (d) $\Gamma(\pi^\dagger; A^\dagger) \leq 2\sigma^2 d/k^2 $ for combinatorial bandits, and (e) $\Gamma(\pi^\dagger; A^\dagger) \leq 2\sigma^2 k $ for contextual bandits. 
	
	Therefore it satisfies similar bounds on satisficing regret. For linear bandits, 
	\[
	\bar{\Delta}_{\infty}(\pi^{\dagger}; A^{\dagger}) \leq \sigma\sqrt{ 2d\times \bar{I}_{\infty}(A^\dagger ; \theta)}.
	\]
\end{remark}

\subsection{Examples of satisficing action sequences}\label{subsec:satisficing-examples}
The concept introduced above is quite abstract. Here we introduce three examples. Each aims to construct a satisficing action sequence that is near optimal (i.e., low $\bar{\Delta}_{\infty}(A^{\dagger}; A^*))$ while requiring limited information about $\theta$ (i.e., low $\bar{I}_{\infty}(A^{\dagger}; \theta)$).  As a warmup, we visualize in \cref{fig:example-satisficing} these examples of satisficing actions; they switch identity infrequently, ensuring they have low entropy, but are nevertheless very nearly optimal.

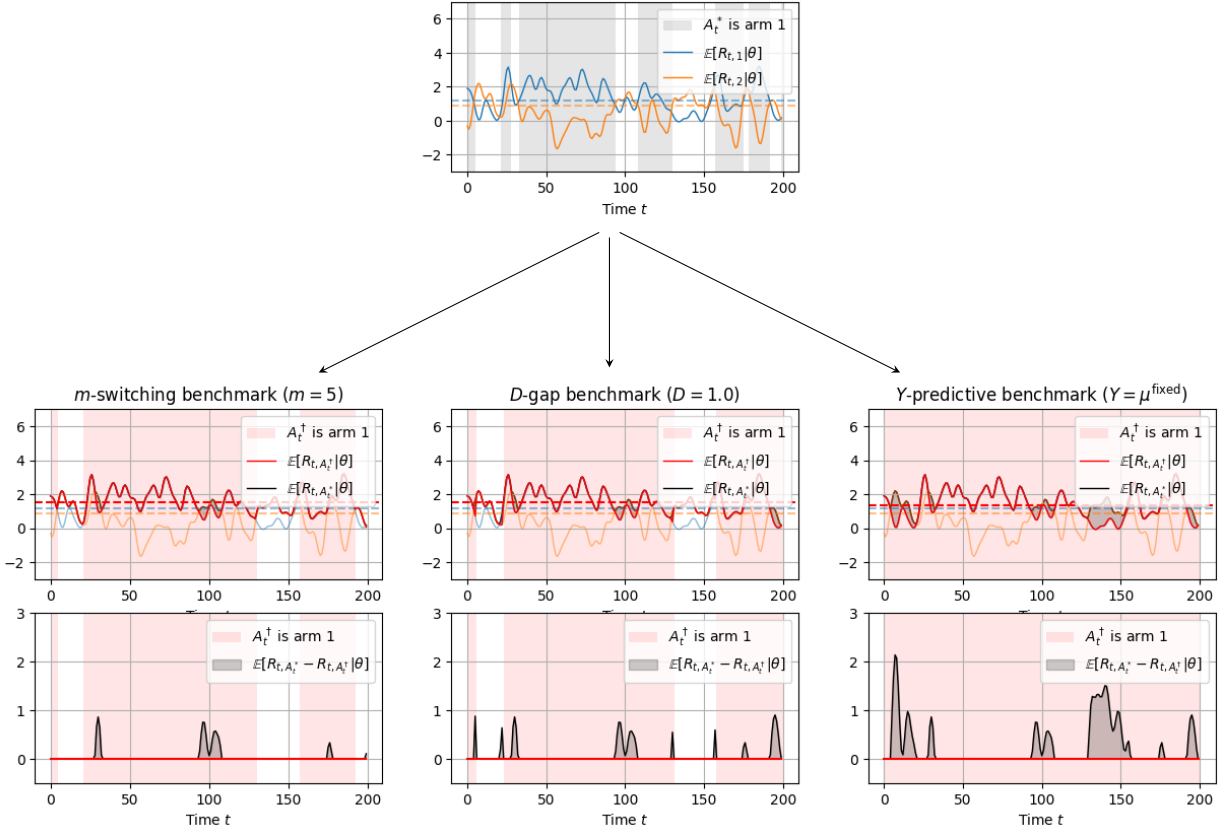
\begin{figure}[ht]
	\centering
 \begin{tikzpicture}
		\node at (0,0) {\includegraphics[width=.99\linewidth]{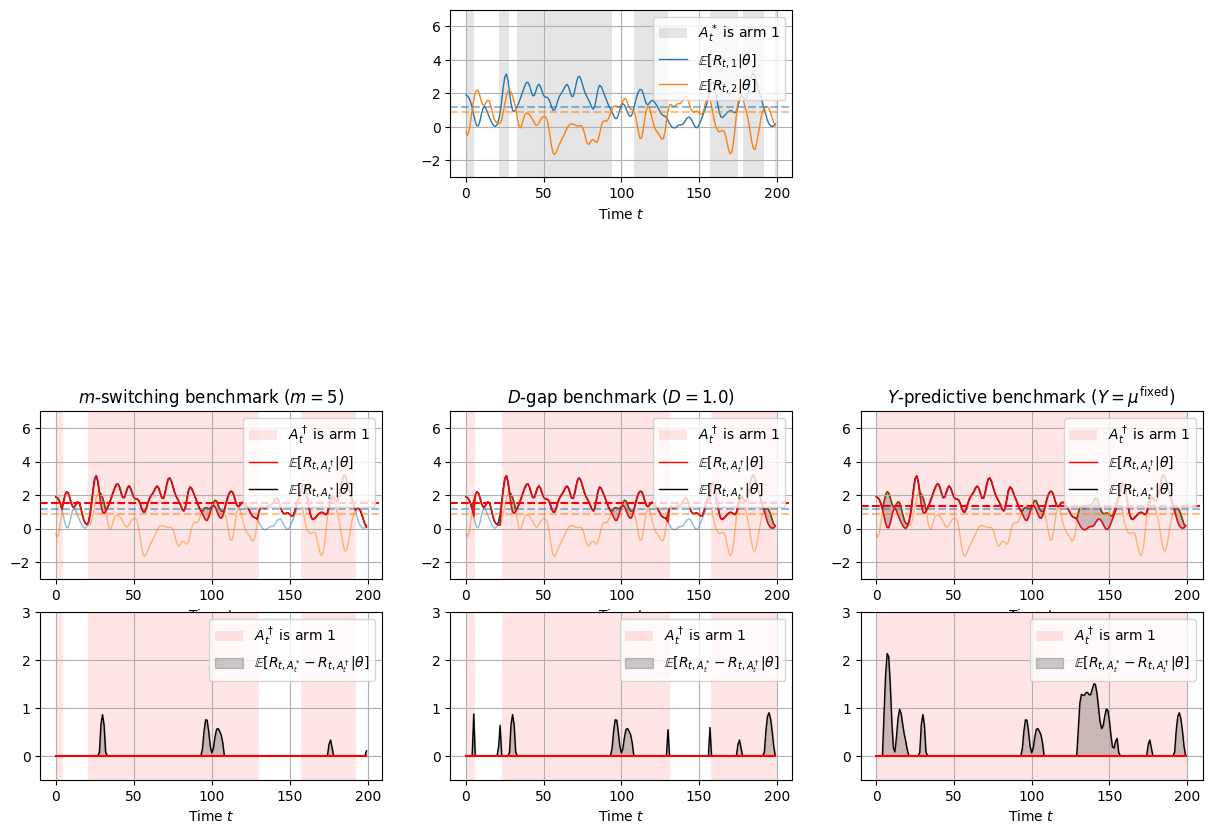}};
            \node at (0,2.5) (origin) {};
            \node at (-4,0.5) (ex1) {};
            \node at (0,0.5) (ex2) {};
            \node at (4,0.5) (ex3) {};
            \draw [->] (origin) -- (ex1);
            \draw [->] (origin) -- (ex2);
            \draw [->] (origin) -- (ex3);
 \end{tikzpicture}
		\caption{
                Three examples of satisficing actions in a two-armed bandit environment described below \cref{ex:sts-obscuring-info}.
                The top figure shows the performance of two arms, as visualized in \cref{fig:example-illustration}, where shaded vs. unshaded regions indicate which action is optimal at each time.
                For each choice of alternative benchmarks (suggested in \cref{ex:sts-restricting-switches,ex:sts-ignoring-small-suboptimality,ex:sts-obscuring-info}), we show two figures stacked vertically.
                The above ones plot the performance of alternative action sequences that are more stable, where shaded vs unshaded regions indicate which action the benchmark chooses.
                The below ones plot the suboptimality of these alternative benchmarks.
                    }
		\label{fig:example-satisficing}
\end{figure}

Our first example is the optimal action sequence among those that switch identity no more than $m$ times within $T$ periods. 
\begin{example}[Restricting number of switches]\label{ex:sts-restricting-switches} Define the number of switches in a sequence $a_{1:T} \in \mathcal{A}^T$ by  $\mathcal{S}_T(a_{1:T}) := 1 +\sum_{t=2}^{T} \mathbb{I}\{a_t \neq a_{t-1}\}$, a measure used in \eqref{eq:switching-rate}. Let
	\[
	(A^\dagger_1, \ldots, A^\dagger_T) \in \argmax_{a_{1:T}\in \mathcal{A}^T: \mathcal{S}_T(a_{1:T}) \leq m}  \frac{1}{T} \sum_{t=1}^{T}   \E\left[ R_{t,a_t} \mid \theta_t \right].
	\]
	be the best action sequence with fewer than $m$ switches. 
	Satisficing regret $\bar{\Delta}_{T}(\pi^\dagger; A^\dagger)$ measures the gap between the rewards its actions generate and the true best action sequence with few switches. By \cref{lem:combinatorial-bound}, one can bound the accumulated bits of information about this satisficing action sequence as:
	\[
	\bar{I}_{T}(A^{\dagger}_{1:T}; \theta) \leq 	\bar{H}_{T}(A^{\dagger}_{1:T}) \leq \frac{m}{T} \cdot \left( 1 + \log \left( 1 + \frac{T}{m} \right) + \log( |\mathcal{A}| ) \right).
	\]	
	Satisficing Thompson sampling can be implemented by at time $t$ drawing a sequence 
	\[ 
	\tilde{A}_{1:T} \in \argmax_{a_{1:T}\in \mathcal{A}^T: \mathcal{S}_T(a_{1:T}) \leq m}  \frac{1}{T} \sum_{\ell=1}^{T}   \E\left[ R_{\ell,a_t} \mid \theta_\ell = \tilde{\theta}_\ell \right]  \quad \text{where} \quad 
	(\tilde{\theta}_1, \ldots, \tilde{\theta}_T) \sim \mathbb{P}(\theta_{1:T} \in \cdot \mid \mathcal{F}_{t-1}). 
	\]
	and then picking $A_t = \tilde{A}_t$. This samples an arm according to the posterior probability it is the arm chosen in the true best action sequence with $m$ switches. A simple dynamic programming algorithm can be used to solve the optimization problem defining $\tilde{A}_{1:T}$. 
\end{example}
The next example also tries to create a more reasonable benchmark that switches less frequently. Here though, we define a satisficing action sequence that switches to a new arm only once an arm's subotimality exceeds some threshold. This leads to a version of satisficing Thompson sampling that has a simpler implementation.  
\begin{example}[Ignoring small suboptimality]\label{ex:sts-ignoring-small-suboptimality}
	Define a sequence of actions as follows:
	\begin{align*}
		A^\dagger_t = \begin{cases}
			A^\dagger_{t-1}  & \text{ if }  \E[R_{t, A_{t-1}^\dagger} \mid \theta_t] \geq \E[\max_{a \in \mathcal{A}}R_{t,a} \mid \theta_t] -D,\\
			\argmax_{a \in \mathcal{A}} \E[R_{t,a} \mid \theta_t] & \text{otherwise}.
		\end{cases}
	\end{align*}
	This sequence of arms switches only when an arm's suboptimality exceeds a distortion level $D>0$. The entropy of $A^\dagger$ will depend on the problem, but we will later bound it in problems with ``slow variation'' as in \cite{besbes14}. Algorithm \ref{alg:sts-ignoring-small-suboptimality} provides an implementation of satisficing Thompson sampling with respect to $A^\dagger$. 
\end{example}
\SetKwComment{Comment}{/* }{ */}
\begin{algorithm}
	\caption{Satisficing Thompson sampling in Example \ref{ex:sts-ignoring-small-suboptimality}}\label{alg:sts-ignoring-small-suboptimality}
	\KwIn{Distortion level $D>0$.}
	Define $\mu_{\theta_t}(a) = \E[R_{t,a} \mid \theta_t]$\;
	
	\For{$t=1,2,\ldots$}{
		Sample $(\tilde{\theta}_1, \ldots, \tilde{\theta}_t) \sim \mathbb{P}\left( (\theta_1, \ldots, \theta_t) \in \cdot \mid \mathcal{F}_{t-1}\right)$\;
  
		$A^{\dagger}_1 = \argmax \mu_{\tilde{\theta}_1}(a)$\;
  
        \For{$s=2,\ldots, t-1$}{
        \eIf{$\mu_{\tilde{\theta}_s}(A^{\dagger}_{s-1}) \geq \max_{a \in \mathcal{A}} \mu_{\tilde{\theta}_{s}}(a) - D$}{
			$A_s^\dagger \gets A^\dagger_{s-1}$\; 
		}{ $A^\dagger_{s} \gets \argmax_{a \in \mathcal{A}} \mu_{\tilde{\theta}_s}(a)$\;
		}
        }		
		Play $A_t=A_t^\dagger$, observe $O_t$ and update history\;
	}
\end{algorithm}

The next example induces satisficing in a very different way. Instead of altering how often the benchmark optimal action can change, it restricts the granularity of information about $\theta$ on which it can depend.  

\begin{example}[Optimizing with respect to obscured information]\label{ex:sts-obscuring-info}
	Let 
	\[
	A^\dagger_t = \argmax_{a \in \mathcal{A}} \E[R_{t,a}\mid Y] \quad \text{where} \quad Y = h( \theta, \xi ),
	\]
	for some function $h(\cdot)$. Since $A^\dagger - Y - \theta$ forms a Markov chain, choosing a $Y$ that reveals limited information about $\theta$ restricts the magnitude of the mutual information $\bar{I}_{\infty}(A^{\dagger}; \theta)$. 
 
	To implement satisficing TS, one needs to draw from the posterior distribution of  $A_t^\dagger$. Since $A_t^\dagger \in \argmax_{a\in \mathcal{A}} \mu_{t,a}(Y)$ where  $\mu_t(Y) := \E[ R_{t,a} \mid Y]$, we can sample from the posterior just as we do in TS: 
	one needs to draw a sample $\tilde{\mu}_{t} \sim \mathbb{P}( \mu_{t}(Y) \in \cdot \mid \mathcal{F}_{t-1})$ from the posterior distribution of $\mu_{t}(Y)$ an then pick $A_t \in \argmax_{a \in \mathcal{A}} \tilde{\mu}_{t,a}$. 

    There are two natural strategies for providing obscured information $Y$ about the latent states $\theta$. First, one can provide noise corrupted view, like $Y= \theta + \xi$ where $\xi$ is mean-zero noise. For stationary problems, this is explored in \cite{russo22}. For the $k$-armed nonstationary bandits in \cref{ex:k-armed-bandit}, providing fictitious reward observations $Y=(\theta_{t,a} + \xi_{t,a})_{t\in \mathbb{N}, a\in \mathcal{A}}$, where $\xi_{t,a}\sim N(0,\sigma^2)$, is similar to the proposed algorithm in \cite{liu23}. The second approach one can take is to reveal some simple summary statistics of $\theta$. 
    
    We describe a simple illustration of the potential benefits of the second approach. Consider a problem with bandit feedback (i.e. $O_{t,a}= R_{t,a}$) and Gaussian reward observations; that is, 
    \[
        R_{t,a} = \theta_{t,a}+W_{t,a}, \quad \theta_{t,a} = \mu_a^\text{fixed} + \mu_{t,a}^\text{GP},
    \]
    where $W_{t,a} \sim \mathcal{N}(0,\sigma^2)$ is i.i.d Gaussian noise, $\mu_a^\text{fixed} \sim \mathcal{N}(0, v^2)$, and $\mu_a^\text{GP} = (\mu_{t,a}^\text{GP})_{t \in \mathbb{N}}$ follows a Gaussian process such that $\textup{Cov}(\mu_{t,a}^\text{GP}, \mu_{s,a}^\text{GP}) = \exp\left( -\frac{1}{2} \left( \frac{t-s}{\tau} \right)^2 \right)$. 
    
    We construct a variant of satisficing TS that should vastly outperform regular TS when $\tau \approx 0$. 
    Choose $Y$ as
    \[
        Y = \big( \mu_a^\text{fixed} \big)_{a \in \mathcal{A}}.
    \]
    Since $\mathbb{E}[ R_{t,a} | Y ] = \mu_a^\text{fixed}$, we have $A_1^\dagger = A_2^\dagger = \ldots = \argmax_{a \in \mathcal{A}} \mu_a^\text{fixed}$, and satisficing TS works as follows.
	\begin{description}
		\item[Satisficing TS:] Sample $\tilde{\mu}_{t,a} \sim \mathcal{N}\left( \E[ \mu_a^\text{fixed} \mid \mathcal{F}_{t-1} ]  \, , \, {\rm Var}\left( \mu_a^\text{fixed} \mid \mathcal{F}_{t-1}\right) \right)$ and pick $A_t \in \argmax_{a \in \mathcal{A}} \tilde{\mu}_{t,a}$. 
		\item[Regular TS:]	 Sample $\tilde{\mu}_{t,a} \sim \mathcal{N}\left( \E[ \mu_a^\text{fixed} + \mu_{t,a}^\text{GP} \mid \mathcal{F}_{t-1} ]  \, , \, {\rm Var}\left( \mu_a^\text{fixed} + \mu_{t,a}^\text{GP}  \mid \mathcal{F}_{t-1}\right)\right)$ and pick $A_t \in \argmax_{a \in \mathcal{A}} \tilde{\mu}_{t,a}$. 
 	\end{description}	
    To understand the difference between these algorithms, consider the limiting regimes in which $\tau \to \infty$ and $\tau \to 0$. In the first case, environment dynamics are so slow that the problems looks like a statationry $k$-armed bandit. 
    
    When $\tau \to 0$, the latent states evolve so rapidly and erratically that the problem looks is again equivalent to a stationary problem. Precisely, as $\tau \to 0$, $\textup{Cov}[\mu_{t,a}^\text{GP}, \mu_{s,a}^\text{GP}] \to \mathbb{I}\{t = s\}$ for any $t,s \in \mathbb{N}$; the process $(\mu_{t,a}^\text{GP})_{t \in \mathbb{N}}$ becomes a white noise process, i.e., behaves almost like a sequence of i.i.d. random variables distributed with $\mathcal{N}(0,1^2)$.
    The overall problem looks like a stationary $k$-armed Gaussian bandit with prior $\mathcal{N}(0,v^2)$ and i.i.d. noise with law $\mathcal{N}(0,\sigma^2 + 1^2)$.

    Satisficing TS is similar in the $\tau \to 0$ and $\tau \to \infty$ limits; in ether case it attempts to identify the arm with best performance throughout the horizon.  But regular TS behaves very differently in these two limits -- making its behavior seemingly incoherent. As $\tau \to \infty$, regular TS coincides with satisficing TS. But as $\tau \to 0$, the samples maximized by regular TS always have high variance, causing the algorithm to explore in a futile attempt to resolve uncertainty about the white noise process $(\mu_{t,a}^\text{GP})_{t \in \mathbb{N}}$.
\end{example}

\section{Extension 2: Generalizing the Entropy Rate with Rate Distortion Theory }\label{sec:rate-distortion} 

Our main result in Theorem \ref{thm:main-result} bounds regret in a large class of sequential learning problems in terms of the entropy rate of the optimal action process --- linking the theory of exploration with the theory of optimal lossless compression. In this section, we replace the dependence of our bounds on the entropy rate with dependence on a \emph{rate distortion function}, yielding our most general and strongest bounds. The section result is quite abstract, but reveals  an intriguing connection between interactive decision-making and compression. 

 As an application, we bound the rate distortion function for environments in which the reward process has low total variation rate. This recovers regret bounds that are similar to those in the influential work of \cite{besbes15}, and greatly strengthens what Theorem \ref{thm:main-result} guarantees under this assumption.
 
While this section is abstract,  the proofs follow easily  from attempting to instantiate \cref{thm:main-result-sts} with the best possible choice of satisficing action sequence. Like \cref{thm:main-result-sts}, the results generalize \cite{russo22} to nonstationary environments.

\subsection{Rate distortion and lossy compression}
To derive a rate distortion function, we momentarily ignore the sequential learning problem which is our focus. Consider a simpler communication problem. Imagine that some observer knows the latent states of the environment; their challenge is to encode this knowledge as succinctly as possible and transmit the information to a downstream decision-maker.

 As a warmup, \cref{subfig:lossy_state_estimation} depicts a more typical problem arising in the theory of optimal communication, in which the goal is to succinctly encode latent states themselves. 
 Assuming the law of the latent states $(\theta_1, \theta_2, \ldots)$ is known, one wants to design a procedure that efficiently communicates information about the realization across a noisy channel. Having observed $\theta_{1:T} = (\theta_1, \theta_2, \ldots, \theta_T)$, the observer transmits some signal $f_T(\theta_{1:T})$ to a decoder, which recovers an estimate $\hat{\theta}_{1:T}$. The encoder is said to communicate at bit-rate $n$ if $f_{T}:\Theta^T \to \{0,1\}^{n\cdot T}$. While $\theta_{1:T}$ could be very complex, $f_T(\theta_{1:T})$ can be encoded in a binary string of length $n \cdot T$. Shannon proved that lossless communication is possible as $T\to \infty$ while transmitting at any bit-rate exceeding the entropy rate\footnote{We use the natural logarithm when defining entropy in this paper. The equivalence between minimal the bit-rate of a binary encoder and the entropy rate requires that log base 2 is used instead.}  of $\theta$. If one is willing to tolerate imperfect reconstruction of the latent states, then the bit-rate can be further reduced. Rate distortion theory quantifies the necessary bit-rate to achieve a given recovery loss, e.g. $\E\left[ \frac{1}{T}\sum_{t=1}^{T} \| \theta_t - \hat{\theta}_t \|^2 \right]\leq D$ in \cref{subfig:lossy_state_estimation}.

 \cref{subfig:lossy_decision-making} applies similar reasoning to a decision-making problem.  Having observed $\theta_{1:T} = (\theta_1, \theta_2, \ldots, \theta_T)$, the observer transmits some signal $f_T(\theta_{1:T})$ to a decision-maker, who processes that information and implements the decision $A^\dagger_{1:T}$. It is possible for the decision-maker to losslessly recover the optimal action process if and only if the encoder communicates at bit-rate exceeding the entropy rate of the optimal action process $\bar{H}_{\infty}(A^*)$. Rate distortion theory tells us that the decision-maker can achieve regret-rate lower than some distortion level $D$ when the bit-rate exceeds $\bar{\mathcal{R}}_{\infty}(D)$ defined as
    \begin{equation}\label{eq:rate-disortion-function}
    \bar{\mathcal{R}}_{\infty}(D) := \inf_{A^{\dagger}}  \bar{I}_{\infty}(A^\dagger  ; \theta  )  \quad \text{subject to}  \quad 		\bar{\Delta}_\infty(A^\dagger;A^*) \leq D.
    \end{equation}
    The function $\bar{\mathcal{R}}_{\infty}: \mathbb{R}_+ \to \mathbb{R}_+$ is called the rate-distortion function. It minimizes the information about the latent states required 
    to implement a satisficing action sequence $A^\dagger$ (see \cref{def:satisficing_action}), over all choices with regret rate less than $D$. For any fixed horizon $T$, one can analogously define 
    \begin{equation}\label{eq:rate-disortion-function-fixed-T}
    \bar{\mathcal{R}}_{T}(D) := \inf_{A^{\dagger}_{1:T}}  \bar{I}_{T}(A^\dagger_{1:T}  ; \theta  )  \quad \text{subject to}  \quad 		\bar{\Delta}_T(A^\dagger_{1:T}; A^*_{1:T}) \leq D.
    \end{equation}

   \begin{figure}
    \begin{framed}
    \centering

        \begin{subfigure}{.99\linewidth} 
            \centering
            \begin{tikzpicture}[>=latex]
        \node [draw, rectangle, minimum width=2cm, minimum height=1cm] (encoder) at (0, 0) {Encoder};
        \node (theta) at (-2.5, 0) {$\theta_{1:T}$};
    
        \node [draw, rectangle, minimum width=2cm, minimum height=1cm] (decoder) at (5, 0) {Decoder};
    
        \node [draw, rectangle, dashed, rounded corners,  minimum width=4cm, minimum height=1cm,  text width=3cm, align=center] (loss) at (10, 0) {Loss: \, \\ $\frac{1}{T} \sum_{t=1}^{T} \|\theta_t - \hat{\theta}_t \|^2$};


        \draw[dashed, ->, bend right=20] (theta) to node[above] {} (loss);
        
        \draw[->] (encoder) -- node[above] {$f_{T}(\theta_{1:T})$} (decoder);
        \draw[->] (theta) -- (encoder);
        \draw[->] (decoder) -- node[above] {$\hat{\theta}_{1:T}$} (loss);    
         \end{tikzpicture}
        \caption{Optimal lossy compression for state estimation}
        \label{subfig:lossy_state_estimation}
    \end{subfigure}

    \vspace{7mm} 
    \begin{subfigure}{.99\linewidth} 
        \centering
        \begin{tikzpicture}[>=latex]
        \node [draw, rectangle, minimum width=2cm, minimum height=1cm] (encoder) at (0, 0) {Encoder};
        \node (theta) at (-2.5, 0) {$\theta_{1:T}$};
    
        \node [draw, rectangle, minimum width=2cm, minimum height=1cm] (decoder) at (5, 0) {Decision-rule};
    
        \node [draw, rectangle, dashed, rounded corners,  minimum width=3.5cm, minimum height=1cm,  text width=2cm, align=center] (loss) at (10, 0) {Regret: \, \\ $\bar{\Delta}_{T}(A^{\dagger}_{1:T} ; A^*_{1:T})$};

        \draw[dashed, ->, bend right=20] (theta) to node[above] {} (loss);

        \draw[->] (encoder) -- node[above] {$f_{T}(\theta_{1:T})$} (decoder);
        \draw[->] (theta) -- (encoder);
        \draw[->] (decoder) -- node[above] {$A^\dagger_{1:T}$} (loss);    
        \end{tikzpicture}
        \caption{Optimal lossy compression for decision-making}
            \label{subfig:lossy_decision-making}
    \end{subfigure}
        \vspace{7mm} 
    \caption{Lossy compression of the environment states}
\end{framed}
\end{figure}

\subsection{General regret bounds in terms of the rate distortion function}

As discussed in Remark \ref{rem:information-ratio-bounds-sts}, many of the most widely studied online learning problems have the feature that there is a uniform upper bound on the information ratio. While there is a price (in terms of squared regret suffered)
for acquiring information, the price-per-bit does not explode regardless of the information one seeks. 

For such problems, we can bound the attainable regret-rate of a DM who learns through interaction by solving the kind of lossy compression problem described above. Namely, we bound regret in terms the information ratio and the problem's rate distortion function. At this point, the proof is a trivial consequence of Theorem \ref{thm:main-result-sts}. But apriori, the possibility of such a result is quite unclear. 

\begin{theorem} \label{thm:rate-distortion}
	Suppose that $\Gamma_{U} := \sup_{A^\dagger}  \inf_{\pi} \Gamma(\pi; A^\dagger) < \infty$. Then 
	\[
	\inf_{\pi} \bar{\Delta}_\infty(\pi) \leq \inf_{D\geq 0} \left\{ D  + \sqrt{ \Gamma_{U} \cdot \bar{\mathcal{R}}_\infty(D)} \right\}.
	\]
    Similarly, for any finite $T<\infty$, 
    \[
    \inf_{\pi} \bar{\Delta}_T(\pi) \leq \inf_{D\geq 0} \left\{ D  + \sqrt{ \Gamma_{U} \cdot \bar{\mathcal{R}}_T(D)} \right\}.
    \]    
\end{theorem}
\begin{proof}
Theorem \ref{thm:main-result-sts} shows that for any satisficing action sequence and policy $\pi$, 
\[
\bar{\Delta}_{\infty}(\pi)    \leq \bar{\Delta}_{\infty}(A^{\dagger}; A^*) +  \sqrt{ \Gamma(\pi, A^\dagger) \times \bar{I}_{\infty}(A^\dagger ; \theta)}.
\]
Taking the infimum over $\pi$ yields
	\begin{equation}\label{eq:bound-for-all-benchmarks}
\inf_{\pi} \bar{\Delta}_{\infty}(\pi)    \leq \bar{\Delta}_{\infty}(A^{\dagger}; A^*) +  \sqrt{ \Gamma_U \times \bar{I}_{\infty}(A^\dagger ; \theta)}.
\end{equation}    
Taking the infimum over of the right-hand side over satisficing action sequences with $\bar{\Delta}_{\infty}(A^{\dagger}; A^*) \leq D$, and then minimizing over $D$ yields the result. 
\end{proof}

As in \cref{sec:dregret}, we provide some more easily interpreted special cases of our result. The next result is an analogue of Corollary \ref{cor:k-armed} which upper bounds the attainable regret-rate in terms of the number of actions and the rate-distortion function $\bar{\cal R}(D)$.  For brevity, we state the result only for the infinite horizon regret rate. 
\begin{corollary}\label{cor:k-armed-rate-distortion} Under any problem in the scope of our formulation,
    \[
    \inf_{\pi} \bar{\Delta}_{\infty}(\pi) \leq \inf_{D\geq 0} \left\{ D+ \sigma \sqrt{2 \cdot |\mathcal{A}| \cdot \bar{R}_{\infty}(D)} \right\}.
    \] 
\end{corollary}

A dependence on the number of actions can be avoided when it is possible to acquire information about some actions while exploring other actions. Full information problems are an extreme case where information can be acquired without any active exploration. In that case the optimal policy is the rule $\pi^{\rm Greedy}$ that chooses $A_t \in \argmax_{a \in \mathcal{A}} \mathbb{E}[R_{t,a} \mid \mathcal{F}_{t-1}]$ in each period $t$. It satisfies the following bound.

\begin{corollary} In  full information problems, where $O_{t,a}=O_{t,a'}$ for each $a,a'\in\mathcal{A}$, we have
	\[
	\bar{\Delta}_\infty(\pi^{\rm Greedy})  \leq \inf_{D\geq 0} \left\{ D+ \sigma \sqrt{2 \cdot \bar{R}_{\infty}(D)} \right\}.
	\]
\end{corollary}

\subsection{Application to environments with low total variation} \label{subsec:variation-budget}

One of the leading ways of analyzing nonstationary online optimization problems assumes little about the environment except for a bound on the total variation across time \citep{besbes14,besbes15, cheung19}. In this section, we upper bound the rate distortion function in terms of the normalized expected total variation in the suboptimality of an arm:
\begin{equation}\label{eq:total-variation}
	\bar{V}_T :=  \E\left[\frac{1}{T}\sum_{t=2}^{T}\max_{a \in \mathcal{A}} \left|\Delta_{t}(a) - \Delta_{t-1}(a)\right| \right] \quad \text{where} \quad \Delta_{t}(a) := \E[ R_{t,A_t^*} - R_{t,a}  \mid  \theta_t].
\end{equation}
It is also possible to study the total variation in mean-rewards $\E[R_{t,a} \mid \theta_t]$, as in \citet{besbes15}.  Figure \ref{fig:example-illustration} displays an environment in which total variation of the optimality gap $\Delta_{t}(a)$ is much smaller, since it ignores variation that is common across all arms.

The next proposition shows it is possible to attain a low regret rate whenever the the total variation of the environment is low. We establish this by bounding the rate distortion function, i.e., the number of bits about the latent environment states that must be communicated in order to implement a near-optimal action sequence. To ease the presentation, we use $\tilde{O}(\cdot)$ notation that hides logarithmic factors, but all constants can be found in \cref{subsec:variation-budget-proof}. 
\begin{proposition}\label{prop:variation-non-constructive}
	The rate distortion function is bounded as
	\[
	\bar{\mathcal{R}}_{T}(D) \leq \tilde{O}\left( 1+\frac{\bar{V}_{T}}{D}  \right).
	\]
	If there is a uniform bound  on the information ratio $\Gamma_U :=  \sup_{A^\dagger} \inf_{\pi} \Gamma(\pi ; A^{\dagger}) < \infty$, then
	\[
	\inf_{\pi} \bar{\Delta}_T\left(  \pi\right) \leq \tilde{O}\left( \left( \Gamma_U \bar{V}_T\right)^{1/3} + \sqrt{ \frac{ \Gamma_U}{T} }\right).
	\]
\end{proposition}
From \cref{cor:k-armed-rate-distortion}, one can derived bounds for $k$-armed bandits similar to those in \cite{besbes15}. One strength of the general result is it can be specialized to a much broader class of important online decision-making problems with bounded information ratio, ranging from combinatorial bandits to contextual bandits. 
\begin{proof}[Proof sketch]
	Even if the total variation $\bar{V}_T$ is small, it is technically possible for the entropy rate of the optimal action process to be very large. This occurs if the optimal action switches identity frequently and unpredictably, but the gap between actions' performance vanishes. 
	
	Thankfully, the satisificing action sequence in \cref{ex:sts-ignoring-small-suboptimality} is near optimal and has low entropy. Recall the definition 
	\begin{align*}
		A^\dagger_t = \begin{cases}
			A^\dagger_{t-1}  & \text{ if }  \E[R_{t, A_{t-1}^\dagger} \mid \theta_t] \geq \E[\max_{a \in \mathcal{A}}R_{t,a} \mid \theta_t] -D\\
			\argmax_{a \in \mathcal{A}} \E[R_{t,a} \mid \theta_t] & \text{otherwise}.
		\end{cases}
	\end{align*}
	This sequence of arms switches only when an arm's suboptimality exceeds a distortion level $D>0$. It is immediate that  $\bar{\Delta}_{T}(A^*; A^\dagger) \leq D$ and our proof reveals
	\[
	\bar{H}_{T}(A^\dagger_{1:T}) \leq \frac{2 \bar{V}_T }{D} \cdot \left( 1 + \log \left( 1 + T \wedge \left(1 \vee \frac{D}{2 \bar{V}_T} \right) \right) + \log( |\mathcal{A}| ) \right) + \frac{3 \log(  |\mathcal{A}| T )}{T}.
	\]
	Together, these give a constructive bound on the rate distortion function. The remaining claim is Theorem \ref{thm:rate-distortion}. 
\end{proof}

While we state this result in terms of fundamental limits of performance, it is worth noting that the proof itself implies Algorithm \ref{alg:sts-ignoring-small-suboptimality} attains regret bounds in terms of the variation budget in the problem types mentioned in Remark \ref{rem:information-ratio-bounds-sts}. 

It is also worth mentioning that the entropy rate of the optimal action process could, in general, be large even in problems with low total-variation. This occurs when the optimal action keeps switching identity, but the arms it switches among tend to still be very nearly optimal. The generalization of  \cref{thm:main-result} to \cref{thm:rate-distortion} is essential to this result.

\section{Extension 3: Bounds Under Adversarial Nonstationarity}\label{sec:adversarial}
The theoretical literature on nonstationary bandit learning  mostly focuses on the limits of attainable performance when faced the true environment, or latent states $\theta$, can be chosen adversarially. At first glance, our results seem to be very different. We use the language of stochastic processes, rather than game theory, to model the realization of the latent states. Thankfully, these two approaches are, in a sense, dual to one another. 

In statistical decision theory, is common to minimax risk by studying Bayesian risk under a least favorable prior distribution. Using this insight, it is possible to deduce bounds on attainable regret in \emph{adversarial} environments from our bounds on \emph{stochastic} environments. We illustrate this first when the adversary is constrained to select latent states  under which the optimal action switches infrequently \citep[see e.g.,][]{suk22}. Then, we study when the adversary must pick a sequence under which arm-means have low total-variation, as in \citet{besbes15}. We bound the rate-distortion function in terms of the total-variation, and recover adversarial regret bounds as a result. 

Our approach builds on \cite{bubeck2016multi}, who leveraged `Bayesian' information ratio analysis to study the fundamental limits of attainable regret of adversarial convex bandits;  see also \cite{lattimore2019information, lattimore2020improved, lattimore21,lattimore2022minimax}. These works study regret relative to the best fixed action in hindsight, wheres we study regret relative to the best changing action sequence --- often called ``dynamic regret.'' In either case, a weakness of this analysis approach is that it is non-constructive. It reveals that certain levels of performance are possible even in adversarial environments, but does not yield concrete algorithms that attain this performance. It would be exciting to synthesize our analysis with recent breakthroughs of  \cite{xu2023bayesian}, who showed how to derive algorithms for adversarial environments out of information-ratio style analysis.  

\subsection{A minimax theorem}
Define the expected regret rate incurred by a policy $\pi$ under parameter realization $\theta'$ by
\[
\bar{\Delta}_{T}(\pi ; \theta') := \E\left[ \frac{1}{T} \sum_{t=1}^{T}  (R_{t,A^*_t} -   R_{t,A_t} \mid \theta_t = \theta_t')\right].
\]
The Bayesian, or average, regret 
\[
\bar{\Delta}_{T}(\pi ; q) := \E_{\theta \sim q}\left[  \bar{\Delta}_{T}(\pi ; \theta)\right],
\] 
produces a scalar performance measure by averaging. Many papers instead study the worst-case regret $\sup_{\theta \in \Xi}  \bar{\Delta}_{T}(\pi; \theta)$ across a constrained family $\Xi \subseteq \Theta^T$ of possible parameter realizations.
The next proposition states that the optimal worst-case regret is the same as the optimal Bayesian regret under a worst-case prior. We state stringent regularity conditions to apply the most basic version of the minimax theorem. Generalizations are possible, but require greater mathematical sophistication. 

\begin{proposition}\label{prop:minimax}
	Let $\Xi$ be a finite set and take $\mathcal{D}(\Xi)$ to be the set of probability distributions over $\Xi$. Suppose the range of the outcome function $g(\cdot)$ in \eqref{eq:outcome-generation} is a finite set. Then,  
	\[ 
	\underbrace{\min_{\pi \in \Pi} \max_{q \in \mathcal{D}(\Xi) } \bar{\Delta}_{T}(\pi ; q)}_{\text{minimax regret}} = \underbrace{\max_{ q \in \mathcal{D}(\Xi)} \min_{\pi \in  \Pi} \bar{\Delta}_{T}(\pi ; q).}_{\text{Bayes regret under least favorable prior}}
	\]
\end{proposition}
\begin{proof}[proof sketch]
	We apply Von-Neumann's minimax theorem to study a two player game between an experimenter, who chooses $\pi$ and nature, who chooses $\theta$. A deterministic strategy for nature is a choice of $\theta \in \Xi$. Take $n=|\Xi|$ and label the possible choice of nature by $\theta^{(1)}, \ldots, \theta^{(n)}$. A randomized strategy for nature is a probability mass function $q=(q_1, \ldots, q_n) \in [0,1]^n$. 
	
	For the experimenter, we need to make a notational distinction between a possibly randomized strategy $\pi$ and a deterministic one, which we denote by $\psi$. 
	A deterministic strategy $\psi$ for the experimenter is a mapping from a history of past actions and observations to an action. Let $\mathcal{O}$ denote the range of $g$, i.e., the set of possible observations. There are at most $m = |\mathcal{A}|^{|\mathcal{O}|^T}$ deterministic policies. Label these $\psi^{(1)}, \ldots, \psi^{(m)}$. 
	A randomized strategy for the experimenter is a probability mass function $\pi = (\pi_1, \ldots, \pi_m) \in [0,1]^m$. 

	Define $\Delta \in \mathbb{R}^{m\times n}$ by $\Delta(i,j) = \bar{\Delta}_{T}(\psi^{(i)} ; \theta^{(j)})$. By Von-Neuman's minimax theorem 
	\[
	\min_{\pi} \max_{q} \, \pi^\top \Delta q = \max_{q} \min_{\pi} \, \pi^\top \Delta q.
	\]	
\end{proof} 

\subsection{Deducing bounds on regret in adversarial environments: environments with few switches}
As a corollary of this result, we bound minimax regret when the adversary is constrained to choose a sequence of latent states under which the optimal action changes infrequently. For concreteness, we state the result for linear bandit problems, though similar statements hold for any of the problems with bounded information ratio discusses in Section \ref{sec:information-ratio-bounds}.

\begin{corollary}[Corollary of \cref{thm:main-result} and \cref{prop:combinatorial-bound}]
        Suppose $\mathcal{A}, \Theta  \subset \mathbb{R}^d $,  rewards follow the linear model 
            $ \E\left[R_{t,a} \mid \theta_t\right] = \theta_t^\top a$, and the range of the reward function $R(\cdot)$ is contained in $[-1,1]$. Under the conditions in \cref{prop:minimax}, for any $T \in \mathbb{N}$, 
	\[
	\inf_{\pi \in \Pi} \max_{ \theta :  \bar{\mathcal{S}}_{T}(A^*;\theta) \leq \bar{S} }  \bar{\Delta}_{T}(\pi ; \theta) \leq  \sqrt{\frac{d}{2} \cdot \bar{S} \cdot \left( 1 + \log \left( 1 + 1/\bar{S} \right) + \log (|\mathcal{A}|) \right)}.
	\]
  where $\bar{\mathcal{S}}_{T}(A^*;\theta) := \frac{1}{T}(1+\sum_{t=2}^{T} \mathbb{I}\{A^*_t \neq A^*_{t-1}\})$ is the switching rate defined in \eqref{eq:switching-rate}. 
\end{corollary}
\begin{proof}
     Fix a prior distribution $q$ supported on $\Xi := \{\theta \in \Theta^T:  \bar{\mathcal{S}}_{T}(\theta) \leq \bar{S}\}$. \cref{thm:main-result} yields the bound 
     $\bar{\Delta}_{T}(\pi) \leq \sqrt{\Gamma(\pi) \bar{H}_{T}(A^*)}$, where implicitly the information ratio $\Gamma(\pi)$ and the entropy rate $\bar{H}_{T}(A^*)$ depend on $q$. We give uniform upper bounds on these quantities. 

     First, since $\bar{\mathcal{S}}_{T}(A^*; \theta) \leq \bar{S}$ for any $\theta \in \Xi$, \cref{lem:combinatorial-bound} implies  
     \[
     \bar{H}_{T}(A^*) \leq  \bar{S} \cdot \left( 1 + \log \left( 1 + 1/\bar{S} \right) + \log |\mathcal{A}|\right).
     \]
    Bounded random variables are sub-Gaussian. In particular, since the reward function $R(\cdot)$ is contained in $[-1,1]$, we know the variance proxy is bounded as $\sigma^2 \leq 1$. From  \cref{sec:information-ratio-bounds}, we know, $\Gamma(\pi^{\rm TS}) \leq 2 \cdot \sigma^2 \cdot d \leq  2 \cdot d$. Combining these results gives 
    \[ 
    \inf_{\pi} \bar{\Delta}_{T}(\pi; q) \leq   \sqrt{\inf_{\pi} \Gamma(\pi) \bar{H}_{T}(A^*)} \leq \sqrt{ 2 \cdot d \cdot \bar{S} \cdot \left( 1 + \log \left( 1 + 1/\bar{S} \right) + \log |\mathcal{A}|\right)}
    \]
    The claim then follows from \cref{prop:minimax}.
\end{proof}

\subsection{Deducing bounds on regret in adversarial environments: environments with low total variation}

From the regret bound for stochastic environments in \cref{prop:variation-non-constructive}, we can deduce a bound on minimax regret in linear bandits when the adversary is constrained in the rate of total variation in mean rewards \citep{besbes15}. The proof is omitted for brevity.  
\begin{corollary}[Corollary of  \cref{prop:variation-non-constructive}]
	        Suppose $\mathcal{A}, \Theta \subset \mathbb{R}^d $,  rewards follow the linear model 
        $ \E\left[R_{t,a} \mid \theta_t\right] = \theta_t^\top a$, and rewards are bounded as $|R_{t,a}|\leq 1$. Under the conditions in \cref{prop:minimax}, for any $T \in \mathbb{N}$, 
	\[
	\inf_{\pi \in \Pi} \max_{ \theta :  \bar{\mathcal{V}}_{T}(\theta) \leq \bar{V} }  \bar{\Delta}_{T}(\pi ; \theta) \leq \tilde{O}\left( \left( d \bar{V} \right)^{1/3} + \sqrt{ \frac{ d }{T} }\right).
	\]
  where $\bar{\mathcal{V}}_{T}(\theta) := \frac{1}{T}\sum_{t=2}^{T}\max_{a \in \mathcal{A}} \left|\Delta_{t}(a) - \Delta_{t-1}(a)\right|$ is the total variation measure used in  \eqref{eq:total-variation}
\end{corollary}

\section{Conclusion and Open Questions}

We have provided a unifying framework to analyze interactive learning in changing environments.
The results offer an intriguing measure of the difficulty of learning: the entropy rate of the optimal action process. 
A strength of the approach is that it applies to nonstationary variants of many of the most important learning problems. 
Instead of designing algorithms to make the proofs work, most results apply to Thompson sampling (TS), one of the most widely used bandit algorithms, and successfully recover the existing results that are proven individually for each setting with different proof techniques and algorithms.

While our analyses offer new theoretical insights, practical implementation of algorithms in real-world problem involves numerous considerations and challenges that we do not address. \cref{sec:TS}
described a simple setting in which Thompson sampling with proper posterior updating resembles 
simple exponential moving average estimators. The A/B testing scenario of \cref{subsec:example} and the news recommendation problem in \cref{ex:news-rec} reflect there is often natural problem structure that is different from what agnostic exponential moving average estimators capture.  To leverage this, it is crucial to properly model the dynamics of the environment and leverage auxiliary data to construct an appropriate prior. Implementing the posterior sampling procedure adds another layer of complexity, given that the exact posterior distribution is often not available in a closed form in for models with complicated dynamics. Modern generative AI  techniques  (e.g. diffusion models) provide a promising path to enhance both model flexibility and sampling efficiency.

Lastly, we call for a deeper exploration of the connection between learning in dynamic environments and the theory of optimal compression.   \cref{sec:rate-distortion} provides intriguing connections to rate-distortion theory, but many questions remain open. One open direction is around information-theoretic lower bounds. For that, we conjecture one needs to construct problem classes in which the uniform \emph{upper bound} $\Gamma_U$ on the information ratio is close to a lower bound. Another open direction is to try to characterize or bound the rate-distortion function in other scenarios.   
In the information-theory literature, numerous studies have provided theoretical characterizations of rate distortion functions \citep{cover2006elements, gray71, blahut97, derpich12, stavrou18, stavrou20}.
It is worth investigating whether a synthesis of these existing rate distortion results with our framework can produce meaningful regret bounds, particularly for the nonstationary bandit environments driven by Markov processes such as Brownian motion \citep{slivkins08} or autoregressive processes \citep{chen23}.
Furthermore, computational methods for obtaining the rate distortion function and the optimal lossy compression scheme \citep{blahut97, jalali08, theis17} can be implemented to construct the rate-optimal satisficing action sequence.

\vfill

{\small
\bibliography{nonstat-ts}}

\newpage

\appendix

\newpage

\section{Proofs} \label{app:proof}

\subsection{Proof of \cref{lem:combinatorial-bound}}
    We first bound the number of possible realizations that involve at most $S_T$ switches:\footnote{One can imagine a two-dimensional grid of size $T \times S_T$, represented with coordinates $\left( (t,s) : t \in [T], s \in [S_T] \right)$.
	A feasible switching time configuration corresponds to a path from the lower left corner $(1,1)$ to the upper right corner $(T,S_T)$ that consists of $T-1$ rightward moves and $S_T-1$ upward moves ; whenever a path makes an upward move, from $(t,s)$ to $(t,s+1)$, we can mark that a switch occurs at time $t$ (if the path makes two or more upward moves in a row, the actual number of switches can be smaller than $S_T$). The number of such paths is given by $\binom{(T-1)+(S_T-1)}{S_T-1}$.}
    \[
        \left| \left\{ (x_1,\ldots,x_T) \in \mathcal{X}^T \left| 1+\sum_{t=2}^T \mathbb{I}\{x_t \ne x_{t-1} \} \leq S_T \right. \right\} \right| 
        \leq \binom{(T-1)+(S_T-1)}{S_T-1} \times |\mathcal{X}|^{S_T}
        \leq \binom{T+S_T-1}{S_T} \times |\mathcal{X}|^{S_T}.
    \]
	Note that for any $k \leq n \in \mathbb{N}$,
	\[
		\binom{n}{k} = \frac{n \times (n-1) \times \ldots \times (n-k+1) }{k!} \leq \frac{n^k}{k!} \leq \frac{n^k}{\sqrt{2\pi k}(k/e)^k} \leq \frac{n^k}{(k/e)^k} = \left( \frac{e n}{k} \right)^k,
	\]
	where the second inequality uses Stirling.
	Therefore,
    \[
        \log\left( |\mathcal{X}|^{S_T} \times \binom{T+S_T-1}{S_T} \right) 
        \leq \log\left( |\mathcal{X}|^{S_T} \times \left( \frac{e (T+S_T-1)}{S_T} \right)^{S_T} \right)
        = S_T \times \left( \log(|\mathcal{X}|) + 1 + \log \left( 1 + \frac{T-1}{S_T} \right)\right).
    \]
    By observing $ \log \left( 1 + \frac{T-1}{S_T} \right) \leq \log \left( 1 + \frac{T}{S_T} \right)$, we obtain the desired result.

\subsection{Proof of \cref{prop:effective-horizon-bound}}
	Let $Z_t := \mathbb{I}\{ A_t^* \ne A_{t-1}^* \}$, an indicator of a ``switch''.
	Then, $\tau_\textup{eff}^{-1} = \mathbb{P}(Z_t = 1)$ and for any $t \geq 2$,
	\begin{align*}
		H(A_t^* | A_{1:t-1}^*) 
			&= H(A_t^* | A_{1:t-1}^* ) + \underbrace{ H( Z_t | A_{1:t-1}^*, A_t^* ) }_{=0}
			\\&= H( (Z_t, A_t^*) | A_{1:t-1}^* ) 
			\\&= H(Z_t | A_{1:t-1}^* ) + H(A_t^* | Z_t, A_{1:t-1}^* )
			\\&\leq H(Z_t) + H(A_t^* | Z_t, A_{1:t-1}^* )
			\\&= H(Z_t) + \mathbb{P}( Z_t = 1 ) H(A_t^* | Z_t=1, A_{1:t-1}^* )
			\\& \qquad + \underbrace{ \mathbb{P}( Z_t = 0 ) H(A_t^* | Z_t=0, A_{1:t-1}^* ) }_{= 0}
			\\&\leq H(Z_t) + \mathbb{P}(Z_t = 1) H(A_t^* | Z_t=1 ).
	\end{align*}
	With $\delta := \tau_\textup{eff}^{-1}$,
	\begin{align*}
		& H(Z_t) + \mathbb{P}( Z_t = 1 ) H(A_t^* | Z_t=1 )
			\\&= \delta \log(1/\delta) + (1-\delta) \log(1/(1-\delta)) + \delta H(A_t^* | Z_t=1 )
			\\&= \delta \log(1/\delta) + (1-\delta) \log(1+\delta/(1-\delta)) + \delta H(A_t^* | Z_t=1 )
			\\&\leq \delta \log(1/\delta) + \delta + \delta H(A_t^* | Z_t=1 )
			\\&= \frac{1}{\tau_\textup{eff}} \left[ \log(\tau_\textup{eff}) + 1 + H(A_t^* | Z_t=1) \right].
	\end{align*}
 Therefore, we deduce that
         \[
            \bar{H}_T(A^*) = \frac{1}{T} \sum_{t=1}^T H(A_t^* | A_{1:t-1}^* )
                \leq \frac{1 + \log(\tau_\textup{eff})+ H(A_t^* | Z_t=1) }{\tau_\textup{eff}}  + \frac{H(A_1^*)}{T}.
        \]

\subsection{Proof of \cref{prop:combinatorial-bound}}
    We use $A_{1:T}^*$ to denote $(A_1^*, \ldots, A_T^*)$ for shorthand.
    Let $\mathcal{S}_T := 1 + \sum_{t=2}^T \mathbb{I}\{A_t^* \ne A_{t-1}^*\}$.
    By \cref{lem:combinatorial-bound},
    \[
        H( A_{1:T}^* | \mathcal{S}_T = n ) \leq n \times \left( 1 + \log \left( 1 + \frac{T}{n} \right) + \log |\mathcal{A}| \right),
    \]
    for any $n \in \{1,\ldots,T\}$.
    Since the right hand side is a concave function of $n$, by Jensen's inequality
    \begin{align*}
        \frac{1}{T} H( A_{1:T}^* | \mathcal{S}_T ) 
            &\leq \mathbb{E}\left[ \frac{\mathcal{S}_T}{T} \times \left( 1 + \log \left( 1 + \frac{T}{\mathcal{S}_T} \right) + \log |\mathcal{A}| \right) \right],
            \\&\leq \frac{\mathbb{E}[ \mathcal{S}_T ]}{T} \times \left( 1 + \log \left( 1 + \frac{T}{ \mathbb{E}[ \mathcal{S}_T ] }  \right) + \log |\mathcal{A}| \right)
            \\&= \bar{S}_T \times \left( 1 + \log \left( 1 + 1/\bar{S}_T  \right) + \log |\mathcal{A}| \right).
    \end{align*}
    On the other hand, we have $H(\mathcal{S}_T) \leq \log T$ since $\mathcal{S}_T \in \{1, 2,\ldots,T\}$.
    Therefore,
    \[
        \bar{H}_T(A^*) = \frac{H( A_{1:T}^* | \mathcal{S}_T )}{T} + \frac{H( \mathcal{S}_T )}{T}
            \leq \bar{S}_T \times \left( 1 + \log \left( 1 + 1/\bar{S}_T  \right) + \log |\mathcal{A}| \right) + \frac{\log T}{T}.
    \]

\subsection{Proof of \cref{thm:lower-bound}}
We start with a proof sketch. 
Our proof is built upon a well-known result established for stationary bandits: there exists a stationary (Bayesian) bandit instance such that any algorithm's (Bayesian) cumulative regret is lower bounded by $\Omega( \sqrt{nk} )$ where $n$ is the length of time horizon.

More specifically, we set $n = \Theta( \tau_{\rm eff} ) \in \mathbb{N}$ and construct a nonstationary environment by concatenating independent $n$-period stationary Gaussian bandit instances, i.e., the mean rewards changes periodically every $n$ time steps.
In each block (of length $n$), the best arm has mean reward $\epsilon > 0$ and the other arms has zero mean reward, where the best arm is drawn from $k$ arms uniformly and independently per block.
When $\epsilon = \Theta( \sqrt{k/n} )$, no algorithm can identify this best arm within $n$ samples, and hence the cumulative regret should increase by $\Omega( n \epsilon )$ per block. 
Consequently, the per-period regret $\bar{\Delta}_\infty(\pi)$ should be $\Omega(\epsilon) = \Omega( \sqrt{k/n} ) = \Omega( \sqrt{k/\tau_{\rm eff}} )$.
In our detailed proof, we additionally employ some randomization trick in determination of changepoints in order to ensure that the optimal action sequence $(A_t^*)_{t \in \mathbb{N}}$ is stationary and $\mathbb{P}( A_t^* \ne A_{t-1}^* ) = \tau_{\rm eff}^{-1}$ exactly.
Now, we give the formal proof. 

\begin{proof} 
We will consider Gaussian bandit instances throughout the proof.
Without loss of generality, we assume $\sigma = 1$ and the noise variances are always one.

We begin by stating a well-known result for the stationary bandits, adopted from \citet[Exercise 15.2]{lattimore20}:
With $\epsilon = (1-1/k) \sqrt{k/n}$, for each $i \in \{1, \ldots, k\}$, let mean reward vector $\mu^{(i)} \in \mathbb{R}^k$ satisfy $\mu_a^{(i)} = \epsilon \mathbb{I}\{ i = a \}$.
It is shown that, when $k > 1$ and $n \geq k$, under any algorithm $\pi$ 
\begin{equation} \label{eq:lower-bound-known-result}
	\frac{1}{k} \sum_{i=1}^k \mathbb{E}_{\mu^{(i)}}^\pi\left[ \sum_{t=1}^n (R_{t,i} - R_{t,A_t}) \right] \geq \frac{1}{8} \sqrt{nk},
\end{equation}
where $\mathbb{E}_{\mu^{(i)}}^\pi\left[ \sum_{t=1}^n (R_{t,i} - R_{t,A_t}) \right]$ is the (frequentist's) cumulative regret of $\pi$ in a $k$-armed Gaussian bandit instance specified by the time horizon length $n$ and mean reward vector $\mu^{(i)}$ (i.e., the reward distribution of arm $a$ is $\mathcal{N}(\mu_a^{(i)}, 1^2)$).
Considering a uniform distribution over $\{ \mu^{(1)}, \cdots, \mu^{(k)} \}$ as a prior, we can construct a Bayesian $K$-armed bandit instance of length $n$ such that $\mathbb{E}[ \sum_{t=1}^n (R_{t,A^*} - R_{t,A_t}) ] \geq \sqrt{nk}/8$ under any algorithm.

Given $\tau_{\rm eff} \geq 2$, set $\tilde{\tau} = \frac{k-1}{k} \tau_{\rm eff}$, $n = \lfloor \tilde{\tau} \rfloor$, and $p = \tilde{\tau} - \lfloor \tilde{\tau} \rfloor$.
Let $N$ be the random variable such that equals $n$ with probability $p$ and equals $n+1$ with probability $1-p$, so that $\mathbb{E}[ N ] = \tilde{\tau}$.
We construct a \emph{stationary} renewal process $(T_1, T_2, \ldots)$ whose inter-renewal time distribution is given by the distribution of $N$.
That is, $T_{j+1} - T_j \stackrel{\text{d}}{=} N$ for all $j \in \mathbb{N}$, and $T_1$ is drawn from the equilibrium distribution of its excess life time, i.e.,
$$ \mathbb{P}( T_1 = x ) = \left\{ \begin{array}{ll} 1/\tilde{\tau} & \text{if } x \leq n, \\ p/\tilde{\tau} & \text{if } x = n+1, \\ 0 & \text{if } x > n+1, \end{array} \right. \quad \forall x \in \mathbb{N}. $$
Since the process $(T_1, T_2, \ldots)$ is a stationary renewal process,
$$ \mathbb{P}\left( \text{renewal occurs at $t$} \right) = \mathbb{P}\left( \exists j, T_j = t \right) = \frac{1}{ \mathbb{E}[N] } = \frac{1}{\tilde{\tau}}, \quad \forall t \in \mathbb{N}. $$

We now consider a nonstationary Gaussian bandit instance where the mean reward vector is (re-)drawn from $\{ \mu^{(1)}, \cdots, \mu^{(k)} \}$ uniformly and independently at times $T_1, T_2, \ldots$.
As desired, the effective horizon of this bandit instance matches the target $\tau_{\rm eff}$:
$$ \mathbb{P}\left( A_t^* \ne A_{t-1}^* \right) = \mathbb{P}\left( \exists j, T_j = t \right) \times \mathbb{P}\left( A_t^* \ne A_{t-1}^* | \exists j, T_j = t  \right) = \frac{1}{\tilde{\tau}} \times \left( 1 - \frac{1}{k} \right) = \frac{1}{\tau_{\rm eff}}. $$
Since $T_{j+1} - T_{j} \geq n$,
$$ \mathbb{E}\left[ \sum_{t=T_j}^{T_{j+1}-1} ( R_{t,A_t^*} - R_{t,A_t} ) \right] \geq \mathbb{E}\left[ \sum_{t=T_j}^{T_j+n-1} ( R_{t,A_t^*} - R_{t,A_t} ) \right] = \frac{1}{k} \sum_{i=1}^k \mathbb{E}\left[ \left. \sum_{t=T_j}^{T_j+n-1} ( R_{t,A_t^*} - R_{t,A_t} ) \right| A_{T_j}^* = i \right] \geq \frac{1}{8} \sqrt{nk}, $$
where the last inequality follows from \cref{eq:lower-bound-known-result}.
Since there are at least $\lfloor T/(n+1) \rfloor$ renewals until time $T$, $ \mathbb{E}\left[ \sum_{t=1}^{T} ( R_{t,A_t^*} - R_{t,A_t} ) \right] \geq \lfloor T/(n+1) \rfloor \sqrt{nk} / 8 $, and therefore,
$$ \bar{\Delta}_\infty(\pi) = \limsup_{T \rightarrow \infty} \mathbb{E}\left[ \frac{1}{T} \sum_{t=1}^{T} ( R_{t,A_t^*} - R_{t,A_t} ) \right] \geq \frac{  \sqrt{nk} }{8(n+1)}. $$
Since $n+1 \leq 2n$ and $n =  \lfloor \frac{k-1}{k} \tau_{\rm eff} \rfloor \leq \tau_{\rm eff}$, we have $\bar{\Delta}_\infty(\pi) \geq \frac{1}{16} \sqrt{ \frac{k}{\tau_{\rm eff}} }$.
\end{proof}

\subsection{Proof of \cref{rem:lambda-info-ratio}}
	Let $\Delta_t := \mathbb{E}[ R_{t,A_t^*} - R_{t,A_t} ]$, $G_t := I( A_t^*; (A_t, O_{t,A_t}) | \mathcal{F}_{t-1} )$, and $\Gamma_{\lambda,t} := \Delta_t^\lambda/G_t$.
	Then, $ \Gamma_\lambda(\pi) = \sup_{t \in \mathbb{N}} \Gamma_{\lambda,t}$, and we have
	\begin{align*}
		\bar{\Delta}_T(\pi) &= T^{-1}\sum_{t=1}^T \Delta_t \\
		&=  T^{-1}\sum_{t=1}^T \Gamma_{\lambda,t}^{1/\lambda} G_t^{1/\lambda}
		\\& \, \leq \Gamma_\lambda(\pi)^{1/\lambda} \cdot \left(  T^{-1} \sum_{t=1}^T G_t^{1/\lambda} \right)
		\\&\stackrel{(a)}{\leq} \Gamma_\lambda(\pi)^{1/\lambda} \cdot \left[  T^{-1} \left( \sum_{t=1}^T G_t \right)^{1/\lambda} \cdot \left( \sum_{t=1}^T 1 \right)^{1-1/\lambda} \right] \\
		\\& = \Gamma_\lambda(\pi)^{1/\lambda} \cdot \left( T^{-1} \sum_{t=1}^T G_t \right)^{1/\lambda}
		\\& \stackrel{(b)}{\leq} \Gamma_\lambda(\pi)^{1/\lambda} \cdot \bar{H}_T(A^*)^{1/\lambda},
	\end{align*}
	where step (a) uses H\"{o}lder's inequality, and step (b) uses $T^{-1}\sum_{t=1}^{T} G_t \leq \bar{H}_T( A^* )$.

\subsection{Proof of \cref{lem:contextual}}
	Recall the definition, $\Gamma(\pi) := \sup_{t \in \mathbb{N}} \Gamma_{t}(\pi)$ where 
	\[
	\Gamma_{t}(\pi)=\frac{ \left( \mathbb{E}\left[ R_{t,A^*_t} - R_{t,A_t} \right] \right)^2 }{ I\left( A_t^*;  (A_t, O_{t,A_t})  \mid \mathcal{F}_{t-1} \right) }.
	\]
	Our goal is to bound the numerator of $\Gamma_{t}(\pi^{\rm TS})$ in terms of the denominator.
	
	Let $\mathbb{E}_{t}\left[ \cdot \right] := \mathbb{E}\left[ ~ \cdot \mid X_t, \mathcal{F}_{t-1} \right]$ denote the conditional expectation operator which conditions on observations prior to time $t$ AND the context at time $t$.
	Similarly, define the probability operation $\mathbb{P}_{t}\left( \cdot \right) := \mathbb{P}\left( \cdot \mid X_t, \mathcal{F}_{t-1} \right)$ accordingly. 
	Define $I_{t}(\cdot; \cdot)$ to be the function that evaluates mutual information when the base measure is $\mathbb{P}_t$.
	
	The law of iterated expectations states that for any real valued random variable $Z$, $\mathbb{E}[\mathbb{E}_t[Z]]=\mathbb{E}[Z]$.
	The definition of conditional mutual information states that for any random variables $Z_1$, $Z_2$,
	\begin{equation}\label{eq:contextual-tower-poperty-of-MI}
		\mathbb{E}\left[I_{t}( Z_1 ; Z_2)\right] = I(Z_1 ; Z_2  \mid X_t, \mathcal{F}_{t-1} ).
	\end{equation}
	
	Under \cref{assumption:contextual}, there exists a function $\mu: \Theta \times \mathcal{X}  \times [k] \rightarrow \mathbb{R}$ such that
	\begin{equation}\label{eq:conetxtual-mu}
		\mu(\theta', x, i) = \mathbb{E}\left[ R_{t,A_t} \mid \theta_t = \theta', i_t=i, A_t\right].  
	\end{equation}
	This specifies expected rewards as a function of the parameter and chosen arm, regardless of the specific decision-rule used. 
	
	The definition of Thompson sampling is the probability matching property on decision-rules, $\mathbb{P}\left( A^*_t = a \mid \mathcal{F}_{t-1}\right)= \mathbb{P}\left( A^*_t = a \mid \mathcal{F}_{t-1}\right)$, for each $a \in \mathcal{A}$. It implies the following probability matching property on arms: with $i^*_t := A_t^*(X_t)\in [k]$. 
	\[ 
	\mathbb{P}_t\left( i^*_t = i \right)= \mathbb{P}\left( A^*_t(X_t) = i  \mid \mathcal{F}_{t-1}, X_t\right)=  \mathbb{P}\left( A_t(X_t) = i \mid \mathcal{F}_{t-1}, X_t\right) = 	\mathbb{P}_t\left( i_t = i \right),
	\]
	which holds for each $i\in [k]$. 
	
	With this setup, repeating the analysis in Proposition 3, or Corollary 1, of \citet{russo16} implies, immediately, that 
	\begin{equation}\label{eq:contextual-old-info-ratio-bound}
		\left(\mathbb{E}_{t}\left[  \mu(\theta_t, X_t, i_t^*) -  \mu(\theta_t, X_t, i_t)  \right]\right)^2  \leq 2 \cdot  \sigma^2 \cdot k \cdot I_{t}\left( i^*_t ; (i_t, R_t) \right). 
	\end{equation}
	(Conditioned on context, one can repeat the same proof to relate regret to information gain about the optimal arm.) Now, we complete the proof:
	\begin{align*}
		\left(\mathbb{E}\left[R_{t,A*_t} - R_{t,A_t} \right]\right)^2 & 
		\overset{(a)}{=} \left( \mathbb{E}\left[ \mu(\theta_t, X_t, i_t^*) -  \mu(\theta_t, X_t, i_t)  \right] \right)^2 \\
		&\overset{(b)}{\leq} \mathbb{E}\left[ \left(\mathbb{E}_{t}\left[  \mu(\theta_t, X_t, i_t^*) -  \mu(\theta_t, X_t, i_t)  \right]\right)^2 \right] \\
		&\overset{(c)}{\leq} 2 \cdot  \sigma^2 \cdot k \cdot \mathbb{E}\left[ I_{t}\left( i^*_t ; (i_t, R_t) \right) \right] \\
		& \overset{(d)}{=}  2 \cdot  \sigma^2 \cdot k \cdot  I\left( i^*_t ; (i_t, R_t) \mid X_t, \mathcal{F}_{t-1} \right) \\
		&\overset{(e)}{\leq}  2 \cdot  \sigma^2 \cdot k \cdot  I\left( A^*_t ; (i_t, R_t) \mid X_t, \mathcal{F}_{t-1} \right)\\
		&\overset{(f)}{\leq}  2 \cdot  \sigma^2 \cdot k \cdot  I\left( A^*_t ; (A_t, R_t) \mid X_t, \mathcal{F}_{t-1} \right) \\
		&\overset{(g)}{=}  2 \cdot  \sigma^2 \cdot k \cdot \left[  I\left( A^*_t ; (A_t, X_t, R_t) \mid \mathcal{F}_{t-1} \right)  - I\left( A^*_t ; X_t \mid \mathcal{F}_{t-1} \right)  \right]  \\
		&\overset{(h)}{\leq} 2 \cdot  \sigma^2 \cdot k \cdot   I\left( A^*_t ; (A_t, X_t, R_t) \mid \mathcal{F}_{t-1} \right) \\
		&\overset{(i)}{=}  2 \cdot  \sigma^2 \cdot k \cdot   I\left( A^*_t ; (A_t, O_t) \mid \mathcal{F}_{t-1} \right),
	\end{align*}
	where step (a)  uses \eqref{eq:conetxtual-mu}, step (b) is Jensen's inequality, step (c) applies \eqref{eq:contextual-old-info-ratio-bound}, step (d) is \eqref{eq:contextual-tower-poperty-of-MI},
	steps (e) and (f) apply the data processing inequality, step (g) uses the chain-rule of mutual information, step (h) uses that mutual information is non-negative, and step (i) simply recalls that $O_t=(X_t, R_t)$. 

\subsection{Proof of \cref{cor:contextual}}
If $\theta_t \in \{ -1, -1+\epsilon, \ldots, 1-\epsilon, 1 \}^{p}$ is a discretized $p$ dimensional vector, and the optimal policy process $(A^*_t)_{t\in \mathbb{N}}$ is stationary, then, by \cref{prop:effective-horizon-bound},  
\begin{align*}
	\bar{H}(A^*) &\leq \frac{1 + \log(\tau_\textup{eff}) + H(A_t^* | A_t^* \ne A_{t-1}^* ) }{\tau_\textup{eff}} \\
	&\leq \frac{1 + \log(\tau_\textup{eff}) + H(\theta_t ) }{\tau_\textup{eff}}\\
	&\leq \frac{1 + \log(\tau_\textup{eff}) + p \ln(2/\epsilon) }{\tau_\textup{eff}}.
\end{align*}
Combining this with \cref{thm:main-result} and the information ratio bound in \cref{lem:contextual} gives 
\[
\bar{\Delta}_\infty(\pi^{\rm TS}) \leq \sqrt{ 2 \cdot \sigma^2 \cdot k \times \frac{1 + \log(\tau_\textup{eff}) + p \ln(2/\epsilon) }{\tau_\textup{eff}  } }=  \tilde{O}\left( \sigma \sqrt{ \frac{p \cdot k}{\tau_\textup{eff}} } \right).
\]

\subsection{Proof of \cref{thm:main-result-sts}}
The proof is almost identical to that of \cref{thm:main-result}.
Let $\Delta_t^\dagger := \mathbb{E}[ R_{t,A_t^\dagger} - R_{t,A_t} ]$, $G_t^\dagger := I( A_t^\dagger; (A_t, O_{t,A_t}) | \mathcal{F}_{t-1} )$, and $\Gamma_t^\dagger = {\Delta_t^\dagger}^2/G_t^\dagger$. 
Then, 
	\[
		\sum_{t=1}^T \Delta_t^\dagger
			= \sum_{t=1}^T \sqrt{\Gamma_t^\dagger} \sqrt{G_t^\dagger}
			\leq \sqrt{ \sum_{t=1}^T \Gamma_t^\dagger } \sqrt{ \sum_{t=1}^T G_t^\dagger }
			\leq \sqrt{ \Gamma(\pi; A^\dagger) \cdot T \cdot \sum_{t=1}^T G_t^\dagger },
	\]
	and
    \begin{align*}
    \sum_{t=1}^T G_t^\dagger 
        &= \sum_{t=1}^T I( A_t^\dagger; (A_t, O_{t,A_t}) | \mathcal{F}_{t-1} )
	\\&\leq I( A_{1:T}^\dagger ; \mathcal{F}_T )
        \\&\leq I( A_{1:T}^\dagger ; (\theta_{1:T}, \mathcal{F}_T) )
        \\&= I( A_{1:T}^\dagger; \theta_{1:T} ) + I( A_{1:T}^\dagger; \mathcal{F}_t | \theta_{1:T} )
	\\&= I( A_{1:T}^\dagger; \theta_{1:T} ),
 \end{align*}
 where the last step uses the fact that $I( A_{1:T}^\dagger; \mathcal{F}_t | \theta_{1:T} ) = 0$ since $A_{1:T}^\dagger \perp \mathcal{F}_t | \theta_{1:T}$ according to \cref{def:satisficing_action}.
    Combining these results,
    \[
        \bar{\Delta}_T(\pi; A^\dagger) = \frac{\sum_{t=1}^T \Delta_t^\dagger}{T}
        \leq \sqrt{ \frac{ \Gamma(\pi; A^\dagger) \cdot  \sum_{t=1}^T G_t^\dagger }{T} }
        \leq \sqrt{ \Gamma(\pi; A^\dagger) \cdot \bar{I}_T(A_{1:T}^\dagger; \theta_{1:T} ) }.
    \]
    The bound on $\bar{\Delta}_\infty(\pi; A^\dagger)$ simply follows by taking limit on both sides.

\subsection{Proof of \cref{prop:variation-non-constructive}}
\label{subsec:variation-budget-proof}
We state and prove a concrete version of \cref{prop:variation-non-constructive}.

\begin{proposition}[Concrete version of \cref{prop:variation-non-constructive}] 
	If $|\mathcal{A}| \geq 2$ and $T \geq 2$, then
	\begin{equation*} \label{eq:drifting-bandit-rate-distortion}
		\bar{\mathcal{R}}_T(D) \leq \frac{2\bar{V}_T}{D} \cdot \left( 1 + \log \left( 1 + \min\left\{ T, \frac{D}{2\bar{V}_T} \right\} \right) + \log |\mathcal{A}| \right) + \frac{3 \log (|\mathcal{A}| T)}{T},
	\end{equation*}
	and
	\begin{align*} \label{eq:drifting-bandit-regret}
		\inf_\pi \bar{\Delta}_T(\pi) 
        &\leq 5 \cdot \left( \Gamma_U \cdot \log |\mathcal{A}| \cdot \bar{V}_T \right)^{1/3} \cdot \sqrt{ \log\left( 1 + \min\left\{ T , \frac{ \big( \Gamma_U \cdot \log |\mathcal{A}| \big)^{1/3} }{ \bar{V}_T^{2/3} }  \right\} \right) } + \sqrt{ \frac{3 \Gamma_U \log(|\mathcal{A}| T)}{T} }
        \\&\leq 10 \cdot \left( \Gamma_U \cdot \log |\mathcal{A}| \cdot \bar{V}_T \right)^{1/3} \cdot \log(T) + \sqrt{ \frac{3 \Gamma_U \log(|\mathcal{A}| T)}{T} }.
	\end{align*}
\end{proposition}

\begin{proof}
For brevity we define $\mu_t(a) := \mathbb{E}[ R_{t,a} | \theta_t ]$ so that $\Delta_t(a) = \mu_t(A_t^*) - \mu_t(a)$.

Fix $D$ and consider a satisficing action sequence $A^\dagger$ that gets synchronized to $A^*$ whenever $D$-gap occurs:
\[
A_t^\dagger = \left\{ \begin{array}{ll} A_t^* & \text{if } t=1 \text{ or } \mu_{t}(A_t^*) \geq \mu_{t}(A_{t-1}^\dagger) + D, \\ A_{t-1}^\dagger & \text{otherwise}. \end{array}  \right.
\]
It is obvious that $\bar{\Delta}_T(A^\dagger;A^*) \leq D$ since $\mu_{t}(A_t^*) - \mu_{t}(A_t^\dagger) \leq D$.

\paragraph{Step 1.}
Let $V_T := \bar{V}_T \times T$, and $S_T^\dagger := \mathbb{E}\left[ 1+\sum_{t=2}^T \mathbb{I}\{ A_t^\dagger \ne A_{t-1}^\dagger \} \right]$.
We first show that $S_T^\dagger \leq \frac{2 V_T}{D}+1$ by arguing that the total variation must increase at least by $D/2$ whenever the satisficing action switches.
Observe that, in every sample path,
\begin{align*}
	2 \times \sum_{t=2}^T \sup_{a \in \mathcal{A}} \left| \Delta_t(a) - \Delta_{t-1}(a) \right|
        &\geq \sum_{t=2}^T \left| \Delta_{t}(A_t^*) - \Delta_{t-1}(A_t^*) \right| + \left| \Delta_{t}(A_{t-1}^\dagger) - \Delta_{t-1}(A_{t-1}^\dagger) \right|
        \\&= \sum_{t=2}^T \left| \Delta_{t}(A_t^*) - \Delta_{t-1}(A_t^*) \right| + \left| \Delta_{t-1}(A_{t-1}^\dagger) - \Delta_{t}(A_{t-1}^\dagger) \right|
        \\&\geq \sum_{t=2}^T \left( \Delta_{t}(A_t^*) - \Delta_{t-1}(A_t^*) \right) + \left( \Delta_{t-1}(A_{t-1}^\dagger) - \Delta_{t}(A_{t-1}^\dagger) \right)
        \\&= \sum_{t=2}^T  \left( \Delta_{t}(A_t^*) - \Delta_{t}(A_{t-1}^\dagger) \right) - \left( \Delta_{t-1}(A_t^*) - \Delta_{t-1}(A_{t-1}^\dagger) \right)
        \\&\stackrel{(a)}{=} \sum_{t=2}^T  \left( \mu_{t}(A_t^*) - \mu_{t}(A_{t-1}^\dagger) \right) - \left( \mu_{t-1}(A_t^*) - \mu_{t-1}(A_{t-1}^\dagger)  \right)
        \\&\stackrel{(b)}{\geq} \sum_{t=2}^T  \left( \mu_{t}(A_t^*) - \mu_{t}(A_{t-1}^\dagger) \right) - \left( \mu_{t-1}(A_{t-1}^*) - \mu_{t-1}(A_{t-1}^\dagger)  \right)
        \\&= \sum_{t=2}^T \left( \mu_{t}(A_t^*) - \mu_{t}(A_t^\dagger) \right) - \left( \mu_{t-1}(A_{t-1}^*) - \mu_{t-1}(A_{t-1}^\dagger)  \right) +  \left( \mu_{t}(A_t^\dagger) - \mu_{t}(A_{t-1}^\dagger) \right) 
        \\&=  \left( \mu_{T}(A_T^*) - \mu_{T}(A_T^\dagger) \right) - \left( \mu_{1}(A_1^*) - \mu_{1}(A_1^\dagger) \right)  + \sum_{t=2}^T \left( \mu_{t}(A_t^\dagger) - \mu_{t}(A_{t-1}^\dagger) \right)
        \\&= \left( \mu_{T}(A_T^*) - \mu_{T}(A_T^\dagger) \right) - 0  + \sum_{t=2}^T \left( \mu_{t}(A_t^\dagger) - \mu_{t}(A_{t-1}^\dagger) \right) \mathbb{I}\{ A_t^\dagger \ne A_{t-1}^\dagger \}
        \\&\stackrel{(c)}{\geq} 0 + \sum_{t=2}^T D \cdot \mathbb{I}\{ A_t^\dagger \ne A_{t-1}^\dagger \},
\end{align*}
where step (a) uses the definition of $\Delta_t(a)$, step (b) uses the fact that $\mu_{t-1}(A_t^*) \leq \mu_{t-1}(A_{t-1}^*)$, and step (c) holds since $\mu_{T}(A_T^*) - \mu_{T}(A_T^\dagger) \geq 0$ and if $A_t^\dagger \ne A_{t-1}^\dagger$ then $A_t^* = A_t^\dagger$ and $\mu_{t}(A_t^*) - \mu_{t}(A_{t-1}^\dagger) \geq D$.
By taking expectation on both sides, we obtain $2 \times V_T \geq D \times \mathbb{E}[ \sum_{t=2}^T \mathbb{I}\{ A_t^\dagger \ne A_{t-1}^\dagger \} ] = D \times (S_T^\dagger - 1)$.

\paragraph{Step 2.}
We further bound $\bar{\mathcal{R}}_T(D)$ by utilizing \cref{prop:combinatorial-bound}:
$$
	\bar{\mathcal{R}}_T(D) 
		\leq \bar{I}_T(A^\dagger;\theta)
		\leq \bar{H}_T(A^\dagger) 
		\leq \frac{S_T^\dagger}{T} \cdot \left( 1 + \log \left( 1 + \frac{T}{S_T^\dagger} \right) + \log |\mathcal{A}| \right) + \frac{\log T}{T}.
$$
Let $x \wedge y = \min\{x,y\}$ and $x \vee y = \max\{x,y\}$.
Since $1 \leq S_T^\dagger \leq \left( \frac{2 V_T}{D}+1 \right) \wedge T$,
\begin{align*}
	\bar{\mathcal{R}}_T(D)
		&\leq \frac{S_T^\dagger}{T} \cdot \left( 1 + \log \left( 1 + \frac{T}{S_T^\dagger} \right) + \log |\mathcal{A}| \right) + \frac{\log T}{T}
		\\&\leq \frac{2 V_T/D + 1}{T} \cdot \left( 1 + \log \left( 1 + \frac{T}{S_T^\dagger} \right) + \log |\mathcal{A}| \right) + \frac{\log T}{T}
		\\&\leq \frac{2 V_T }{DT} \cdot \left( 1 + \log \left( 1 + \frac{T}{S_T^\dagger}  \right) + \log |\mathcal{A}| \right) + \frac{1 + \log(1+T) + \log |\mathcal{A}|}{T} + \frac{\log T}{T}
		\\&\leq \frac{2 V_T }{DT} \cdot \left( 1 + \log \left( 1 + T \wedge \left(1 \vee \frac{DT}{2 V_T} \right) \right) + \log |\mathcal{A}| \right) + \frac{2\log T + 3\log |\mathcal{A}|}{T} + \frac{\log T}{T}
            \\&=\frac{2 V_T }{DT} \cdot \left( 1 + \log \left( 1 + T \wedge \left(1 \vee \frac{DT}{2 V_T} \right) \right) + \log |\mathcal{A}| \right) + \frac{3\log (|\mathcal{A}|T)}{T},
\end{align*}
where the last inequality uses $1 \leq 2 \log |\mathcal{A}|$ for $|\mathcal{A}| \geq 2$ and $\log(1+T) \leq 2 \log T$ for $T \geq 2$.

\paragraph{Step 3.}
To simplify notation, let $k := |\mathcal{A}|$, and $\kappa := 1 + \log k$.
We consider a satisficing action sequence $A^\dagger$ constructed with $D=D^* := \left( 2 \cdot \Gamma_{U} \cdot \kappa \cdot \bar{V}_T \right)^{1/3}$.
By \cref{thm:rate-distortion}, with $\tau^* :=  T \wedge \left(1 \vee \frac{D^*}{2 \bar{V}_T} \right)$,
\begin{align*}
	\inf_\pi \bar{\Delta}_T(\pi)
		&\leq D^* + \sqrt{ \Gamma_{U} \times \bar{\mathcal{R}}_T(D^*)  }
		\\&\leq D^* + \sqrt{ \Gamma_{U} \times \left[  \frac{2\bar{V}_T}{D^*} \cdot \left( \kappa + \log \left( 1 + \tau^* \right) \right) + \frac{3 \log (k T)}{T} \right]  }
		\\&\stackrel{(a)}{\leq} D^* + \sqrt{ \Gamma_{U} \times \left[  \frac{2\bar{V}_T}{D^*} \cdot \left( \kappa + \log \left( 1 + \tau^* \right) \right) \right]  } + \sqrt{ \frac{3 \Gamma_{U} \log (k T)}{T} }
		\\&= D^* + \sqrt{ \frac{ 2 \Gamma_{U} \kappa \bar{V}_T }{D^*} \times \left( 1 + \frac{ \log \left( 1 + \tau^* \right) }{\kappa} \right)  } + \sqrt{ \frac{3 \Gamma_{U} \log (k T)}{T} }
		\\&= D^* + D^* \times \sqrt{ 1 + \frac{ \log \left( 1 + \tau^* \right) }{\kappa} } + \sqrt{ \frac{3 \Gamma_{U} \log (k T)}{T} }
		\\&\stackrel{(b)}{\leq} 2.6 \times D^* \times \sqrt{ \log \left( 1 + \tau^* \right) } + \sqrt{ \frac{3 \Gamma_{U} \log (k T)}{T} },
\end{align*}
where step (a) uses $\sqrt{x+y} \leq \sqrt{x} + \sqrt{y}$ for any $x, y \geq 0$, and step (b) uses $\kappa \geq 1$ and $\log \left( 1 + \tau^* \right)  \geq \log 2$.
Since $\kappa \leq 3 \log k$, we have 
$$
	D^* = (2 \cdot \Gamma_{U} \cdot \kappa \cdot \bar{V}_T )^{1/3}
		\leq (6 \cdot \Gamma_{U} \cdot \log k \cdot \bar{V}_T )^{1/3}
		\leq 1.9 \cdot (\Gamma_{U} \cdot \log k \cdot \bar{V}_T )^{1/3}.
$$
Therefore,
$$
	\inf_\pi \bar{\Delta}_T(\pi)
		\leq 5 \cdot (\Gamma_{U} \cdot \log k \cdot \bar{V}_T )^{1/3} \times \sqrt{ \log \left( 1 + \min\left\{ T, \frac{(\Gamma_{U} \cdot \log k \cdot \bar{V}_T )^{1/3}}{\bar{V}_T} \right\} \right) } + \sqrt{ \frac{3 \Gamma_{U} \log (k T)}{T} }.
$$
The final claim is obtained by observing that $\sqrt{\log(1+T)} \leq 2 \log T$ for $T \geq 2$.
\end{proof}

\section{Comparison with \cite{liu23}} \label{app:predictive-sampling}

In this section, we provide a detailed comparison between our work and \citet{liu23}.

Both papers adopt a Bayesian perspective that models the nonstationarity through a stochastic latent parameter process whose law is known to the decision maker. Both papers consider Thompson sampling (TS), which is the rule that employs probability matching with respect to the optimal action sequence $A^*$. The primary concern in \cite{liu23} revolves around the potential shortcomings of TS in rapidly changing environments where learning $A^*$ 
is hopeless. Our paper considers this shortcoming of TS in \cref{sec:sts}.
To tackle this challenge, both studies propose some variants of TS designed to `satisfice', i.e., aim to learn and play a different target other than $A^*$.

More specifically, \citet{liu23} proposes a brilliant algorithm called \emph{predictive sampling} (PS) that improves over TS by ``deprioritizing acquiring information that quickly loses usefulness''.
Using our notation, PS's learning target is $A_t^\dagger = \argmax_{a \in \mathcal{A}} \mathbb{E}[ R_{t,a} | 
 \mathcal{F}_{t-1}, (R_{t',a'})_{a' \in \mathcal{A}, t' \geq t+1} ]$, the arm that would have been played by a Bayesian decision maker informed of all future rewards but not the latent parameters.\footnote{
    In its implementation, it generates a fictitious future scenario, and makes an inference on the best arm after simulating the posterior updating procedure with respect to this future scenario.
    This makes the algorithm to take into account the total amount of information about the current state that the decision maker can potentially accumulate in the future, and hence prevents it from being overly optimistic about the learnability of the target.
    A similar idea can be found in \citet{min19} for the stationary bandit problems with a finite-time horizon.}
By ignoring the unresolvable uncertainty in the latent states --- i.e. that which remains even conditioned on all future reward observations --- PS avoids wasteful exploration.
We remark that one of our suggested satisificing action sequence designs, \cref{ex:sts-obscuring-info}, was motivated from \citet{liu23}.
But our satisficing TS does not cannot fully generalize PS since we have restricted the learning target to be determined independently from the algorithm's decisions (see \cref{def:satisficing_action}), while $A_t^\dagger$ above depends on the history of observations under the algorithm.

\citet{liu23} uses an information-ratio analysis to bound an algorithm’s ‘foresight regret’ in terms of the total possible ‘predictive information.’ 
Their custom modifications to the information-ratio analysis in \cite{russo16} enable an analysis of predictive sampling which seems not to be possible otherwise. By contrast, our analysis does not apply to predictive sampling, but produces a wealth of results that do not follow from \cite{liu23}. We elaborate on three main differences in the information-theoretic analysis. 

First, \cite{liu23} bound a custom suboptimality measure which they call foresight regret, whereas we bound conventional (`dynamic') regret. 
Compared to the conventional regret which we denote by $\bar{\Delta}_T(\pi; A^*)$,  foresight regret can serve as a tighter benchmark that better quantifies the range of achievable performance, since it is always non-negative while not exceeding the conventional regret.
For the same reason, however, their results do not provide upper bounds on the conventional regret , a metric treated in a huge preceding literature.
Since our work bounds the conventional regret metric, it enables us to recover results that are comparable to the existing literature; see for instance Section \ref{sec:adversarial}. 

Second, whereas \cite{liu23} introduce a new information ratio, partly specialized to predictive sampling, we use the same information ratio as in \cite{russo16} and many subsequent papers.  Using the conventional  definition of the information ratio  enables us to systematically inherit all the upper bounds established for stationary settings (See Sec.~\ref{sec:information-ratio-bounds}). \cite{liu23} provide an upper bound on their modified information ratio for a particular nonstationary environment they termed `modulated $k$-armed Bernoulli bandit,' (this is identical to our \cref{ex:news-rec}) but more research is required to bound it in other settings. 

Finally, due to using a different definitions of the information ratio, our bounds depend on different notions of the maximial cumulative information gain. 
The corresponding term in our \cref{thm:main-result} is simply the entropy rate of the optimal action process (and elsewhere,  a mutual information rate).
As one of the most fundamental metrics in information theory, numerous techniques can be employed to bound this entropy rate as in \cref{subsec:entropy-rate-bounds}. 
The regret bound in \cite{liu23} involves a term which they call cumulative predictive information. They bound this quantity in the special case of modulated $k$-armed bandits (equivalent to \cref{ex:news-rec}). In that example, their bound on cumulative predictive information is roughly $k$ times larger than our bound on the entropy rate of the optimal action process. As a result, \cite{liu23} obtain a final regret bound that scales with $k$ (see theorem 5 of \citet{liu23}) whereas our bound scales with $\sqrt{k \log(k)}$ (see \cref{ex:news-rec-revisit}).

\section{Numerical Experiment in Detail} \label{app:numerical}

We here illustrate the detailed procedure of the numerical experiment conducted in \cref{subsec:example}.

\paragraph{Generative model.}
We say a stochastic process $(X_t )_{t \in \mathbb{N}} \sim \mathcal{GP}( \sigma_X^2, \tau_X )$ if $(X_1, \ldots, X_t)$ follows a multivariate normal distribution satisfying
$$ \mathbb{E}[ X_i ] = 0, \quad \text{Cov}( X_i, X_j ) = \sigma_X^2 \exp\left( -\frac{1}{2}\left( \frac{i-j}{\tau_X} \right)^2 \right), $$
for any $i, j \in [t]$ and any $t \in \mathbb{N}$.
Note that this process is stationary, and given horizon $T$ a sample path $(X_1, \ldots, X_T)$ can be generated by randomly drawing a multivariate normal variable from the distribution specified by $\sigma_X^2$ and $\tau_X$.

As described in \cref{subsec:example}, we consider a nonstionary two-arm Gaussian bandit with unit noise variance:
$$ R_{t,a} = \underbrace{ \theta_t^{\rm cm} + \theta_{t,a}^{\rm id}}_{:= \mu_{t,a}} + \epsilon_{t,a}, \quad \forall a \in \{1,2\}, t \in \mathbb{N}, $$
where $\epsilon_{t,a}$'s are i.i.d. noises $\sim \mathcal{N}(0,1^2)$, $\big( \theta_t^{\rm cm} \big)_{t \in \mathbb{N}}$ is the common variation process $\sim \mathcal{GP}( 1^2, \tau^{\rm cm} )$, and $\big( \theta_{t,a}^{\rm id} \big)_{t \in \mathbb{N}}$ is arm $a$'s idiosyncratic variation process $\sim \mathcal{GP}( 1^2, \tau^{\rm id} )$.

Note that the optimal action process is completely determined by $\theta^{\rm id}$:
$$ A_t^* = \left\{ \begin{array}{cc} 1 & \text{if } \theta_{t,1}^{\rm id} \geq \theta_{t,2}^{\rm id} \\ 2 & \text{if } \theta_{t,1}^{\rm id} < \theta_{t,2}^{\rm id} \end{array} \right. , $$
and the optimal action switches more frequently when $\tau^{\rm id}$ is small (compare \cref{fig:example-illustration-fast} with \cref{fig:example-illustration}).

\begin{figure}[ht]
\begin{center}
\centerline{\includegraphics[width=0.6\columnwidth]{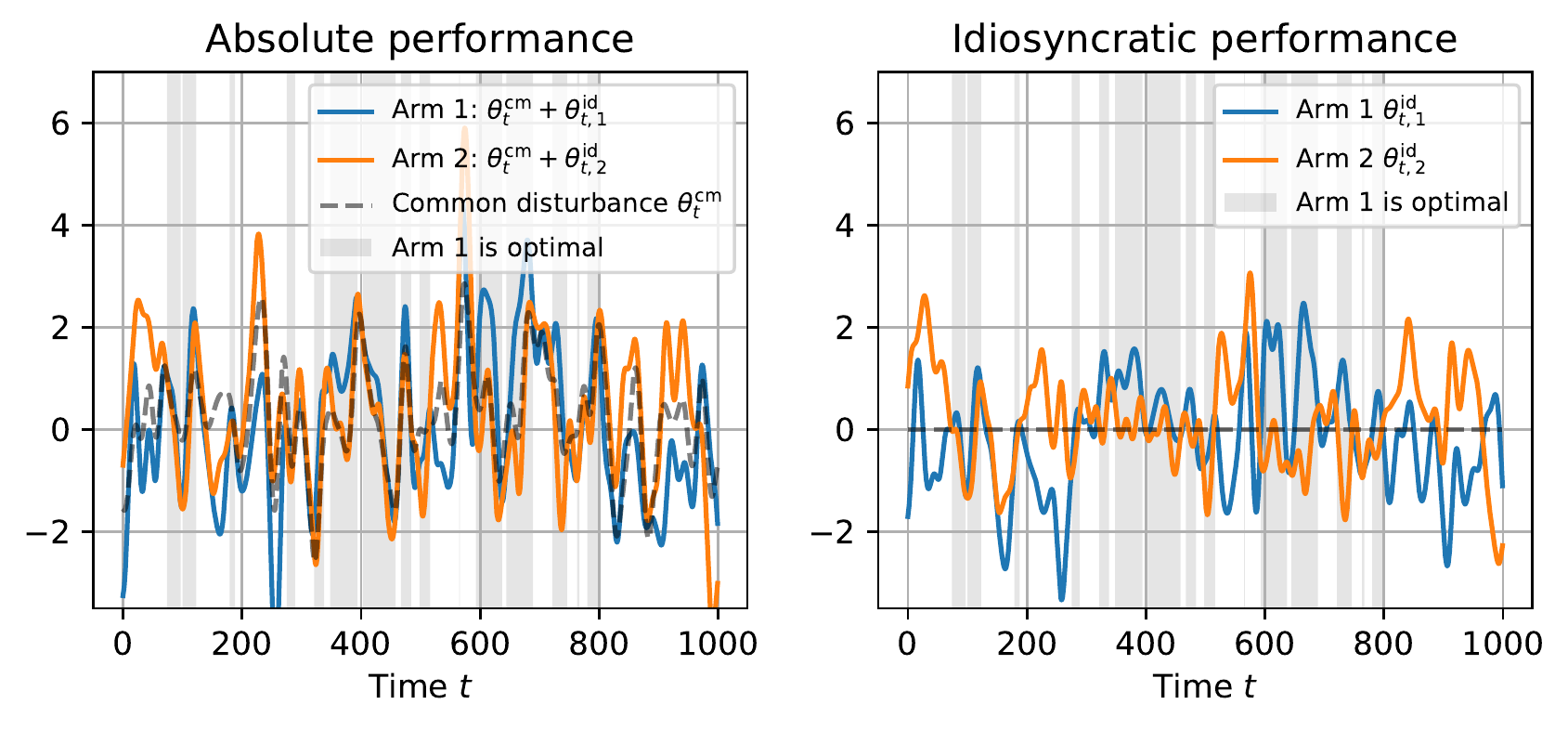}}
\caption{A sample path generated with $\tau^{\rm cm}=\tau^{\rm id}=10$ (cf., \cref{fig:example-illustration} was generated with $\tau^{\rm cm}=10$ and $\tau^{\rm id}=50$).}
\label{fig:example-illustration-fast}
\end{center}
\vskip -0.2in
\end{figure}

Consequently, the problem's effective horizon $\tau_{\rm eff} := 1/\mathbb{P}( A_t^* \ne A_{t-1}^* )$, defined in \eqref{eq:effective-horizon}, depends only on $\tau^{\rm id}$.
To visualize this relationship, we estimate $\tau_{\rm eff}$ using the sample average of the number of switches occurred over $T=1000$ periods (averaged across 100 sample paths), while varying $\tau^{\rm id}$ from $1$ to $100$.
See \cref{fig:effective-horizon}.
As expected, $\tau_{\rm eff}$ is linear in $\tau^{\rm id}$ (more specifically, $\tau_{\rm eff} \approx \pi \times \tau^{\rm id}$ by Rice formula).

\begin{figure}[ht]
\begin{center}
\centerline{\includegraphics[width=0.4\columnwidth]{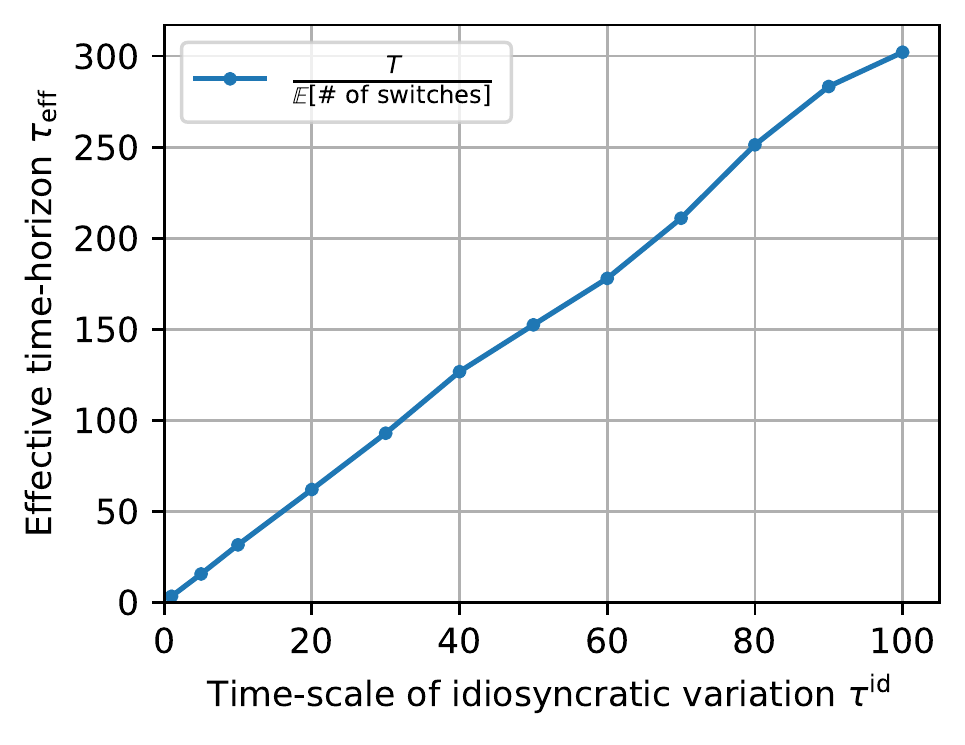}}
\caption{The effective time horizon $\tau_{\rm eff}$ as a function of $\tau^{\rm id} \in \{1,5,10, \ldots, 100\}$, estimated from 100 sample paths randomly generated.}
\label{fig:effective-horizon}
\end{center}
\vskip -0.2in
\end{figure}

\paragraph{Tested bandit algorithms.}
Given the generative model described above, we evaluate four algorithms -- Thompson sampling (TS), Sliding-Window TS (SW-TS; \citet{trovo20}), Sliding-Window Upper-Confidence-Bound (SW-UCB; \citet{garivier2011upper}), and Uniform.

\emph{Thompson sampling} (TS) is assumed to know the dynamics of latent state processes as well as the exact values of $\tau^{\rm cm}$ and $\tau^{\rm id}$ (i.e., no model/prior misspecification).
More specifically, in each period $t$, $\pi^{\rm TS}$ draws a random sample $(\tilde{\theta}_{t,1}^{\rm id}, \tilde{\theta}_{t,2}^{\rm id})$ of the latent state $(\theta_{t,1}^{\rm id}, \theta_{t,2}^{\rm id})$ from its posterior distribution, and then selects the arm $A_t \gets \argmax_{a \in \{1,2\}} \tilde{\theta}_{t,a}^{\rm id}$.
Here, the posterior distribution of $(\theta_{t,1}^{\rm id}, \theta_{t,2}^{\rm id})$ given the history $\mathcal{F}_{t-1} = (A_1, R_1, \ldots, A_{t-1}, R_{t-1} )$ is a multivariate normal distribution that can be computed as follows.
Given the past action sequence $(A_1, \ldots, A_{t-1}) \in \mathcal{A}^{t-1}$, the (conditional) distribution of $(R_1, \ldots, R_{t-1}, \theta_{t,1}^{\rm id}, \theta_{t,2}^{\rm id})$ is given by
$$ \left. \left[ \begin{array}{c} R_1 \\ \vdots \\ R_{t-1} \\ \theta_{t,1}^{\rm id} \\ \theta_{t,2}^{\rm id} \end{array} \right] \right| (A_1, \ldots, A_{t-1}) \sim \mathcal{N}\left( 0 \in \mathbb{R}^{(t-1)+2}, \left[ \begin{array}{cc} \Sigma_{t,RR} \in \mathbb{R}^{(t-1) \times (t-1)} & \Sigma_{t,\theta R}^\top \in \mathbb{R}^{(t-1) \times 2} \\ \Sigma_{t,\theta R} \in \mathbb{R}^{2 \times (t-1)} & \Sigma_{t,\theta \theta} \in \mathbb{R}^{2 \times 2} \end{array} \right] \right), $$
where the pairwise covariances are given by $ (\Sigma_{t,RR})_{ij} := \text{Cov}( R_i, R_j | A_i, A_j ) = \mathbb{I}\{ i = j \} \cdot \text{Var}(\epsilon_{i,a}) + \text{Cov}( \theta_i^{\rm cm}, \theta_j^{\rm cm} ) + \mathbb{I}\{ A_i = A_j \} \cdot \text{Cov}( \theta_{i,a}^{\rm id}, \theta_{j,a}^{\rm id} )$ for $i,j \in [t-1]$, $(\Sigma_{t,\theta R})_{ai} := \text{Cov}( R_i, \theta_{t,a}^{\rm id} | A_i ) = \mathbb{I}\{ A_i = a \} \cdot \text{Cov}( \theta_{i,a}^{\rm id}, \theta_{t,a}^{\rm id}) $ for $a \in [2]$ and $i \in [t-1]$, and $\Sigma_{t,\theta \theta}$ is the identity matrix.
Additionally given the reward realizations, 
$$ \left. \left[ \begin{array}{c} \theta_{t,1}^{\rm id} \\ \theta_{t,2}^{\rm id} \end{array} \right] \right| \mathcal{F}_{t-1} \sim \mathcal{N}( \hat{\mu}_t \in \mathbb{R}^2, \hat{\Sigma}_t \in \mathbb{R}^{2 \times 2} ), \quad 
\text{where} \quad \hat{\mu}_t :=  \Sigma_{t,\theta R} \Sigma_{t, RR}^{-1} \left[ \begin{array}{c} R_1 \\ \vdots \\ R_{t-1} \end{array} \right] , \quad \hat{\Sigma}_t := \Sigma_{t,\theta \theta} - \Sigma_{t,\theta R} \Sigma_{t,R R}^{-1} \Sigma_{t, \theta R}^\top. $$

\emph{Sliding-Window TS} is a simple modification of stationary Thompson sampling such that behaves as if the environment is stationary but discards all observations revealed before $L$ periods ago.
More specifically, in each period $t$, $\pi^{\rm SW-TS}$ with window length $L$ draws a random sample of mean rewards $\tilde{\mu}_{t,a}$ from $\mathcal{N}( \hat{\mu}_{t,a}, \hat{\sigma}_{t,a}^2 )$ where
$$ \hat{\mu}_{t,a} := \frac{ \sum_{s=\max\{1, t-L\}}^{t-1} R_s \mathbb{I}\{A_s = a\} }{1 + N_{t,a}}
	, \quad \hat{\sigma}_{t,a}^2 := \frac{1}{1+ N_{t,a}}
	, \quad N_{t,a} :=  \sum_{s=\max\{1, t-L\}}^{t-1} \mathbb{I}\{A_s = a\}, $$
and then selects the arm $A_t \gets \argmax_{a \in \mathcal{A}} \tilde{\mu}_{t,a}$.
Here, $L$ is a control parameter determining the degree of adaptivity.

Similarly, \emph{Sliding-Window UCB} implements a simple modification of UCB such that computes UCB indices defined as
$$ U_{t,a} := \hat{\mu}_{t,a} + \beta \frac{1}{\sqrt{N_{t,a}}}
	, \quad \hat{\mu}_{t,a} := \frac{ \sum_{s=\max\{1, t-L\}}^{t-1} R_s \mathbb{I}\{A_s = a\} }{N_{t,a}}
	, \quad N_{t,a} :=  \sum_{s=\max\{1, t-L\}}^{t-1} \mathbb{I}\{A_s = a\}, $$
and then selects  the arm $A_t \gets \argmax_{a \in \mathcal{A}} U_{t,a}$.
Here, $L$ is a control parameter determining the degree of adaptivity, and $\beta$ is a control parameter determining the degree of exploration.

\emph{Uniform} is a na\"{i}ve benchmark policy that always selects one of two arms uniformly at random.
One can easily show that $\bar{\Delta}_T( \pi^{\rm Uniform} ) \approx 0.57$ in our setup, regardless of the choice of $\tau^{\rm cm}$ and $\tau^{\rm id}$.

\paragraph{Simulation results.}
Given a sample path specified by $\theta$, we measure the (pathwise) Ces\`{a}ro average regret of an action sequence $A$ as
$$ \bar{\Delta}_T(A; \theta ) := \frac{1}{T} \sum_{t=1}^T (\mu_t^* - \mu_{t,A_t}), \quad \text{where} \quad \mu_{t,a} := \theta_t^{\rm cm} + \theta_{t,a}^{\rm id}, \quad \mu_t^* := \max_{a \in \mathcal{A}} \mu_{t,a}. $$

Given an environment specified by $(\tau^{\rm cm}, \tau^{\rm id})$, we estimate the per-period regret of an algorithm $\pi$ using $S$ sample paths:
$$ \hat{\bar{\Delta}}_T(\pi; \tau^{\rm cm}, \tau^{\rm id}) := \frac{1}{S} \sum_{s=1}^S \bar{\Delta}_T( A^{\pi,(s)}; \theta^{(s)} ),  $$
where $\theta^{(s)}$ is the $s^\text{th}$ sample path (that is shared by all algorithms) and $A^{\pi,(s)}$ is the action sequence taken by $\pi$ along this sample path.
In all experiments, we use $T = 1000$ and $S = 1000$.

We first report convergence of instantaneous regret in \cref{fig:numerical-convergence}: we observe that $\mathbb{E}[ \mu_t^* - \mu_{t,A_t} ]$ quickly converges to a constant after some initial transient periods, numerically verifying the conjecture made in \cref{rem:conjecture}.
While not reported here, we also observe that the Ces\`{a}ro average $\bar{\Delta}_T(A; \theta )$ converges to the same limit value as $T \rightarrow \infty$ in every sample path, suggesting the ergodicity of the entire system.
\begin{figure}[ht]
\begin{center}
\centerline{\includegraphics[width=0.5\columnwidth]{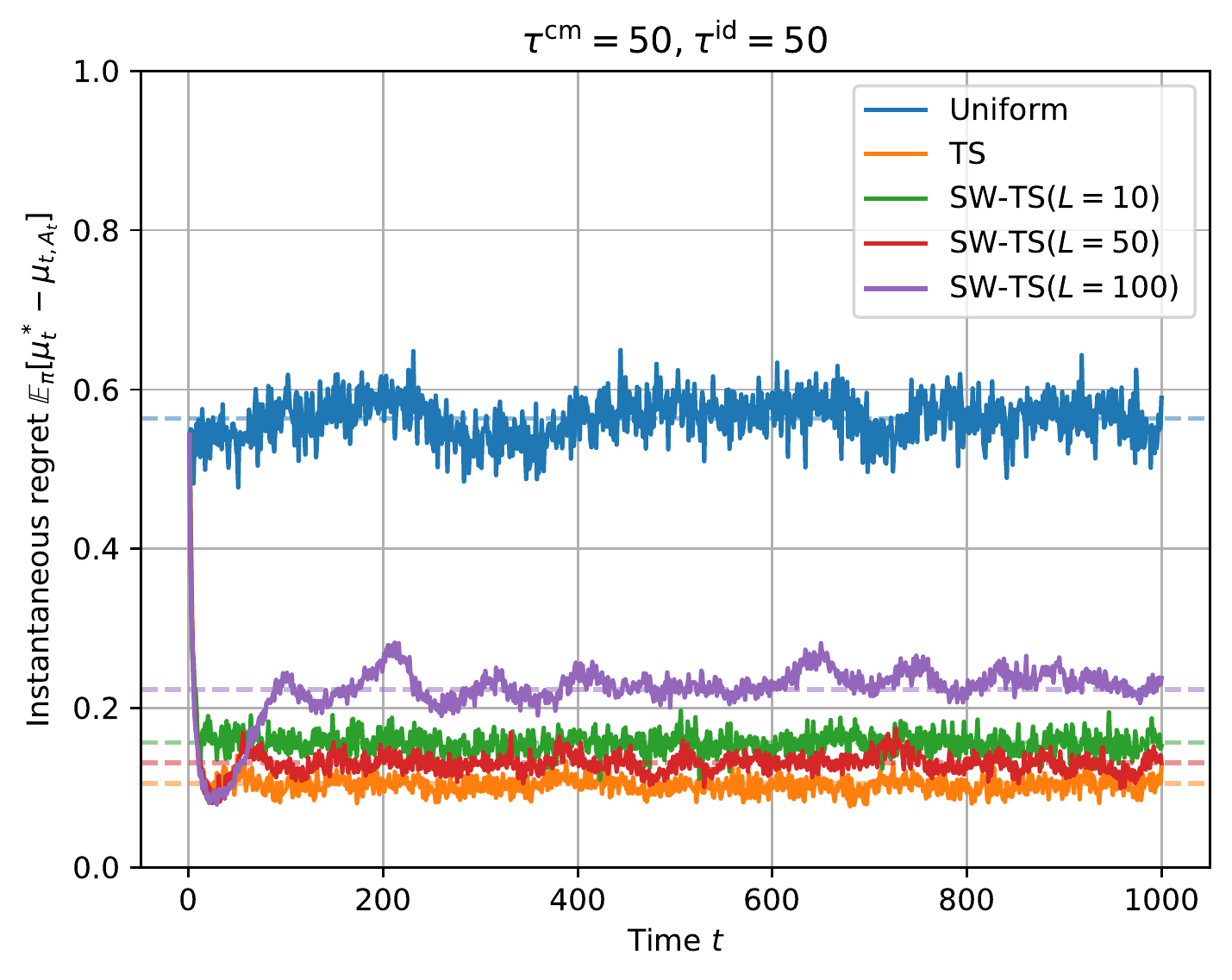}}
\caption{Convergence of instantaneous regret $\mathbb{E}[ \mu_t^* - \mu_{t,A_t} ]$ in the case of $\tau^{\rm cm} = \tau^{\rm id} = 50$. The solid lines report the instantaneous regret of the algorithms, averaged across $S=1000$ sample paths, and the dashed horizontal lines represent the estimated per-period regret.}
\label{fig:numerical-convergence}
\end{center}
\vskip -0.2in
\end{figure}

We next examine the effect of $\tau^{\rm id}$ and $\tau^{\rm cm}$ on the performance of algorithms, and provide the detailed simulation results that complement \cref{fig:example-performance} of \cref{subsec:example}.
While varying $\tau^{\rm id}$ and $\tau^{\rm cm}$, we measure the per-period regret $\hat{\bar{\Delta}}_T(\pi; \tau^{\rm cm}, \tau^{\rm id})$ of algorithms according to the procedure described above.
We observe from \cref{fig:example-performance2} that for every algorithm its performance is mainly determined by $\tau^{\rm id}$, independent of $\tau^{\rm cm}$, numerically confirming our main claim -- the difficulty of problem can be sufficiently characterized by the entropy rate of optimal action sequence, $\bar{H}(A^*)$, which depends only on $\tau^{\rm id}$.
We additionally visualize the upper bound on TS's regret that our analysis predicts (\cref{ex:two-armed-bandit-revisited}).
We also observe that TS performs best across all settings, perhaps because TS exploits the prior knowledge about nonstationarity of the environment, whereas SW-TS or SW-UCB performs well when the window length roughly matches $\tau^{\rm id}$.

\begin{figure}[ht]

\begin{center}
    \begin{minipage}{0.45\textwidth}
        \centering
        \includegraphics[width=0.9\textwidth]{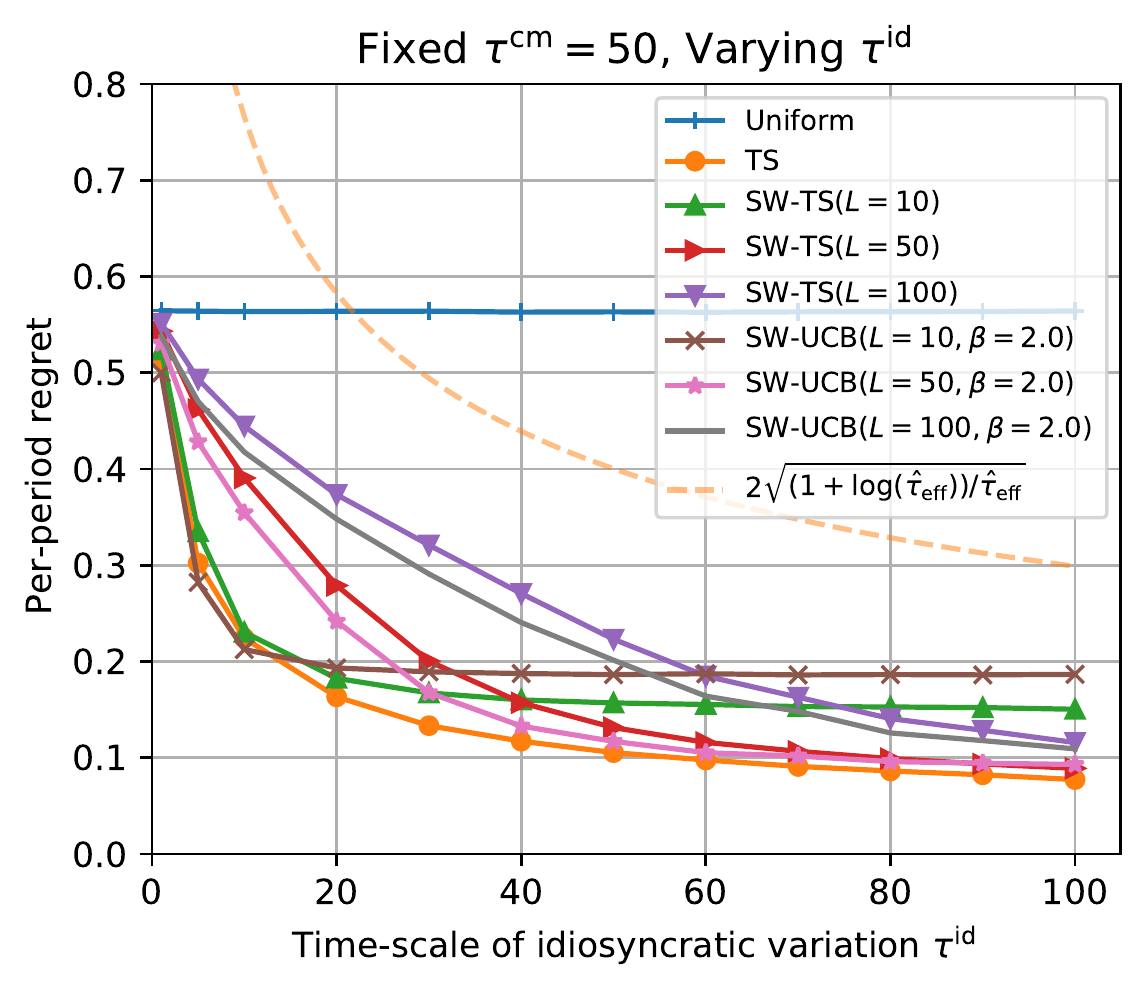} 
    \end{minipage}
    \begin{minipage}{0.45\textwidth}
        \centering
        \includegraphics[width=0.9\textwidth]{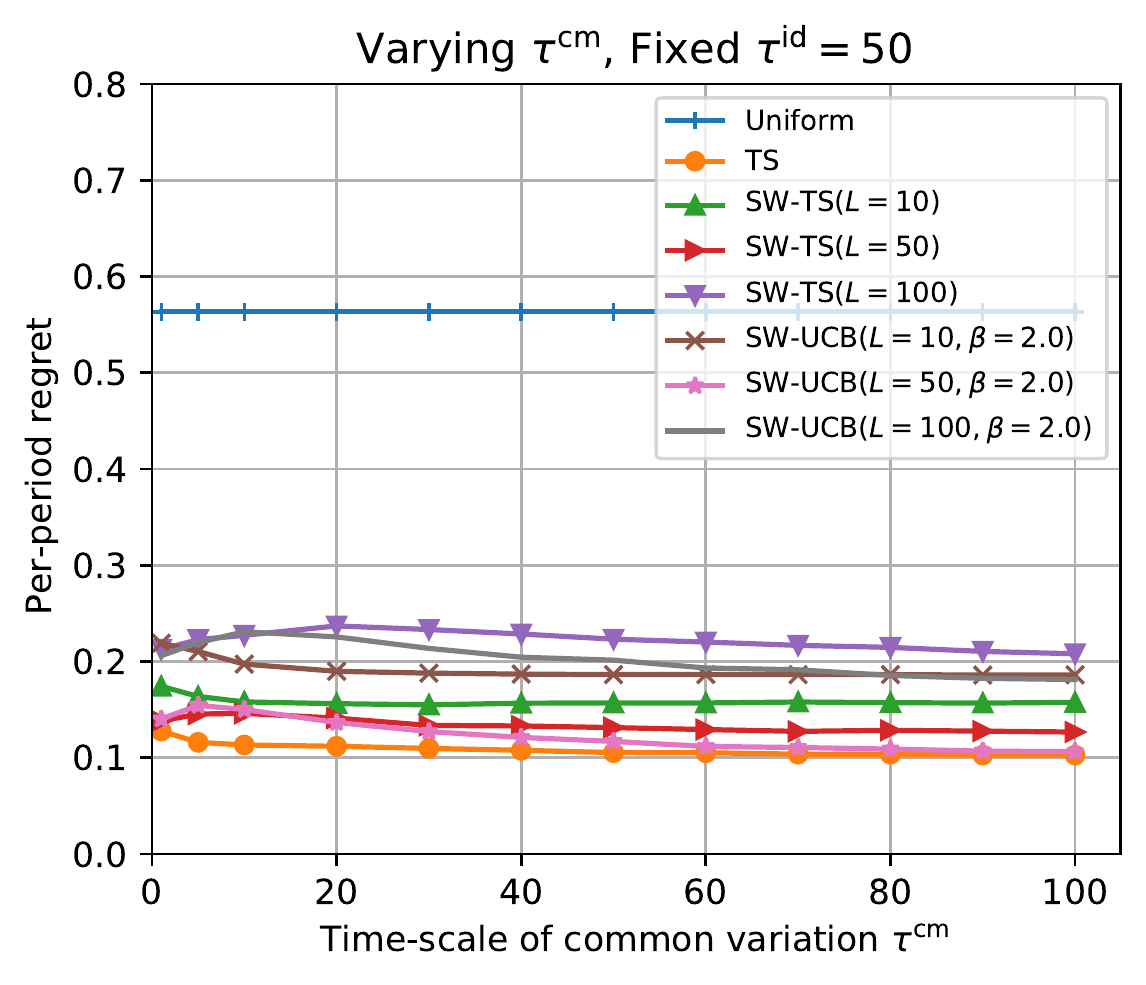} 
    \end{minipage}
    \caption{The effect of $\tau^{\rm id}$ (left) and $\tau^{\rm cm}$ (right) on the performance of algorithms. The dashed line in the left plot represents the upper bound on $\bar{\Delta}_\infty(\pi^{\rm TS})$, implied by \cref{cor:k-armed} and \cref{ex:two-armed-bandit-revisited}.}
\label{fig:example-performance2}
\end{center}
\vskip -0.2in
\end{figure}

\end{document}